%% file: main.tex
\documentclass[final,onefignum,onetabnum]{siamart190516}

\usepackage{graphicx}
\usepackage{placeins}
\usepackage{wrapfig}

\usepackage{amsmath}
\usepackage{amsfonts}
\usepackage{amssymb}
\usepackage{url}
\usepackage{tikz}
\usetikzlibrary{bayesnet, positioning}

\usepackage{subcaption}
\usepackage{algorithm}
\usepackage{algorithmic}

\newtheorem{remark}[theorem]{Remark}
\renewcommand{\thesubsubsection}{\thesubsection.\alph{subsubsection}}

\DeclareMathOperator*{\argmax}{arg\,max}
\DeclareMathOperator*{\argmin}{arg\,min}

\include{macros}

\title{Bayesian System ID: optimal management of parameter, model, and measurement uncertainty}

\author{Nicholas Galioto $^*$ and Alex A. Gorodetsky
\thanks{N. Galioto and A.A. Gorodetsky are with the Department of Aerospace Engineering, University of Michigan, Ann Arbor, MI 48109, USA,
        \textit{email:} {\tt\small \{ngalioto, goroda\}@umich.edu}}
}

\begin{document}

\maketitle

\begin{abstract}
    We evaluate the robustness of a probabilistic formulation of system identification (ID) to sparse, noisy, and indirect data. Specifically, we compare estimators of future system behavior derived from the Bayesian posterior of a learning problem to several commonly used least squares-based optimization objectives used in system ID. Our comparisons indicate that the log posterior has improved geometric properties compared with the objective function surfaces of traditional methods that include differentially constrained least squares and least squares reconstructions of discrete time steppers like dynamic mode decomposition (DMD). These properties allow it to be both  more sensitive to new data and less affected by multiple minima -- overall yielding a more robust approach.  Our theoretical results indicate  that least squares and regularized least squares methods like dynamic mode decomposition and sparse identification of nonlinear dynamics (SINDy) can be derived from the probabilistic formulation by assuming noiseless measurements. We also analyze the computational complexity of a Gaussian filter-based approximate marginal Markov Chain Monte Carlo scheme that we use to obtain the Bayesian posterior for both linear and nonlinear problems. We then empirically demonstrate that obtaining the marginal posterior of the parameter dynamics and making predictions by extracting optimal estimators (e.g., mean, median, mode) yields orders of magnitude improvement over the aforementioned approaches. We attribute this performance to the fact that the Bayesian approach captures parameter, model, and measurement uncertainties, whereas the other methods typically neglect at least one type of uncertainty.
\end{abstract}


\section{Introduction}

Recovering nonlinear models of dynamical systems from data is quickly becoming a primary enabling technology for analysis and decision making in fields spanning science and engineering where first principles models are often incomplete or simply unavailable. Examples range from forecasting the weather and climate~\cite{Glahn_1972, Gneiting_2005, Maqsood_2004}, predicting fluid flows~\cite{Lumley_2007, Sirovich_1987, Aubry_1991}, and enabling adaptive control~\cite{Polycarpou_1993}. All of these fields have a long history of developing estimation and system identification techniques such as advanced Kalman filtering in forecasting~\cite{Ott_2004, Hunt_2004}, dynamic mode decomposition for computational fluid dynamics~\cite{schmid_2010}, and a wide ranging set of schemes in adaptive control~\cite{Leffens_1982, Slotine_1987, Craig_1987}. In this paper we compare the implicit and explicit optimization formulations posed by several representative approaches, and we demonstrate that algorithms that appropriately manage parameter, model, and measurement uncertainty in a cohesive manner are often more robust than more standard least squares-based approaches.

For any system identification approach, there are two primary challenges: (1) parameterizing a model space over which to search and (2) posing an optimization problem whose minimum yields an optimal model. A great majority of recent work has focused on addressing the first challenge, primarily due to the rapid availability of machine learning software. These recent works seek to learn neural network representations of problems because of their representation capacity~\cite{Becerra_2005, Mastorocostas_2002, Delgado_1995}. These works are partly motivated by the belief that modern systems are complicated and existing linear or linear-subspace models are no longer capable of representing the systems we seek to model. 

In this paper, we explore the second challenge -- that of posing an optimization problem, or, more generally, specifying a goal whose minimum will yield a system with predictive power. We argue that this problem is equally, if not more, important than appropriately parameterizing a model space. We support this assertion by showing that many currently used optimization objective specifications fail to recover models \textit{even when the correct model class is known}. Specifically, these specifications cause system identification techniques to break down in the presence of sparse measurements and/or noisy data.  

We advocate a probabilistic approach to system dynamics that explicitly provides for the representation and incorporation of three uncertainties: parameter uncertainty, model uncertainty, and measurement uncertainty. This probabilistic setting, given in Section~\ref{sec:problem}, poses the problem as a hidden Markov model and is well known in the estimation and filtering literature across disciplines~\cite{Sarkka_2013,Law2015,Barfoot2017}. Despite being well known, this setting has not been thoroughly compared to predominant system identification approaches in the context of model learning rather than filtering/smoothing. 

The solution to this formulation is a posterior distribution of the model parameters given the observed data. As a result, predictions and forecasting become probabilistic -- weighting future outcomes by their relative probabilities.  This posterior distribution must be computed using computational inference approaches such as Markov Chain Monte Carlo or variational inference. Given a posterior distribution, goal-oriented estimators can be extracted based on a specified loss and risk metric~\cite{Berger_1985}. For instance, it is well known that the posterior mean is the optimal estimator for Bayes risk with squared loss, and the posterior median is optimal for $L_1$ loss.

To this end, our contributions involve proving that several existing and popular approaches for system identification, sparse identification of nonlinear dynamics (SINDy)~\cite{Brunton_2016} and dynamic mode decomposition (DMD)~\cite{schmid_2010}, are realizations of the probabilistic framework under some limiting assumptions (they assume no measurement uncertainty). We choose these two approaches because they are representative of many (nonlinear) least squares type approaches that are used. We then empirically demonstrate that we yield improved predictions compared to these approaches on wide varying problems. Concretely, our contributions are the following:
\begin{enumerate}
    \item A complexity analysis of the unscented Kalman filter MCMC (UKF-MCMC) algorithm developed in~\cite{Erazo_2018}, which enables an approximate marginal Markov Chain Monte Carlo algorithm to sample from the marginal posterior of the model parameters;
    \item Theorems~\ref{th:DMD} and~\ref{th:sindy} proving that DMD and SINDy can be viewed as specific cases of the presented probabilistic approach with additional assumptions of zero measurement noise; and
    \item A wide ranging set of numerical simulation results demonstrating the robustness and improved prediction quality of our approach in all cases, including sparse and noisy data. 
\end{enumerate}

The UKF-MCMC approach mentioned in the first contribution refers to a computational algorithm that targets the marginal posterior distribution of the model parameters to avoid performing inference over the joint parameter-state space. It can be viewed as an approximation of the marginal likelihood that is traditionally very difficult to compute for dynamical systems~\cite{Haario2015}.  The UKF-MCMC algorithm is one of a number of algorithms that have been recently developed that draw on Gaussian-based filtering to approximate the marginal likelihood~\cite{Noh_2019, Drovandi_2019, Khalil_2015}. These algorithms trade off the approximation quality for some additional computational efficiency compared with the seminal particle-marginal approach of Andrieu~\cite{Andrieu_2009}, which is able to reconstruct the exact posterior. Nevertheless, our results indicate that the posterior approximated by the UKF-MCMC algorithm is still able to reconstruct systems with good accuracy. 

We apply the UKF-MCMC algorithm to the hierarchical Bayesian setting where we explicitly learn the process and measurement covariance of the dynamical system. Furthermore, we use standardized uninformative priors for the model parameters and standard half-normal priors for the unknown covariances~\cite{Gelman_2006}. As a result, our algorithm requires no additional parameters, besides number of MCMC samples, compared to competing single-point estimators (DMD and SINDy). Furthermore, we provide a computational complexity analysis showing that the expense of our approach compared to these existing approaches grows linearly with the number of data points. However, our accuracy gains are shown to sufficiently offset this expense.

The second contribution aims to uncover, or at least interpret, some of the underlying assumptions that have led to observed poor performance of the methods to which we compare.  Many data-driven methods claim a certain degree of objectivity (as compared to, for instance, the Bayesian approach we propose here) because they avoid placing strong assumptions (priors) on the system model that may influence the method's estimate. In reality, however, ``analyses that have the appearance of objectivity virtually always contain hidden, and often quite extreme, subjective assumptions''~\cite{Berger_1985}.   It will be shown that the estimators DMD and SINDy hold the hidden assumption that uncertainty enters only through process noise and that the measurements are noiseless.  Conversely, techniques such as parameter optimization of deterministic ODEs account only for noise in the measurements and not in the process.  The Bayesian estimator presented here will be shown to outperform these common approaches as soon as their underlying assumptions are violated, even in the modifications of these algorithms that incorporate denoising~\cite{Hemati_2017, Chartrand_2011}.

The rest of this paper is structured as follows. In Section~\ref{sec:lsq_objectives}, we explain the central problem this paper hopes to address: how common least squares-based system ID approaches create objective functions with certain undesirable features.  We illustrate this problem by providing the contours of two common objective functions for a simple two-dimensional problem and show how the Bayesian approach incorporates the advantageous features of both without including the problematic features.  In Section~\ref{sec:problem}, we detail the probabilistic framework of a system ID problem, including the problem setup and primary goals.  Then in Section~\ref{sec:existing_approaches}, we provide an analysis of two existing approaches to system ID: DMD and SINDy.  In this section, we give a brief explanation on the implementation of both algorithms and then provide theory that unveils the underlying assumptions used by these approaches.  Section~\ref{sec:algorithm} outlines the algorithms used to implement the Bayesian approach and provides a comparison of their computational complexity to that of DMD and sparse regression.  Finally, Section~\ref{sec:experiments} applies the Bayesian algorithm to five different dynamical systems including linear, nonlinear, chaotic, and PDE systems.  Comparisons of the Bayesian algorithm to DMD and SINDy are given, and it is shown that not only is the Bayesian approach just as effective for systems for which these common approaches display exemplary performance, but the Bayesian algorithm also remains robust in certain regimes where DMD and SINDy fail.

\section{Representative challenges in common least squares approaches to system ID}
\label{sec:lsq_objectives}

In this section, we highlight the geometry of the objective functions of several representative optimization formulations for system identification that we explore in this paper. Specifically, we consider three objectives: one that considers measurement uncertainty but no process/model uncertainty; one that assumes process uncertainty but no measurement uncertainty; and finally our proposed approach that considers both process and measurement uncertainty. The first two approaches are the most commonly used, but we show that they suffer from multiple minima and poor data sensitivity, respectively. Furthermore, while variations of these approaches are used on complex systems, we highlight their limitations in an extremely simple setting of recovering a linear pendulum.

To motivate the results, consider a simple setting where the true model is a linear oscillator with a frequency of 2.00rad/s. Suppose that the learning objective is to identify the frequency. One might intuitively believe that the following least squares objective (in the time-domain) would appropriately penalize incorrect frequencies $\omega$
\begin{align}
   J(\omega, T) =  \int_0^T \Big(\text{cos}(2.00t) - \text{cos}(\omega t)\Big)^2dt,
\end{align}
where $J(\omega, T)$ measures the ``error'' of estimating some parameter $\omega.$ An optimization scheme would then try to find the parameter $\omega$ to minimize $J$. This objective is not derived from any arguments, rather it is intuitively specified and here we attempt to see whether this specification makes sense.\footnote{For this linear problem, it is more appropriate to consider frequency-domain system ID, which would not encounter the problems described here. However, these types of time-domain system ID procedures using least squares-based regression/machine learning approaches are increasingly being used for complex nonlinear systems~\cite{Chen_2020, Rackauckas_2020, Erichson_2019, Raissi_2019}, and we seek to show that they can be limited in an extremely simple setting.}

Prior to considering the full system identification, we consider a property of the least squares objective. We compare the cost of two parameters $\omega=2.01$rad/s and $\omega=4.00$rad/s at two different times $T=10$ and $T=1000.$ In the case where we obtain noise-free data for ten seconds, we obtain $J(2.01, 10) = 0.02$ and $J(4.00, 10) = 9.63$ -- as we desire, the cost of estimating $\omega=2.01$rad/s is more than 100 times lower than estimating $\omega=4.00$rad/s. However, suppose that we obtain data for 100 times longer. Then we obtain $J(2.01,1000) = 1053.96$ and $J(4.00, 1000) = 999.58$ -- the relative difference between the two objectives has shrunk tremendously. 

In this example, even small perturbations from the true parameters of a system yield large errors given enough time, and, in this case, greatly reduce the relative benefit of $\omega = 2.01$ over $\omega = 4.00.$ In simpler terms, this example demonstrates that as the number of data points increases, the relative difference between $\omega = 2.01$ and $\omega=4.00$ decays! The practical implication is that optimization formulations may have significantly more difficulty in distinguishing between correct and incorrect parameters. The issue here is that the least squares objective does not seem to behave as intuition would expect, nor does it match the behavior we are aiming to achieve. Specifically, we seek an objective function that exaggerates the difference between parameters with small errors and those with large errors as more data are obtained.

In this paper, we show that an approach that introduces (and then seeks to reduce) the uncertainty in parameters, models, and measurements leads to objective functions that are far better behaved.  For example, to account for imprecision in our parameter estimates, we must include process noise in our model, allowing us to reposition our reconstructed trajectory at the time of every measurement. This way, the model is not given enough time for errors to accumulate and eventually become indistinguishable from a far worse model. While the process noise is not the model error, it does encapsulate the fact that the predicted motion is incorrect. In fact, we empirically show that it should be included even when the model class spanned by the parameters encapsulates the truth.

To this end, we consider three objective functions
\begin{align}
    \theta^* &= \argmin_{\theta} \  \sum_{i=1}^n \lVert (y_i - x(t_i)) \rVert_2^2  \quad \text{subject to} \quad  \frac{dx}{dt} = f(t, x; \theta) \label{eq:detobj} \\
    \theta^* &= \argmin_{\theta} \  \sum_{i=2}^n \lVert y_{i} - \Psi(y_{i-1}; \theta) \rVert_2^2 \label{eq:noobscobj}\\
    \theta^{*} &= \argmax_{\theta} \log( p(y_1,\ldots,y_{n} \mid \theta)) \label{eq:ourobj}
\end{align}
where $\theta$ are model parameters, $f$ are continuous dynamics representing the time derivatives of a problem, and $\Psi$ are discrete propagators. The first objective, Equation~\eqref{eq:detobj}, assumes deterministic dynamics and performs least squares regression to match the trajectory of a differential equation to the data. The least squares objective here implicitly accounts for measurement noise, and is widely used in the literature~\cite{Ayed_2019, Peng_2010, Evensen_1998}. The second objective, Equation~\eqref{eq:noobscobj}, assumes there is no measurement noise, only process noise/model uncertainty, and instead builds a propagator between observations. This objective is representative of DMD~\cite{schmid_2010} and similar least squares approaches~\cite{Wu_2020, Constantine_2012,Dai_1989}. The final objective, Equation~\eqref{eq:ourobj}, is the log marginal likelihood arising from Theorem~\ref{th:post} that we advocate for in this paper. Note that standard $L_2$ and sparsity-enhancing $L_1$ regularizations/priors can also be included to each of these objectives, but they do not change the qualitative conclusions.

Figure~\ref{fig:objCont} shows these objective functions for the case of learning a continuous time linear pendulum
\begin{align}
    \begin{bmatrix}\dot{x}_1 \\ \dot{x}_2 \end{bmatrix} = 
    \begin{bmatrix} 0 & \theta_1 \\ \theta_2 & 0  \end{bmatrix}
    \begin{bmatrix}x_1 \\ x_2 \end{bmatrix},
    \label{eq:pend_conts_param}
\end{align}
where the true parameters are $\theta_1 = 1$ and $\theta_2 = -g/L.$ Here, $g$ is the acceleration due to gravity and $L$ is the pendulum length. Data are obtained in 0.1 second increments with noise standard deviation of 0.1. Each column corresponds to a different objective. The first column corresponds to the objective of Equation~\eqref{eq:detobj}, the second corresponds to Equation~\eqref{eq:noobscobj}, and the third to Equation~\eqref{eq:ourobj}.

The rows of this figure correspond to 20, 40, and 80 data points collected in 0.1 second increments, respectively. In the left panel, we see that the assumption of a deterministic system (no process noise) results in many local minima, each of which represents a system that matches the data closely at some points and at other points may be completely opposite.  In the middle column we see that excluding the measurement noise has smoothed over certain features of the deterministic system and the objective becomes insensitive to the number of data points. This panel corresponds to the shape of the objective used by DMD. In this case, finding the minimum of the objective is fast, but the reconstructed system may lack some of the key features of the true trajectory and is not tremendously affected by increasing data.  Lastly, the third column represents the objective arising from our probabilistic approach that considers all types of uncertainty. Only in this approach do we see that increasing data has a beneficial effect on the objective function. Multiple local minima do not exist, the characteristic shape seen in the left column remains, and the objective becomes steeper in the region of the minimum.

\begin{figure}
    \centering
    \begin{subfigure}{0.35\linewidth}
        \includegraphics[width=\linewidth]{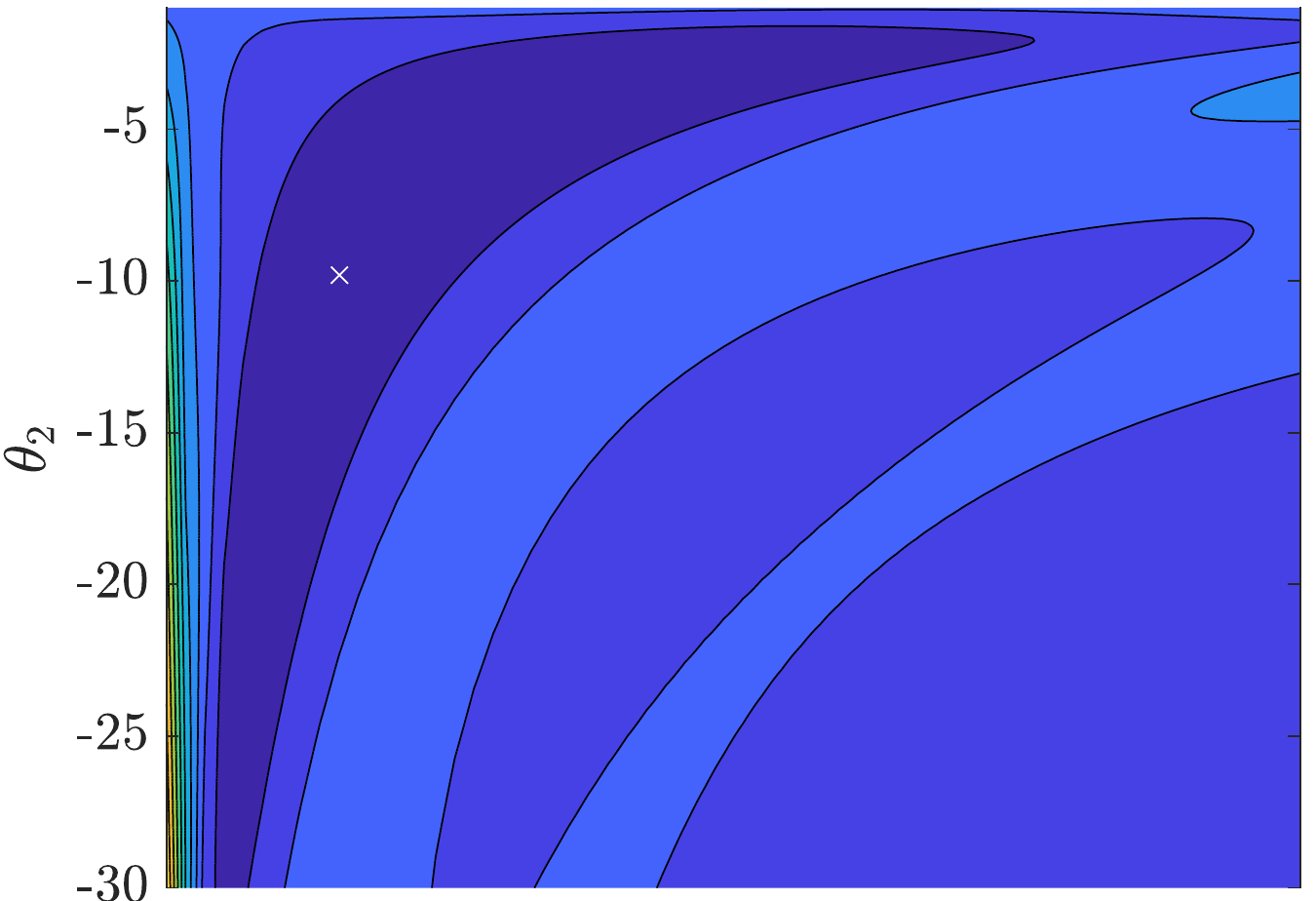}
        \label{fig:objCont11}
    \end{subfigure}
    \begin{subfigure}{0.31\linewidth}
        \includegraphics[width=\linewidth]{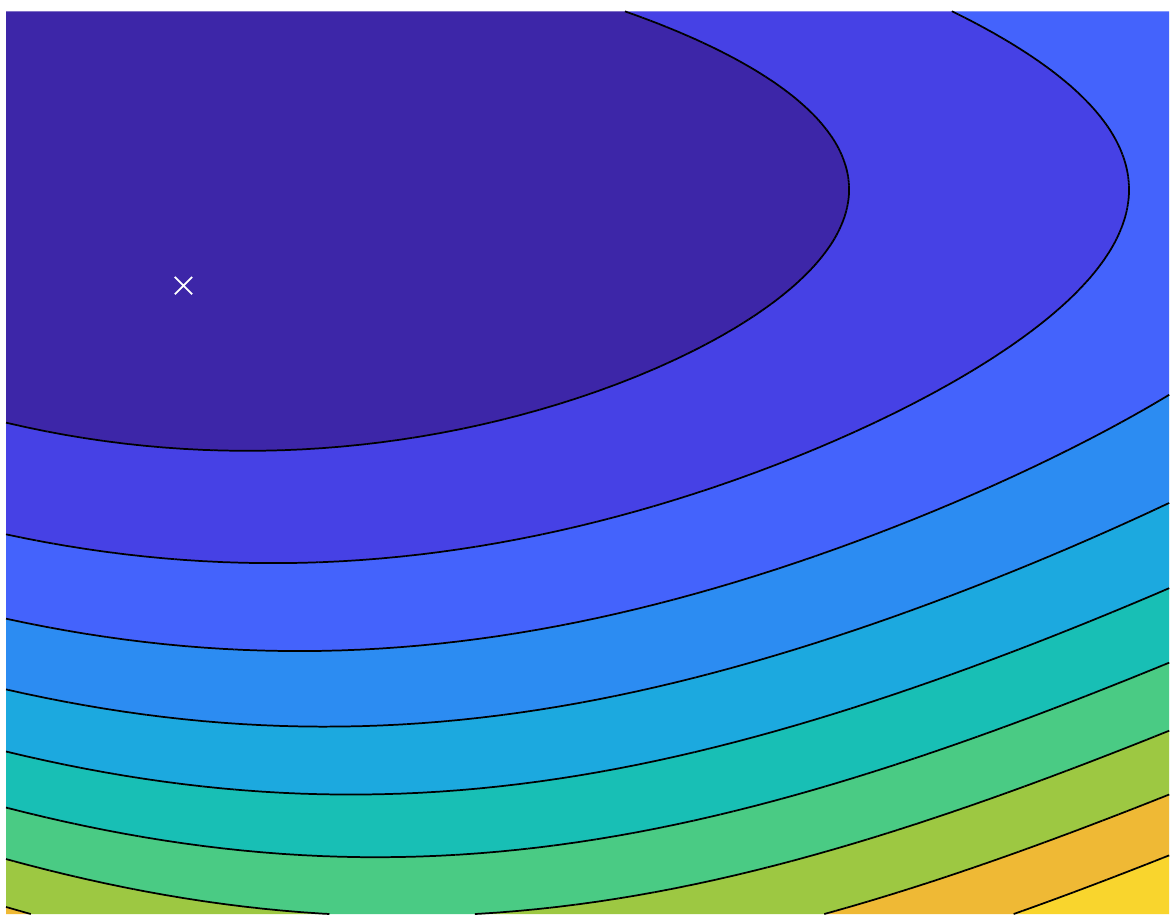}
        \label{fig:objCont13}
    \end{subfigure}
    \begin{subfigure}{0.32\linewidth}
        \includegraphics[width=\linewidth]{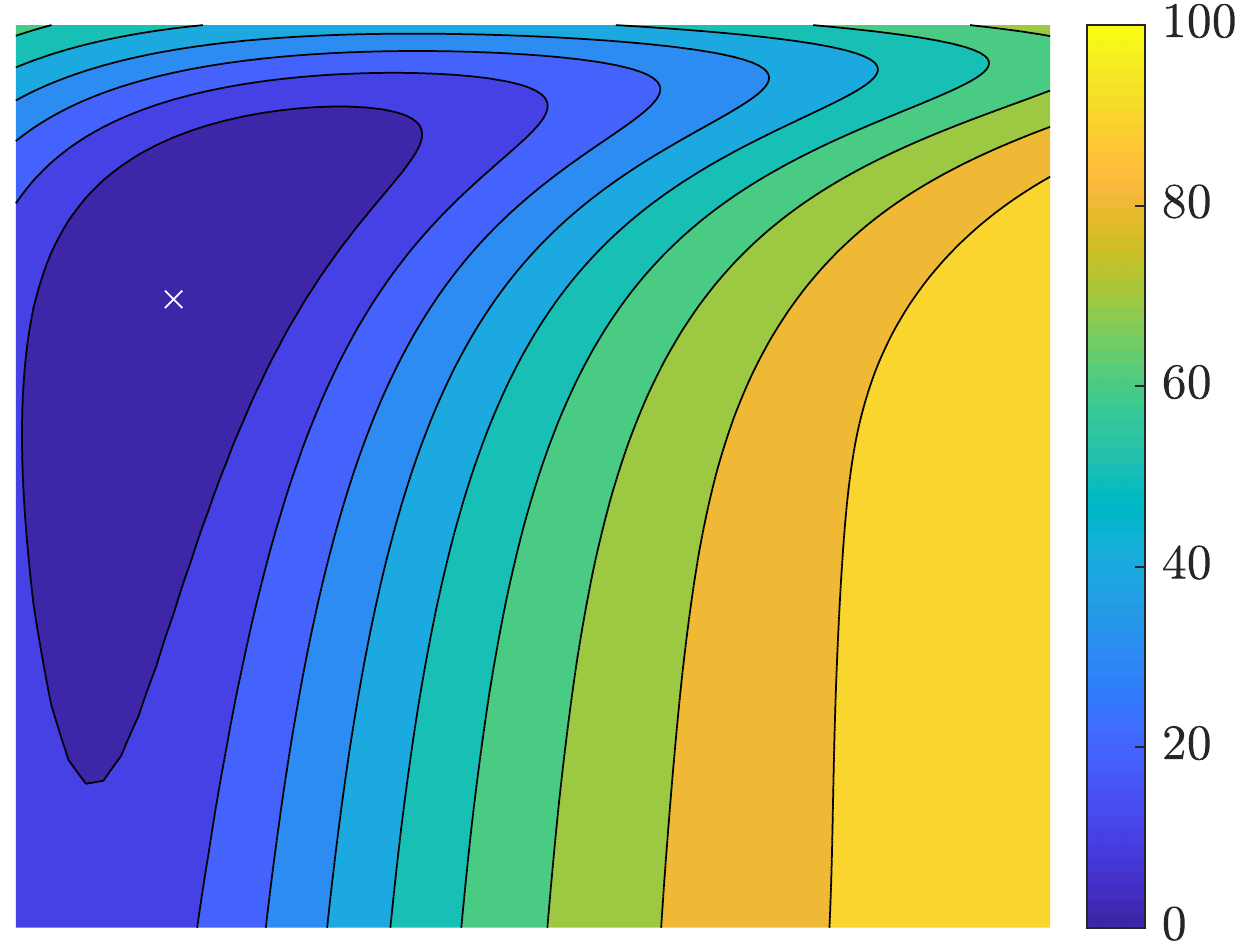}
        \label{fig:objCont12}
    \end{subfigure}
    \begin{subfigure}{0.35\linewidth}
        \includegraphics[width=\linewidth]{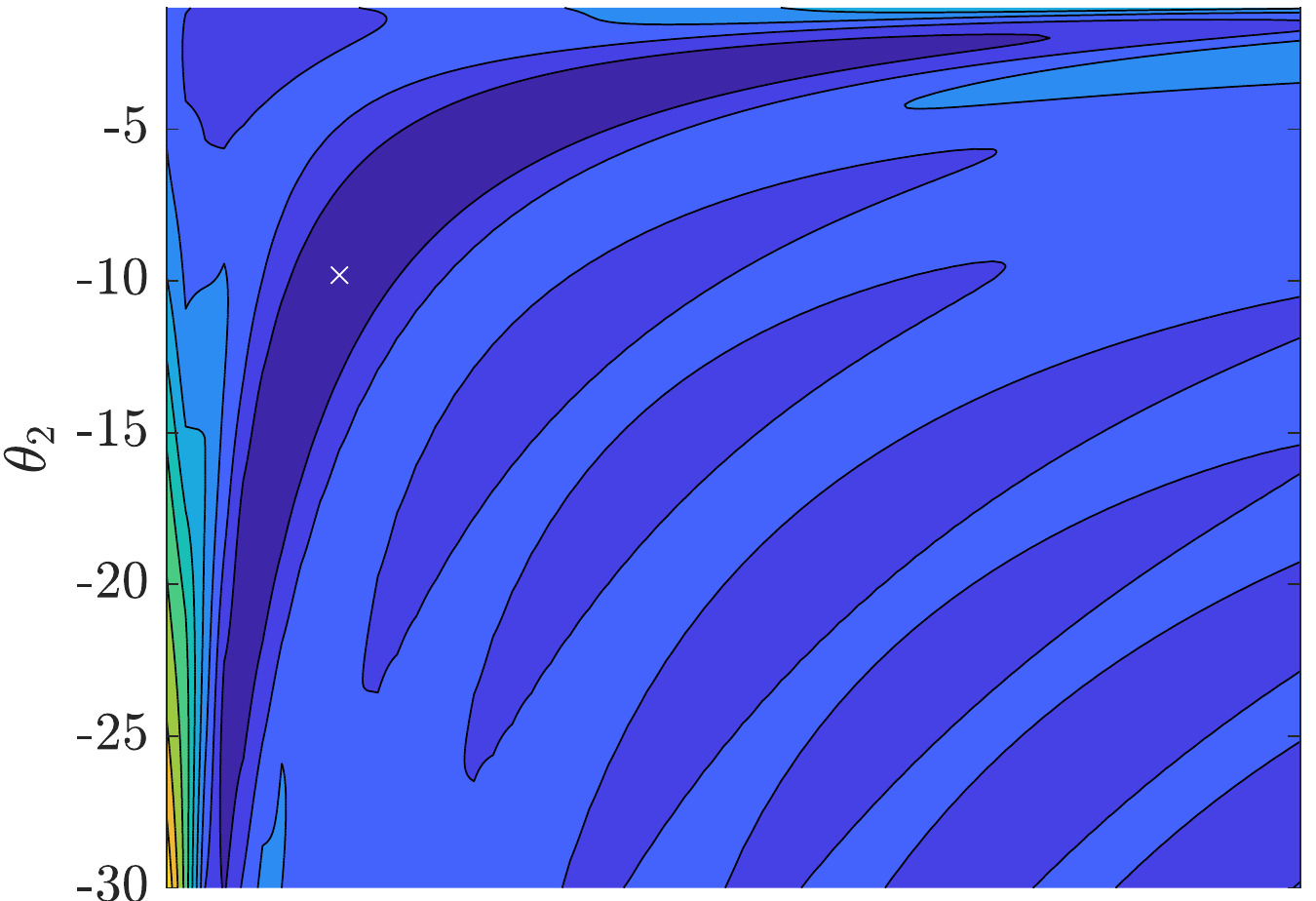}
        \label{fig:objCont21}
    \end{subfigure}
    \begin{subfigure}{0.31\linewidth}
        \includegraphics[width=\linewidth]{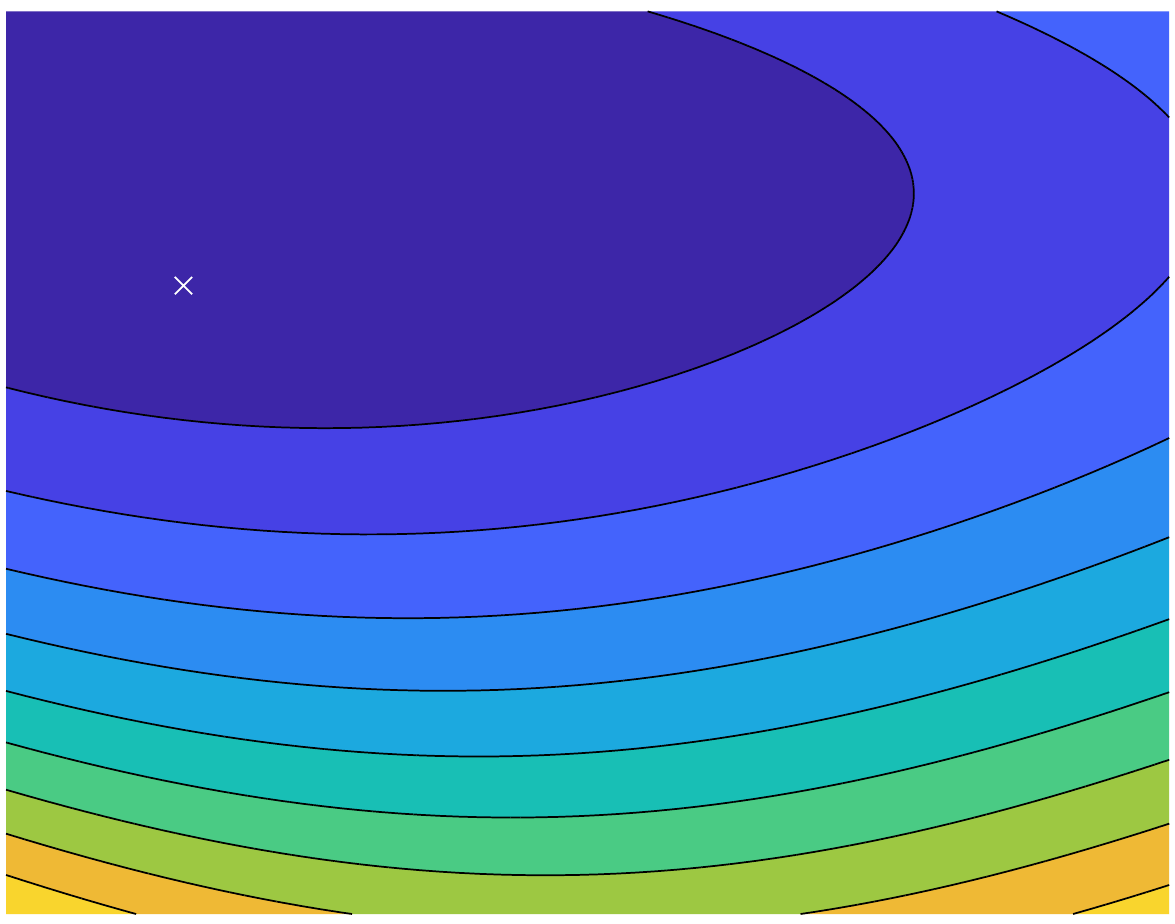}
        \label{fig:objCont23}
    \end{subfigure}
    \begin{subfigure}{0.32\linewidth}
        \includegraphics[width=\linewidth]{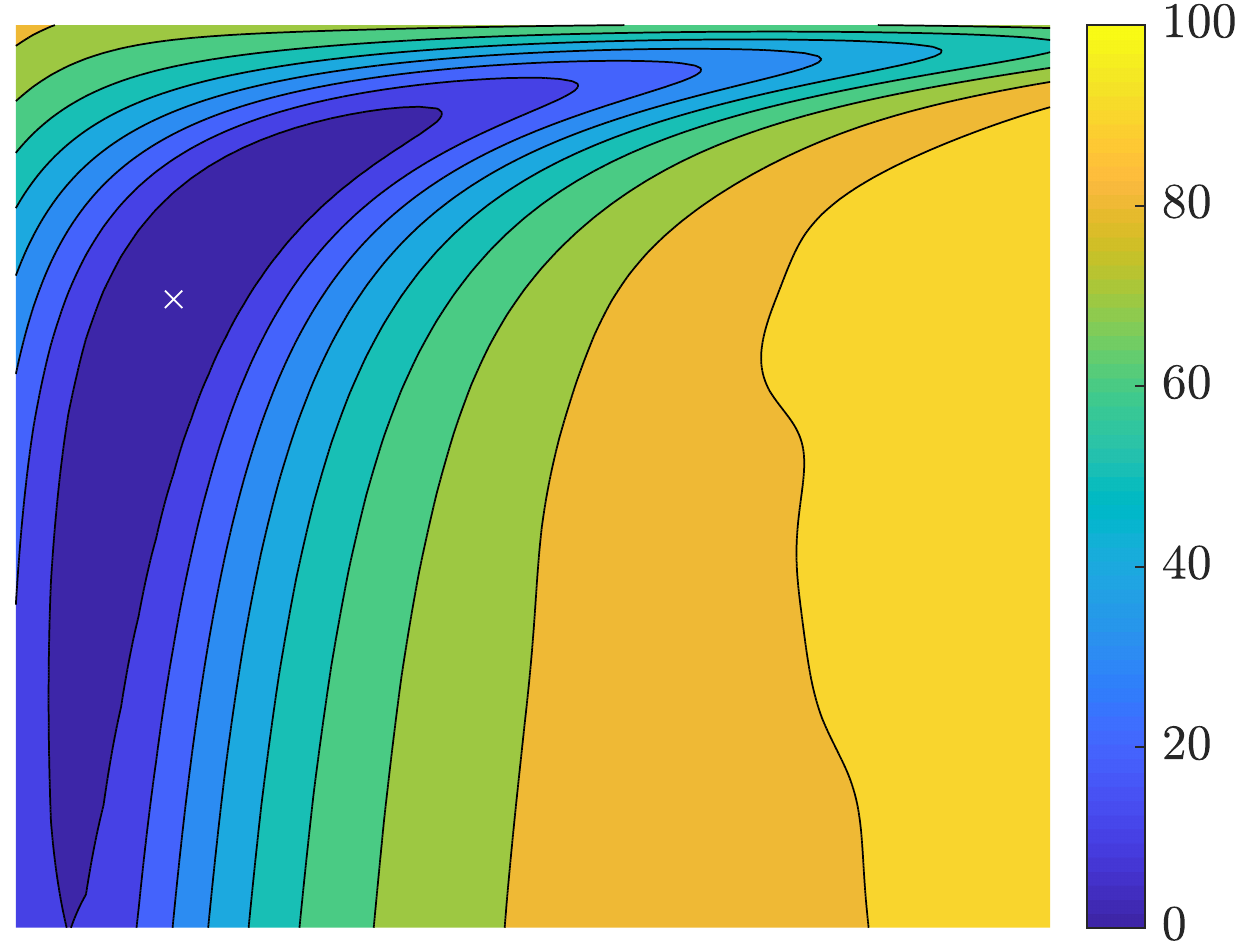}
        \label{fig:objCont22}
    \end{subfigure}
    \begin{subfigure}{0.35\linewidth}
        \includegraphics[width=\linewidth]{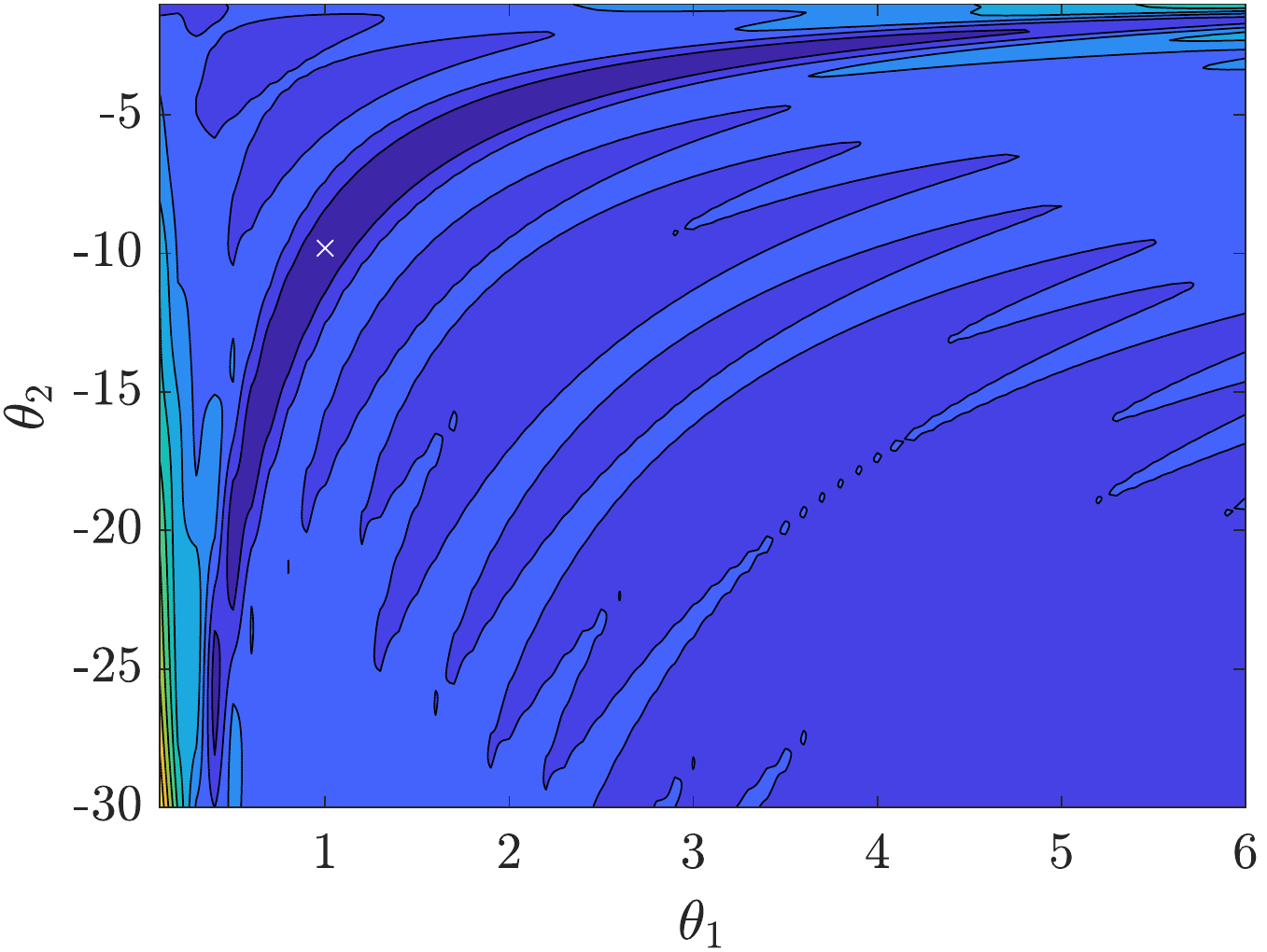}
        \caption{No Process Noise}
        \label{fig:objCont31}
    \end{subfigure}
    \begin{subfigure}{0.31\linewidth}
        \includegraphics[width=\linewidth]{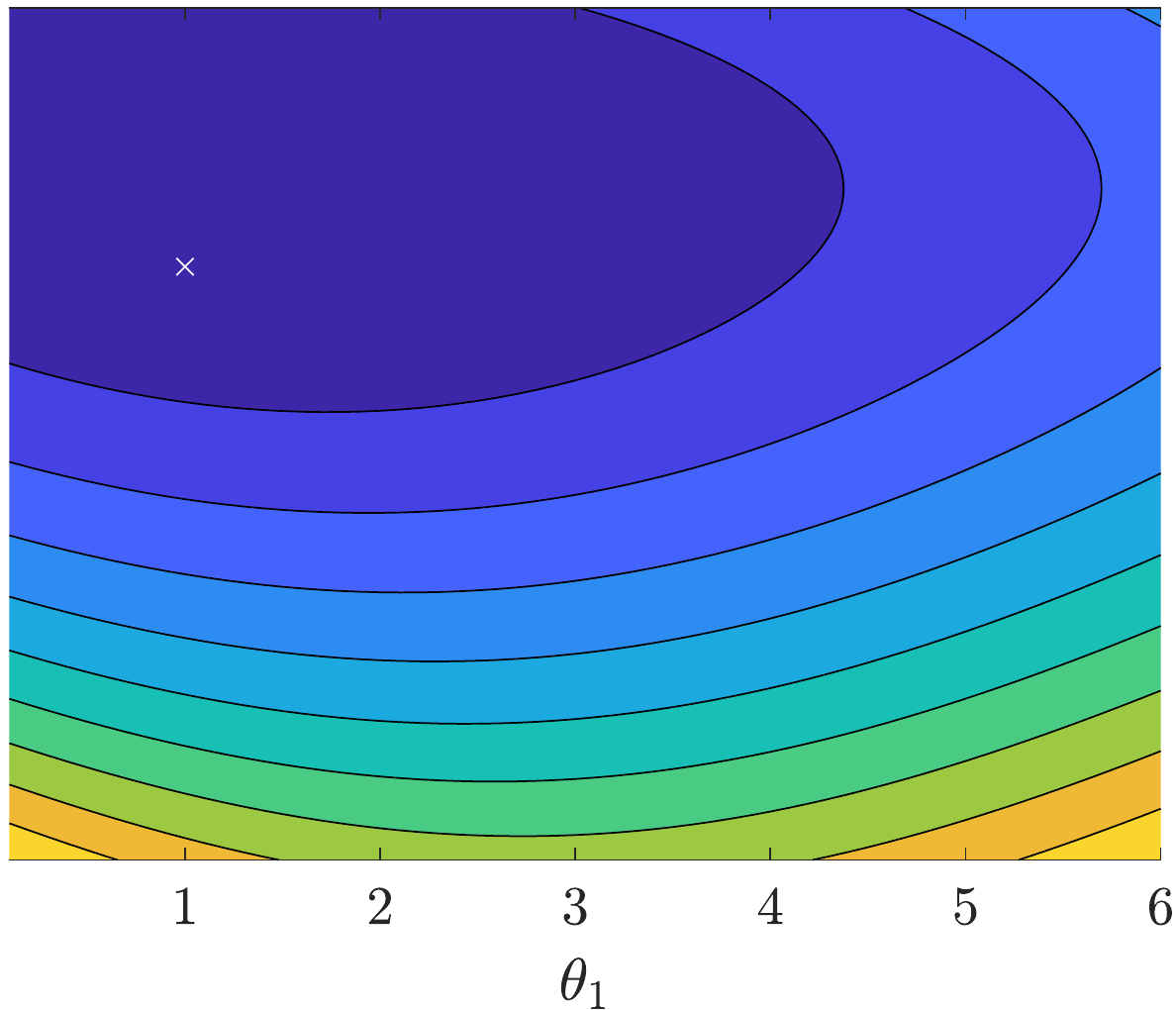}
        \caption{No Measurement Noise}
        \label{fig:objCont33}
    \end{subfigure}
    \begin{subfigure}{0.32\linewidth}
        \includegraphics[width=\linewidth]{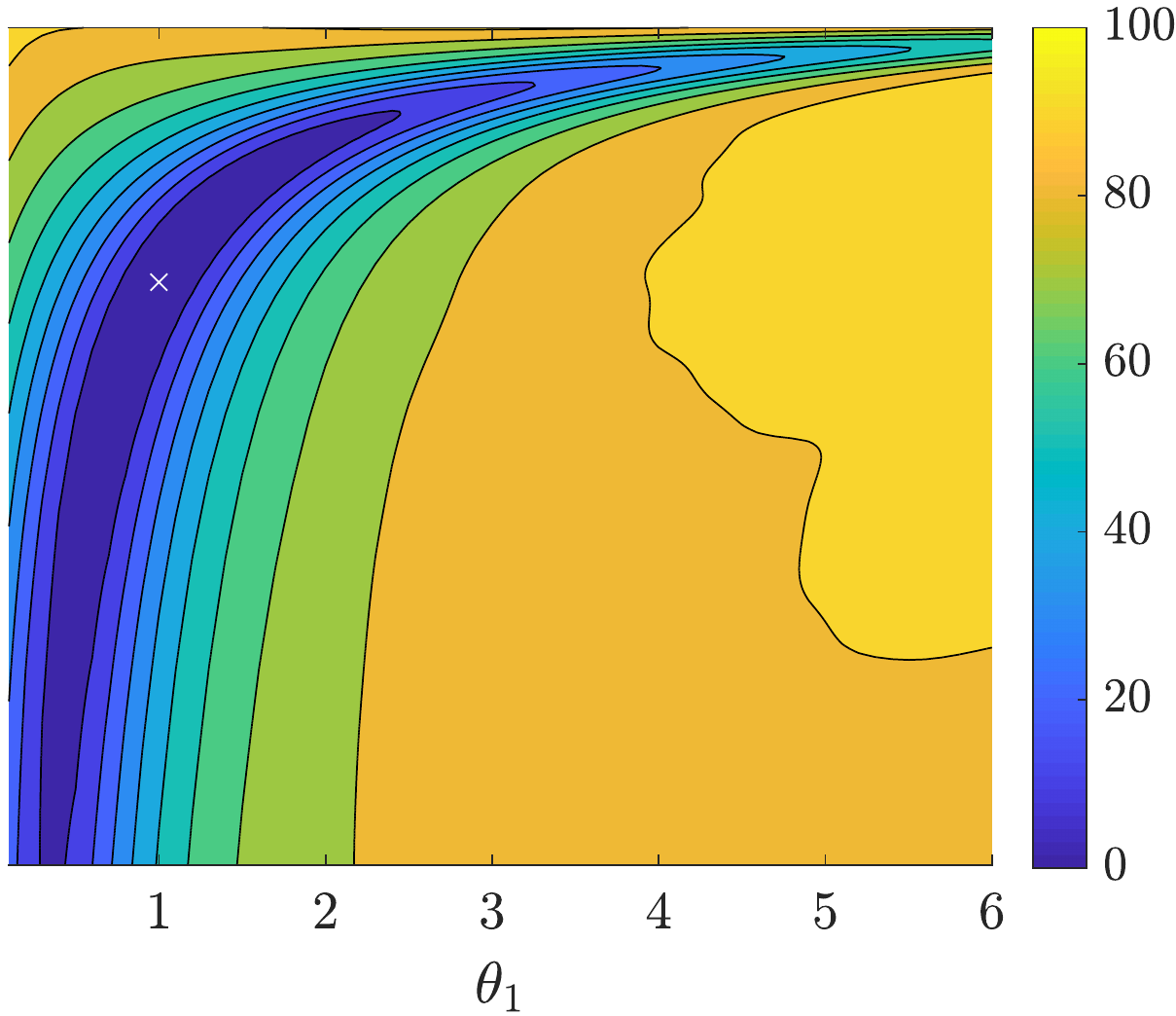}
        \caption{Noise Mixture}
        \label{fig:objCont32}
    \end{subfigure}
    \caption{Comparison of three optimization objectives for the identification of a linear pendulum. The left column uses a least squares objective that neglects process noise, the middle column uses a least squares objective that neglects measurement noise, and the right column is the log marginal likelihood that accounts for both. The rows correspond to the objective functions obtained after 20, 40, and 80 data points are taken at 0.1 second intervals from top to bottom. White crosses indicate true parameters. Neglecting process noise results in many local minima. Neglecting measurement noise results in an objective insensitive to the number of data. The Bayesian approach results in the ideal scenario where the objective becomes steeper in the direction of the minimum as the amount of data increases.}
    \label{fig:objCont}
\end{figure}

\section{General problem setting}\label{sec:problem}

In this section, we describe the probabilistic framework for system identification, the problem statement, and a description of our high-level solution approach.

\subsection{Probabilistic formulation}

In this section, we describe the probabilistic inference problem.

Let $\reals$ denote the set of reals and $\posint$ denote the set of positive integers. Let us define the norm of a vector as $\lVert a \rVert_C^2 = a^TC^{-1}a$. Let $(\Omega, \mathcal{F}, \prob)$ denote a probability space, $\dimx \in \posint$ denote the size of a state space, $\dimy \in \posint$ denote the size of an observation space, and $k \in \posint$ denote a time index corresponding to a time $t_k \geq 0$. Sequential time indices will typically occur with a constant interval $\Delta$ so that $t_k = t_{k-1} + \Delta$. 

We model the dynamical systems as evolving hidden/uncertain states $X_k \in \reals^\dimx$ through a discrete-time dynamical system. We obtain information about these states through a noisy measurement operator providing us data $y_{k} \in \reals^{\dimy}.$ These data can be viewed as realizations of another stochastic process $Y_k$ that is dependent on the hidden states.  

The dynamics and measurement operators are uncertain and the parameters $\theta \in \reals^\dimpar$ for $\dimpar \in \posint$ define a search space over which we will seek to learn the system.  We partition the parameters  $\theta = (\thetdyn, \thetobs, \thetsig, \thetgam)$ into different aspects of the problem including the dynamics model parameters $\thetdyn$, observation model parameters $\thetobs$, process noise parameters $\thetsig$, and observation noise parameters $\thetgam.$ Together these states, observations, and parameters are related through a hidden Markov model describing a discrete-time stochastic process~\cite{Sarkka_2013}
\begin{equation}
\begin{aligned}
  X_k &= \Psi(X_{k-1},\theta_{\Psi}) + \xi_k; &\xi_k \sim \mathcal{N}(0,\Sigma(\theta_{\Sigma})) \\ 
  Y_{k} &= h(X_{k}, \theta_{h}) + \eta_{k};  &\eta_k \sim \mathcal{N}(0,\Gamma(\theta_{\Gamma})), 
  \end{aligned} 
  \qquad \text{for} \quad   k = 1,\ldots, \numObs
   \label{eq:sys}
\end{equation}
where $\Psi: \reals^\dimx \times \reals^\dimpar \to \reals^\dimx$ is the dynamics operator, $\xi_{k}$ is the process noise with uncertain covariance $\Sigma(\thetsig)$, $h: \reals^\dimx \times \reals^\dimpar \to \reals^\dimy$ is the observation/measurement operator, $Y_{k}$ is the predictive stochastic process for the observable, and $\eta_{k} $ is the observation noise with uncertain covariance $\Gamma(\thetgam).$ Finally, we have an additional source of uncertainty corresponding to the initial condition of the states $X_0.$  A visual representation, in the form of a Bayesian network of this model is provided by Figure~\ref{fig:sys}.

Examples of $\Psi$ could include physics-inspired PDE operators~\cite{Wu_2020},  empirical linear models (a matrix), or nonlinear models such as neural networks~\cite{Chen_2018, Tsoulos_2009, Lagaris_1998}. The observation operator $h$ is typically some known sensor model that may or may not have uncertain calibration parameters $\thetobs$. We include the parameters for the observation model to maintain generality.

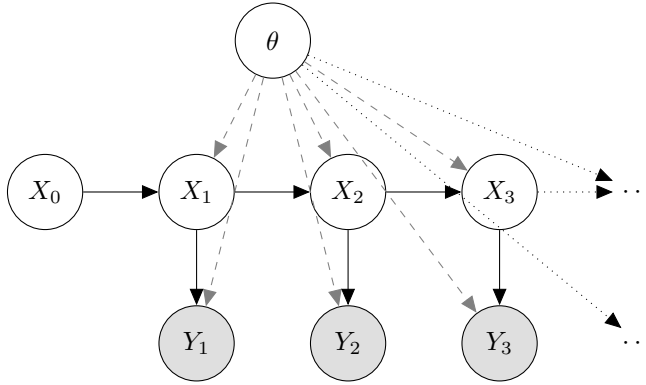
\begin{figure}
\centering
  \begin{tikzpicture}
      \node[obs, minimum width=1cm] at (0, 0) (d1) {$Y_1$};
      \node[obs, minimum width=1cm,right=of d1] (d2) {$Y_2$};
      \node[obs, minimum width=1cm, right=of d2] (d3) {$Y_3$};
      \node[right=of d3] (yn) {$\cdots$};

      \node[latent, minimum width=1cm, above=of d1] (x1) {$X_1$};
      \node[latent, minimum width=1cm, above=of d2] (x2) {$X_2$};
      \node[latent, minimum width=1cm, above=of d3] (x3) {$X_3$};
      \node[latent, minimum width=1cm, left=of x1] (x0) {$X_0$};
      \node[right=of x3] (xn) {$\cdots$};      

      \node[latent, minimum width=1cm, above=of x2, xshift=-1cm] (p1) {$\theta$};
      
      \draw[->, dashed, gray] (p1) -- (d1);
      \draw[->, dashed, gray] (p1) -- (d2);
      \draw[->, dashed, gray] (p1) -- (d3);
      \draw[->,dotted] (p1) -- (yn);      

      \draw[->] (x1) -- (d1);
      \draw[->] (x2) -- (d2);
      \draw[->] (x3) -- (d3);

      \draw[->] (x0) -- (x1);
      \draw[->] (x1) -- (x2);
      \draw[->] (x2) -- (x3);
      \draw[->,dotted] (x3) -- (xn);
      
      \draw[->, dashed, gray] (p1) -- (x1);
      \draw[->, dashed, gray] (p1) -- (x2);
      \draw[->, dashed, gray] (p1) -- (x3);
      \draw[->,dotted] (p1) -- (xn);
  \end{tikzpicture}
\caption{Bayesian network representation of the system identification problem. The data are realizations of the random variables $Y_k$.}
\label{fig:sys}
\end{figure}

System~\eqref{eq:sys} implicitly defines several probability distributions that completely describe our state of knowledge. The first distribution reflects the Markovian propagation dynamics
\begin{align}
    \probd(X_{k} \mid X_{k-1}, \thetdyn, \thetsig) &= \frac{1}{\sqrt{2\pi}^\dimx \lvert \Sigma(\thetsig)\rvert^{\frac{1}{2}}}\exp\left(-\frac{1}{2} e(X_k, X_{k-1},\thetdyn, \thetsig)\right) 
    \label{eq:prob_states} \\
    e(X_k, X_{k-1},\thetdyn, \thetsig) &= \lVert X_k - \Psi(X_{k-1}, \thetdyn) \rVert_{\Sigma(\thetsig)}^2,
    \label{eq:states_error}
\end{align}
where the error function $e$ represents the \textit{misfit}, or model error, of the dynamics under a fixed set of parameters. 

The next distribution reflects the noisy measurement models
\begin{align}
\probd(Y_k \mid X_{k}, \thetobs, \thetgam) &= \frac{1}{\sqrt{2\pi}^\dimy \lvert \Gamma(\thetgam)\rvert^{\frac{1}{2}}}
\exp\left(-\frac{1}{2} r(X_k, \thetobs, \thetgam)\right)
\label{eq:prob_meas} \\
    r(X_k, \thetobs, \thetgam) &= \lVert Y_k - h(X_k, \thetobs) \rVert_{\Gamma(\thetgam)}^2,
    \label{eq:meas_residual}
\end{align}
where the residual $r$ represents the misfit between the states and the observed measurements. Together, along with a prior, these distributions will enable us to concretely form the learning problem, which we establish in the next section.

\subsection{Goals}

In this section, we describe our two objectives: system identification (learning) and prediction/forecasting. 

Our learning objective is to determine a dynamical model $\Psi$. Specifically, this objective requires representing our knowledge about the parameters $\theta$ (or $\thetdyn$) after data are obtained. This knowledge is represented via a conditional distribution over $\theta$ given the observed data. This distribution is given by Bayes' rule 
\begin{equation}
    \probd(\theta \mid \yn) = \probd(\theta) \frac{\like(\theta ; \yn)}{\probd(\yn)},
    \quad \textrm{ where } \yn = \left(y_{1}, \ldots, y_{\numObs}\right),
    \label{eq:target_conditional}
\end{equation}
where the prior is denoted by $\probd(\theta)$ and the marginal likelihood is a function of the unknown parameter
\begin{equation}
    \like(\theta ; \yn) \equiv \probd(Y_1=y_1,\ldots,Y_n=y_\numObs \mid \theta).
\end{equation}
This conditional/posterior distribution captures all the relevant information about our parameters contained in the data.
It will be useful to leverage the sequential/Markovian nature of the process to factorize this likelihood as
\begin{equation}
     \like(\theta ; \yn) = \probd(Y_1=y_1 \mid \theta) \prod_{k=2}^\numObs \probd\left(Y_k=y_{k} \mid \theta, \mathcal{Y}_{k-1}\right) = \like_1(\theta; \mathcal{Y}_1)\prod_{k=2}^\numObs \like_i\left(\theta; \mathcal{Y}_i\right),
     \label{eq:marg_like_frac}
\end{equation}
where we have set $\like_1(\theta; \mathcal{Y}_1) \equiv \probd(Y_1=y_1 \mid \theta)$ and
$\like_k(\theta ; \mathcal{Y}_k) \equiv \probd(Y_k=y_{k} \mid \theta, \mathcal{Y}_{k-1})$ for $k=2,\ldots,\numObs$.

Our second goal is to predict, or forecast,  the system state at some future time $t_k$. This prediction could either be the full posterior predictive distribution $\probd(X_k \mid \yn)$ or some ``best estimate'' $X_k^*$ that can be derived from the posterior to satisfy some optimality conditions~\cite{Berger_1985}.  Furthermore, these two goals (system identification and prediction) are related in that the prediction is obtained by averaging over all possible system parameters, weighted according to the posterior distribution,
\begin{equation}
    \probd(X_k \mid \yn) = \int \probd(X_{k} \mid \theta) \probd(\theta \mid \yn) d\theta.
    \label{eq:pred}
\end{equation}

\subsection{Marginal likelihood computation}
\label{subsec:marg_like_computation}
In this section, we review the formulas for computing the marginal likelihood used in the target distribution~\eqref{eq:target_conditional}. We first present the general case~\cite{Sarkka_2013}. Then, we specialize the general case into special cases with (1) zero process noise (model error is ignored) and (2) noiseless, invertible measurements (measurement error is ignored). The formulation for evaluating the marginal likelihood is provided by the result of Theorem~\ref{th:post}. 
\begin{theorem}[Marginal likelihood (Th. 12.1~\cite{Sarkka_2013})]\label{th:post}
Let $\yk \equiv \left\{y_{i}; i \leq k \right\}$ denote the set of all observations up to time $k$. Let the initial condition be uncertain with distribution $\probd(X_0 \mid \theta).$ Then the marginal likelihood~\eqref{eq:marg_like_frac} is defined recursively in three stages: prediction
\begin{equation}
    \probd( X_{k+1} \mid \theta, \yk) = 
    \frac{\int \exp\left(-\frac{1}{2} e(X_{k+1}, X_k,\thetdyn, \thetsig)\right) \probd(X_{k} \mid \theta,  \yk) dX_k}{\sqrt{2\pi}^\dimx \lvert \Sigma(\thetsig)\rvert^{\frac{1}{2}}}
    \label{eq:predict}
\end{equation}
update,
\begin{equation}
    \probd\left(X_{k+1} \mid \theta, \mathcal{Y}_{k+1} \right) 
    = \probd(X_{k+1} \mid \theta, \yk)      \frac{\exp\left(-\frac{1}{2} r(X_{k+1}, \thetobs, \thetgam)\right)}{\sqrt{2\pi}^\dimy \lvert \Gamma(\thetgam)\rvert^{\frac{1}{2}}\probd(Y_{k+1} \mid \theta, \yk)} \label{eq:update}
\end{equation}
and marginalization,
\begin{equation}
   \like_{k+1}(\theta \mid \mathcal{Y}_{k+1}) = \int \probd(X_{k+1} \mid \theta, \yk)      \frac{\exp\left(-\frac{1}{2} r(X_{k+1}, \thetobs, \thetgam)\right)}{\sqrt{2\pi}^\dimy \lvert \Gamma(\thetgam)\rvert^{\frac{1}{2}}} dX_{k+1} \label{eq:marg}
\end{equation}
for $k = 1,2,\ldots.$
\end{theorem}
This theorem provides a recursive algorithm for evaluating the marginal likelihood. This recursion requires not only maintaining a standard Bayesian filter for the prediction and update steps, but also keeping track of the marginalized distribution after every observation. Extensions to situations where data are not obtained at every time step is trivial -- for times when no data are obtained, the update step is skipped. 

\subsubsection{Zero process noise}
\label{subsubsec:zero_process_noise}
In the standard parameter estimation problem, e.g. Equation~\eqref{eq:detobj}, we model the dynamics as deterministic where uncertainty enters only through  measurement noise.  In this case, the Markovian property that leads to distribution~\eqref{eq:prob_states} reduces to a Dirac delta function
\begin{equation}
    \probd(X_k\mid X_{k-1},\thetdyn,\thetsig) = \delta_{X_k}(\Psi(X_{k-1},\thetdyn)).
    \label{eq:probx_det}
\end{equation}

The assumption of zero process noise leads to the marginal likelihood given in Theorem~\ref{th:marg_like_determ}.

\begin{theorem}[Marginal likelihood -- zero process noise]\label{th:marg_like_determ}
Let the dynamics model be deterministic. Then, the marginal likelihood~\eqref{eq:marg_like_frac} is defined recursively as 
\begin{equation}
    \like_k(\theta;\yk) = \frac{\exp\left(-\frac{1}{2}\lVert y_k-h(\Psi^k(X_0,\thetdyn),\thetobs)\rVert_{\Gamma(\thetgam)}^2\right)}{\sqrt{2\pi}^\dimy\lvert\Gamma(\thetgam)\rvert^{\frac{1}{2}}} \quad  \text{for } k=1,\ldots, \numObs, 
\end{equation}
where $\Psi^{k}$ denotes $k$ applications of the dynamics model. Moreover the log marginal likelihood becomes
\begin{equation}
    \log\like(\theta;\yn) = \sum_{k=1}^\numObs \left(-\frac{1}{2}\lVert y_k-h(\Psi^k(X_0,\thetdyn),\thetobs)\rVert_{\Gamma(\thetgam)}^2\right) - \frac{\numObs\dimy}{2}\log 2\pi - \frac{\numObs}{2}\log\lvert\Gamma(\thetgam)\rvert. \label{eq:marg_like_det}
\end{equation}
\end{theorem}
\begin{proof}
The proof follows from the fact that a deterministic system must follow a fixed trajectory defined entirely by the parameters. In other words, we have $\probd(\xn \mid \thetdyn, \yn) = \probd(\xn \mid \thetdyn) = \delta_{\Psi(X_0,\thetdyn),\ldots,\Psi^{n}(X_{0},\thetdyn)}(\xn)$. As a result, the measurement model can be written as a function of the parameters 
\begin{equation*}
    \probd(Y_k \mid X_{k}, \theta) = \probd(Y_{k} \mid \Psi^k(x_0,\thetdyn), \thetobs, \thetgam) = \probd(Y_k \mid \theta).
\end{equation*}
This distribution is Gaussian with mean $h(\Psi^k(x_0, \thetdyn), \thetobs)$ and covariance $\Gamma(\thetgam).$ Applying these same facts at each time step completes the proof.
\end{proof}

\subsubsection{Noiseless and invertible measurements}

In this section, we consider the ramifications on the posterior of assuming no measurement noise. In the next section we will show that several least squares optimization approaches correspond to this case.

Consider an invertible observation operator so that the states are uniquely determined $X_k = h^{-1}(Y_k)$. Using this assumption in System~\eqref{eq:sys} leads to a Markovian system for the system observables
\begin{equation}
    Y_{k+1} = h\left(\Psi\left(h^{-1}(Y_k), \thetdyn\right) + \xi_k, \thetobs\right)  \quad \textrm{for } k = 1,\ldots,\numObs-1,
    \label{eq:sys_noiseless}
\end{equation}
where $\xi_k \sim  \mathcal{N}(0, \Sigma(\thetsig)).$

This assumption yields the marginal likelihood given in Theorem~\ref{th:marg_like_noiseless} below.
\begin{theorem}[Marginal likelihood -- noiseless, invertible observations]\label{th:marg_like_noiseless}
Let $h$ be an invertible operator and the measurements be noiseless. Then, the marginal likelihood~\eqref{eq:marg_like_frac} is defined recursively as 
\begin{equation}
  \like_k(\theta ; \yk)  = \lvert  \nabla  h^{-1}(y_k) \rvert \frac{\exp\left(-\frac{1}{2} \lVert h^{-1}(y_{k}) - \Psi\left(h^{-1}\left(y_{k-1}\right), \thetdyn \right)\rVert_{\Sigma(\thetsig)}^2\right)}{\sqrt{2\pi}^{\dimx}\lvert \Sigma(\thetsig)\rvert^{\frac{1}{2}}}
\end{equation}
for $k=2,\ldots, \numObs$ and
\begin{equation}
\begin{aligned}
    \log\like_1(\theta; \mathcal{Y}_1) &=
    \log \int \exp\left( \lVert h^{-1}(y_1) - \Psi(X_0; \thetdyn) \rVert_{\Sigma(\thetsig)}^2\right) \probd(X_0 \mid \theta) dX_0 - \\
    &\qquad \frac{\dimx}{2}\log 2\pi -\frac{1}{2} \log \lvert \Sigma(\thetsig)\rvert.
    \end{aligned}
\end{equation}
Together, the log marginal likelihood becomes
\begin{equation}
    \begin{aligned}
    \log \like(\theta ; \yn) &=  \sum_{k=2}^{\numObs} \left( \log \lvert  \nabla  h^{-1}(y_k) \rvert  - \frac{1}{2}\lVert h^{-1}(y_{k}) - \Psi\left(h^{-1}\left(y_{k-1}\right), \thetdyn \right)\rVert_{\Sigma(\thetsig)}^2   \right) - \\ &\qquad \frac{\numObs\dimx}{2}\log 2\pi -  \frac{\numObs}{2}\log \lvert \Sigma(\thetsig)\rvert  +  \log\like_1(\theta; \mathcal{Y}_1).
    \end{aligned}
    \label{eq:log_marg_like}
\end{equation}
\end{theorem}

\begin{proof}
The proof is trivial by noticing that the Markovian system for the observables~\eqref{eq:sys_noiseless} defines a sampling distribution using the change of variables formula
\begin{equation*}
    \probd(Y_k\mid\theta) = \lvert\nabla h^{-1}(Y_k)\rvert\frac{\exp\left(-\frac{1}{2}\lVert h^{-1}(Y_k)-\Psi(h^{-1}(Y_{k-1}),\thetdyn)\rVert_{\Sigma(\thetsig)}^2\right)}{\sqrt{2\pi}^\dimx\lvert\Sigma(\thetsig)\rvert^{\frac{1}{2}}}.
\end{equation*}
The given result is obtained by using this sampling distribution as the likelihood.
\end{proof}

As we intuitively expected, there is no marginalization over states under this assumption because the learning problem effectively ``resets'' after every data point. After the reset, the states are at their true value, and optimization progresses to ensure that the residual of propagation between true values is small. This is exactly the same methodology that inspires the least squares regression based approaches such as DMD and SINDy. In fact, we will show that special assumptions on $h$ and $\Psi$ recover these least squares approaches.

\begin{remark}[Data on initial condition]
If the initial condition is treated as beginning when the data are obtained, then the log likelihood for the first data point becomes independent of the parameters and we can set it to an arbitrary constant.
\end{remark}

\subsection{Decision making}
\label{subsec:decision_making}
Whereas traditional system ID and ML approaches define the problem through an optimization objective, the Bayesian approach separates learning and decision making. In effect, it provides a way of generating new optimization objectives and interpreting existing ones. Here, we briefly comment on the fact that this separation comes in the form of a two step procedure: (1) computing the posterior and (2) extracting a goal-oriented estimator through the specification of a loss function. For detailed discussion of these topics we refer the reader to~\cite{Berger_1985}.

First note that we have considered $\theta$ to contain all uncertain parameters in the problem. For prediction, however, it is standard to make predictions into the future using deterministic models based on $\Psi$. As a result, we can partition the parameters $\theta = \left(\theta_\Psi, \theta_h, \theta_{\Sigma}, \theta_{\Gamma}\right)$ into those that correspond to the dynamics, observations, process noise, and measurement noise, respectively. 
Next we define the \textit{posterior predictive} distribution of the states as an average over all possible values of the dynamics parameters conditioned on the observations
\begin{equation}
    \probd(X_k \mid \yn)  = \int \probd(X_k \mid \theta_{\Psi}) \probd(\theta_{\Psi} \mid \yn) d\theta_{\Psi},
\end{equation}
where we will use a deterministic prediction that discards the process noise
\begin{equation}
    \probd(X_k \mid \theta_{\Psi}) = \delta_{\Psi^{k}(x_0, \thetdyn)}(X_{k}).
\end{equation}
This restriction is not explicitly necessary, but it is representative of how learned models are used in practice.

Finally, we can extract several estimators to use as ``point estimates'' from the posterior. For example, the mean estimator, which corresponds to the optimal estimator for the squared loss~\cite{Berger_1985},
\begin{equation}
    X_k^{\avg} = \ee_{\thetdyn \mid \yn} \left[\probd(X_{k} \mid \yn)\right],
\end{equation}
or the MAP estimator,
\begin{equation}
    X_k^{\map} = \argmax_{\tilde{X}_k} \ \probd( \tilde{X}_k \mid \yn).
\end{equation}

Note that these estimators do not, in general, have a 1-1 correspondence with the approaches that use the mean, MAP, and median of the \textit{parameters} rather than the states, i.e., the mean estimator
\begin{equation}
    \theta^{\tavg} = \ee_{\theta \mid \yn}\left[\probd(\theta \mid \yn)\right],
\end{equation}
or the MAP estimator
\begin{equation}
    X_k^{\tmap} = \Psi^{k}(x_0, \thetdyn^*); \quad \thetdyn^* =  \argmax_{\theta} \probd(\theta \mid \yn).
\end{equation}
In other words, we do not require a fixed estimator for the model to have a point estimate of the prediction. Indeed, the most likely dynamics may not actually correspond to the most likely states for nonlinear models.

\section{Analysis of existing approaches}
\label{sec:existing_approaches}

In this section, we analyze two common state-space system ID approaches that have recently garnered some success. These approaches are representative of those that ignore measurement noise (or account for it by heuristic ``denoising''). We seek to demonstrate that many least squares approaches can be interpreted within the framework of Equation~\eqref{eq:sys}.  Our choice of dynamic mode decomposition (DMD) and sparse identification of nonlinear dynamics (SINDy) are representative of algorithms that use least squares and/or regularization. Our main results are that DMD can be interpreted as a maximum likelihood estimate and SINDy as a MAP estimate  under zero-noise and invertible observation operator assumptions. 

\subsection{Dynamic mode decomposition (DMD)}
\label{subsec:dmd}
Dynamic mode decomposition (DMD) is a data-driven method for system identification that is used to identify the `dynamic modes' of a dynamical system~\cite{schmid_2010}. These modes reveal characteristics such as unstable growth modes, resonance, and spectral properties~\cite{Proctor_2014}.  DMD is favorable when the system at hand is high dimensional but has some hidden low-dimensional structure, as is the case in many fluids problems. DMD first organizes a series of measurements at regular time intervals into two matrices
\begin{align}
    \begin{array}{lr}
        Y = \begin{bmatrix}y_1 & y_2 & \dots & y_{\numObs-1} \end{bmatrix}; &
        Y' = \begin{bmatrix}y_2 & y_3 & \dots & y_\numObs \end{bmatrix},
    \end{array}
\end{align}
and then seeks a linear operator $A$ which maps the observables from one time step to the Next i.e., $Y'=AY$. To find $A$, one simply minimizes the Frobenius norm of $AY-Y'$ by solving the least squares problem
\begin{align}
   A = \argmin_{\tilde{A}} \sum_{k=2}^{\numObs}\lVert y_{k} - \tilde{A}y_{k-1}\rVert^2.
    \label{eq:dmd_ls}
\end{align}
The solution is given by $A = Y'Y^\dagger$, where $\dagger$ denotes the pseudo-inverse.

The method given above may at first appear only applicable to linear systems, but~\cite{Rowley_2009} showed that in the nonlinear case, the approximated operator $A$ and its corresponding modes are approximations to the linear but infinite-dimensional Koopman operator and Koopman modes respectively, thus revealing its applicability to nonlinear systems.

Next we show the least squares procedure for DMD can also be \textit{derived} directly from the general probabilistic system~\eqref{eq:sys} under certain assumptions.
\begin{theorem}[DMD as a maximum likelihood of system~\eqref{eq:sys}]\label{th:DMD}
Assume a linear model $\Psi(X_k, \thetdyn) = \thetdyn X_k$; identity observation operator $h = I$; noiseless measurements $\Gamma(\thetgam) = 0$; and identity process noise $\Sigma(\thetsig) = I.$ Then, the maximum marginal likelihood estimator corresponding to System~\eqref{eq:sys} is equivalent to the least squares objective of the DMD problem~\eqref{eq:dmd_ls}.

\end{theorem}
\begin{proof}
This result uses a straightforward application of Theorem~\ref{th:marg_like_noiseless}. Without loss of generality, we use the fact that the first measurement is of the initial condition, and therefore we can ignore $\like_1.$ Here, we have an identity observation operator, and therefore the inverse and Jacobian are also the identity. The dynamics are linear and unknown so we can write $A  \equiv \thetdyn$. Together, these facts require that the log marginal likelihood~\eqref{eq:log_marg_like} becomes
\begin{equation}
\log \like(\theta ; \yn) =  -\sum_{k=2}^{\numObs} \frac{1}{2}\lVert y_{k} - A y_{k-1} \rVert^2 -\frac{\numObs\dimx}{2}\log 2\pi,
\end{equation}
which is our stated result. Clearly, the maximizer of this function is equivalent to the minimizer of~\eqref{eq:dmd_ls}.
\end{proof}

While the invertible measurement operator is not a restrictive assumption because all DMD cares about is mapping observables and not underlying states, Theorem~\ref{th:DMD} shows why DMD may not be appropriate for cases where the observations are noisy. This fact has been recognized in the literature and several procedures for rectifying this issue have been proposed.  For instance,~\cite{Hemati_2017} showed that total least squares is a more appropriate algorithm to identify $A$ when measurement noise is present, a method known as \textit{total DMD} (TDMD). For a full analysis of the total least squares problem, see~\cite{Golub_1980,Huffel_1989}. We will empirically compare TDMD to our approach in Section~\ref{sec:experiments}, where we see that it also performs worse than the posterior predictive mean. Future work will attempt to determine the assumptions that TDMD makes in the context of System~\eqref{eq:sys}.

In \cite{Takeishi_2017}, another connection between the Bayesian approach to DMD was developed that infers the Koopman modes and eigenfunctions of the Koopman operator directly, rather than learning the dynamical operator itself.  That work showed that when the measurements are noiseless, the MLE of their Bayesian model, TDMD, and DMD all provide the same estimate.  In contrast, here we have provided our result in terms of the underlying hidden state dynamics rather than explicitly assuming observation dynamics. 

One benefit of the analysis in our context is that our use of an underlying state-space model makes the framework valid even when the observations cannot be written using a Markovian (zero-lag) model as in Equation~\eqref{eq:sys_noiseless}, which was required for the approach developed in~\cite{Takeishi_2017}. In fact, this result can be interpreted to indicate that zero-lag DMD is most effective if the observation operator is invertible.

\subsection{Regularized regression for nonlinear models}
\label{subsec:sindy}
Least squares optimization can also be used for identifying nonlinear systems by searching in a linear subspace. In these cases, it is often advantageous to add regularization to seek parsimonious solutions. One such approach that uses a sparsity enhancing regularization is the method of 
 sparse regression or sparse identification of nonlinear dynamics (SINDy)~\cite{Brunton_2016}.  
 
 These approaches organize a library of candidate functions (linear and nonlinear) into a matrix. They then aim to approximate the time derivative in the span of this library. For instance,
\begin{align}
    \dot{x} = f(x) \approx \begin{bmatrix}1 & x & x^2 & \dots & x^\dimpar\end{bmatrix}
    \begin{bmatrix}\theta_0 \\ \theta_1 \\ \vdots \\ \theta_\dimpar\end{bmatrix}.
    \label{eq:vec}
\end{align}
This example uses monomial candidate functions, but any basis (wavelets, orthogonal polynomials, empirical bases) can be used. 

Suppose that the general dictionary of terms is given by $\Xi:\reals^{\dimx} \to \reals^{\dimx}$ so that the deterministic portion of some continuous-time autonomous dynamics can be written as a linear system with respect to the parameters/coefficients of the functions in the dictionary $\dot{x} = \Xi(x)\thetdyn.$ If direct data were available on the states and derivatives, one might then try to solve a (regularized) linear least squares problem for the parameters
\begin{equation}
    \thetdyn = \argmin_{\tilde{\theta}} \sum_{k=1}^\numObs \lVert \dot{x}_k - \Xi(x_k) \tilde{\theta} \rVert_2^2 + \lambda \lVert \tilde{\theta} \rVert,
\end{equation}
where $\lambda$ is a regularization weight and the norm can be chosen by the user. If the $L_1$ norm is chosen, this becomes a sparse regression problem. 

Practical applications, however, do not have data on the derivative of each state $\dot{x}_i$. As a result, various numerical approximations can be made, and this is the approach taken by the SINDy algorithm. Here, we will consider one type of numerical approximation to the derivative, but our analysis can be extended to others. If a forward-difference approximation to the time derivative is taken, then the SINDy objective function is 
\begin{equation}
    \thetdyn = \argmin_{\tilde{\theta}} \sum_{k=2}^\numObs \left \lVert \frac{x_k - x_{k-1}}{\Delta t} - \Xi(x_{k-1}) \tilde{\theta} \right \rVert_2^2 + \lambda \lVert \tilde{\theta} \rVert.
    \label{eq:sindy}
\end{equation}
Notice that this approach requires direct observation of the states. Next we show that it is also equivalent to the maximum \textit{a posteriori} of our target conditional distribution under more strict assumptions.

\begin{theorem}[SINDy as a maximum \textit{a posteriori} estimate of system~\eqref{eq:sys}]\label{th:sindy}
Let $\Xi(x): \reals^\dimx \to \reals^\dimx$ denote a library of candidate functions for continuous time drift dynamics. Let $\Psi(x; \thetdyn)$ denote the resulting discrete-time operator that uses a forward-Euler integration scheme
\begin{equation}
    \Psi(X, \thetdyn) = X + \Delta t \Xi(X) \thetdyn.
\end{equation}
Furthermore, assume an identity observation operator $h=I$; noiseless measurements $\Gamma(\thetgam) = 0$; identity process noise $\Sigma(\thetsig)=I$; and a Laplace prior $\probd(\thetdyn) \propto \exp\left(-\tilde{\lambda}|\thetdyn|\right).$ Then, the MAP estimate of the conditional distribution given in Equation~\eqref{eq:target_conditional} is equivalent to the SINDy estimator obtained by minimizing~\eqref{eq:sindy}.

\end{theorem}
\begin{proof}
This proof is again a straightforward application of Theorem~\ref{th:marg_like_noiseless}. Recall that the data are taken on the initial condition, and note that we have $Y_k = X_k.$  The log-marginal likelihood~\eqref{eq:log_marg_like} is then
\begin{align}
    \log \like(\theta; \yn) &= - \frac{1}{2} \sum_{k=2}^\numObs\left  \lVert y_k -  \left(y_{k-1} + \Delta t \Xi(y_{k-1}) \thetdyn\right) \right \rVert_{2}^2 - \frac{\numObs\dimx}{2}\log 2\pi \\ 
                            &=  -\frac{\Delta t}{2} \sum_{k=2}^\numObs  \left \lVert \frac{y_k - y_{k-1}}{\Delta t} - \Xi(y_{k-1})\thetdyn \right \rVert_2^2 - \frac{\numObs\dimx}{2}\log 2\pi.
\end{align}
Then we can drop the parameter-independent term and add the log prior to obtain a posterior that is proportional to 
\begin{align}
    \log \probd(\theta; \yn) &\propto -\frac{\Delta t}{2} \sum_{k=2}^\numObs \left \lVert \frac{y_k - y_{k-1}}{\Delta t} - \Xi(y_{k-1})\thetdyn \right \rVert_2^2  - \tilde{\lambda} \lvert \thetdyn \rvert  \\
    &= -\frac{\Delta t}{2}\left( \sum_{k=2}^\numObs \left \lVert \frac{y_k - y_{k-1}}{\Delta t} - \Xi(y_{k-1})\thetdyn \right \rVert_2^2  + \frac{2\tilde{\lambda}}{\Delta t} \lvert \thetdyn \rvert \right).
\end{align}
Maximizing the posterior is equivalent to minimizing the term in the parentheses. By setting $\lambda \equiv \frac{2\tilde{\lambda}}{\Delta t}$, we see that this is the exact form of the SINDy objective~\eqref{eq:sindy}.
\end{proof}

\section{Algorithm and computational complexity}
\label{sec:algorithm}
In this section, we describe an approximate marginal MCMC approach that has recently been introduced and analyzed in parallel by several different fields~\cite{Erazo_2018, Noh_2019, Drovandi_2019, Khalil_2015}. This approach is fundamentally based on approximately evaluating the marginal likelihood described in Theorem~\ref{th:post}. 

\subsection{Algorithm}
Theorem~\ref{th:post} provides a recursive approach to evaluate the marginal likelihood that avoids computation of a high-dimensional integral, but this theorem still requires the evaluation of lower-dimensional integrals. In the linear case, the solution to these recursive integrals can be found using the Kalman filter, however no solution is available for general nonlinear systems.

When no closed-form solution exists for these integrals, nonlinear filtering techniques can be introduced. These can include ensemble Kalman filtering~\cite{Evensen_1994}, Gaussian filtering (including cubature Kalman filter~\cite{Arasaratnam_2009} and unscented Kalman filter~\cite{Julier_1997}), and particle filtering~\cite{Gordon_1993}. Of these filters, only the particle filter has been proven to enable an exact pseudomarginal MCMC scheme~\cite{Andrieu_2010}. The other schemes approximate the prediction, update, and marginalization equations --  yielding a (generally) biased estimate of the posterior. Nevertheless, they are often more computationally tractable and have empirically shown good performance. 

These algorithms embed these filters within the accept-reject step of Metropolis-Hastings MCMC scheme, as shown in Algorithm~\ref{alg:pMCMC}. We slightly modify the UKF-MCMC scheme of~\cite{Erazo_2018} by using delayed-rejection adaptive Metropolis MCMC~\cite{Haario_2006} instead of the standard Metropolis-Hastings MCMC. Specifically, the log posterior enters these schemes during the computation of the likelihood portion of the posterior
\begin{align}
    \alpha = \text{min}\left(1, \frac{\hat{\like}(\theta^*;\yn)\probd(\theta^*)}{\hat{\like}(\theta^{(k-1)};\yn)\probd(\theta^{(k-1)})}\frac{\pi(\theta^{(k-1)})}{\pi(\theta^*)}  \right),
\label{eq:marg_acc}
\end{align}
where $\pi(\theta)$ is the proposal distribution and $\hat{\like}(\theta;\yn)$ is the likelihood estimator. As we mentioned above, in the linear case we use a Kalman filter to exactly evaluate the marginal likelihood ($\hat{\like}(\theta;\yn)\equiv\like(\theta;\yn)$) . This algorithm is shown in Algorithm~\ref{alg:margpost_KF}. In the nonlinear case,  we approximate each distribution to be Gaussian and approximate the marginal posterior using an unscented Kalman filter (UKF) as shown in Algorithm~\ref{alg:margpost_UKF}.  In the UKF algorithm, $\alpha$ and $\kappa$ are parameters that determine the spread of the sigma points around the mean, $\beta$ is a parameter used for incorporating prior information on the distribution of $x$, and the notation $[\cdot]_i$ denotes the $i$-th row of the matrix~\cite{Sarkka_2013}.  

\begin{algorithm}
  \begin{algorithmic}[1]
    \REQUIRE  Prior distribution $\probd(\theta)$ \\
    \quad \quad UKF-based likelihood estimator $\hat{\like}(\theta; \yn)$ \\    
    \quad \quad Proposal distribution $\pi(\theta)$\\
    \quad \quad Initial sample $\theta^{(0)}$
    \ENSURE Samples from stationary distribution $\probd(\theta \mid \yn)$
    \STATE Compute $\hat{z}^{(0)} = \hat{\like}(\theta^{(0)};\yn)$
    \FOR {$k = 1$ to $N$}
    \STATE $\theta^{*} \sim \pi $ \quad Sample from proposal
    \STATE $z^{*} = \hat{\like}(\theta^{*};\yn)$ \quad Compute estimated likelihood
    \STATE Compute acceptance probability
    \begin{equation} 
      \alpha = \min\left(1, \frac{z^{*} \probd(\theta^{*})}{z^{(k-1)} \probd(\theta^{(k-1)})} \frac{\pi(\theta^{(k-1)})}{\pi(\theta^{*})}\right)
      \label{eq:ARprob}
    \end{equation}
    \STATE Accept $\theta^{(k)} =\theta^{*}$ and $z^{(k)} = z^{*}$ with probability $\alpha$; otherwise $\theta^{(k)} =\theta^{(k-1)}$ and $z^{(k)} = z^{(k-1)}$
    \ENDFOR
  \end{algorithmic}
\caption{Approximate marginal MCMC for Bayesian inference}
\label{alg:pMCMC}
\end{algorithm}

\subsection{Computational complexity}
We will show in Section~\ref{sec:experiments} that this approach yields more robust estimators than competing system ID approaches by accounting for measurement noise; however, this robustness will be at the cost of slightly increased computational complexity.
In this section, we assess the cost of the algorithm both in the linear case where the Kalman filter is used and the nonlinear case where the UKF is used by counting the number of floating-point operations (flops) required by each algorithm.  

For this analysis, addition, subtraction, multiplication, and division of two floating point numbers and the logarithm of one floating point number all count as one flop.  The multiplication of an $m\times n$ matrix by an $n\times p$ matrix then counts as $mp(2n-1)$ flops because each of the $mp$ entries of the product matrix requires $n$ multiplications and $n-1$ additions.\footnote{We only consider the naive matrix-multiplication scheme, not the asymptotically more optimal approaches such as Strassens algorithm.}  Similarly, the multiplication of an $m\times n$ matrix by an $n\times 1$ vector requires $n(2n-1)$ flops.  Additionally, we approximate the cost of a Cholesky decomposition, matrix inversion, and determinant performed on an $n\times n$ matrix all to be $n^3/3$ flops.  Furthermore, the complexity of these algorithms strongly depends on the complexity of the dynamical and measurement models used, which will vary from problem to problem.  For the sake of generality, we define the computational complexity of the dynamical model $\Psi$ and measurement model $h$ to be denoted as $F$ and $H$ respectively.  Clearly in the linear case, these variables will not be needed as the dynamical and measurement models are matrices, and the number of flops can be calculated without loss of generality.  The number of flops for each algorithm will be given in terms of the problem dimensions, so recall the following notation: $\dimx$ the dimension of the state, $\dimy$ the dimension of the measurements, $\dimpar$ the number of parameters, and $\numObs$ the total number of measurements available.

Our analysis focuses entirely on the computation of the marginal likelihood, which is the dominant cost of the MCMC algorithm. The complexity of the rest of the algorithm will depend on the complexity of the MCMC algorithm and prior selected by the user, but is typically orders of magnitude lower than the likelihood computation. In the following analysis, we provide results for the Kalman filtering algorithm, the unscented Kalman filtering algorithm, their prediction and update subcomponents, DMD, and sparse regression. Table~\ref{tab:numOps} shows the number of different types of operations required by each algorithm. Table~\ref{tab:numFlops} shows the number of flops for each algorithm where the computation of the regularization term in sparse regression is excluded.  Note that although the mean and covariance of the marginal likelihood are computed in the update step of the Bayesian algorithms, the computation of the log of this distribution is excluded from this step, and is instead included only in the total.  Also, the 18 flops outside the parentheses in the UKF total count comes from the formation of the weights, which is required only once at the beginning of the algorithm.  

In determining the number of flops used in DMD, we counted the number of flops needed to solve the normal equation $A=Y'Y^T(YY^T)^{-1}$ where $Y,Y'\in\mathbb{R}^{\dimy\times\numObs-1}$.  Similarly, sparse regression was considered to be the computation $\Theta=(\Xi^T\Xi)^{-1}\Xi^T\dot{X}$, where $\Xi\in\mathbb{R}^{n-1\times p/m}$ and $\dot{X}\in\mathbb{R}^{n-1\times m}$.  In practice, this computation is performed multiple times with an increasingly small $\Xi$ matrix, but for this analysis, only one iteration of the optimization procedure is considered.  To execute TDMD, a singular value decomposition (SVD) of the concatenated matrix $\begin{bmatrix}Y^T & Y'^T\end{bmatrix}\in\mathbb{R}^{n-1\times 2m}$ is first performed, which has computational complexity on the order of $\mathcal{O}(m^2n+n^2m+m^3)$.  The solution of the total least squares problem is then given by $A=-V_1V_2^T(V_2V_2^T)^{-1}$. Let $r$ be the rank of matrix $\begin{bmatrix}Y^T & Y'^T\end{bmatrix}$.  Then, $V_1\in\mathbb{R}^{m\times 2m-r}$ is a matrix composed of the first $m$ rows of the last $2m-r$ right singular vectors, and $V_2\in\mathbb{R}^{m\times 2m-r}$ is a matrix composed of the last $m$ rows of the last $2m-r$ right singular vectors.  The computational complexity of this least squares problem is then $\frac{19}{3}m^3-2m^2r-2m^2$.  Since $m=d$ in the case of DMD the total computational complexity is on the order $\mathcal{O}(d^3 + d^2n + n^2d)$.  Thus the added cost of including measurement noise is on the order $\mathcal{O}(n^2d)$.

\begin{table}[]
\centering
\caption{Tally of matrix and vector operations of algorithms~\ref{alg:margpost_KF},~\ref{alg:margpost_UKF}, DMD, and SINDy. VEW and MEW are element-wise vector and matrix operations respectively such as addition, subtraction, and element-wise multiplication and division.  MV is a matrix-vector or vector-vector multiplication, and MM is matrix-matrix multiplication.  Inv is a matrix inversion, Det a determinant, and Chol a Cholesky decomposition.}
\label{tab:numOps}
\begin{tabular}{|l|l|l|l|l|l|l|l|}
\hline
\textbf{Algorithm} & \textbf{VEW} & \textbf{MEW} & \textbf{MV} & \textbf{MM} & \textbf{Inv} & \textbf{Det} & \textbf{Chol} \\ \hline \hline
KF Prediction      & $0$          & $1$          & $1$         & $2$         & $0$          & $0$          & $0$           \\ \hline
KF Update          & $2$          & $2$          & $3$         & $6$         & $1$          & $0$          & $0$           \\ \hline
KF Total           & $4n$         & $3n$         & $6n$        & $8n$        & $2n$         & $n$          & $0$           \\ \hline
\hline
UKF Prediction     & $4d$         & $8$          & $0$         & $1$         & $0$          & $0$          & $1$           \\ \hline
UKF Update         & $4d+2$       & $14$         & $1$         & $5$         & $1$          & $0$          & $1$           \\ \hline
UKF Total          & $(8d+4)n$    & $22n$        & $3n$        & $6n$        & $2n$         & $n$          & $2n$          \\ \hline
\hline
DMD                & $0$          & $0$          & $0$         & $3$         & $1$          & $0$          & $0$           \\ \hline
\hline
Sparse Regression  & $0$          & $0$          & $0$         & $3$         & $1$          & $0$          & $0$           \\ \hline
\end{tabular}
\end{table}

\begin{table}[]
\centering
\caption{Flop count of algorithms~\ref{alg:margpost_KF},~\ref{alg:margpost_UKF}, DMD, and SINDy}
\label{tab:numFlops}
\begin{tabular}{|l|l|}
\hline
\textbf{Algorithm} & \textbf{Flop Count}                                                                 \\ \hline \hline
KF Prediction      & $4d^3+d^2-d$                                                                        \\ \hline
KF Update          & $2d^3 + \frac{1}{3}m^3 +6d^2m+4dm^2+-d^2-m^2+3dm-1$                                 \\ \hline \hline
KF Total           & $n(6d^3 + m^3 +6d^2m+4dm^2+m^2+3dm-d+3m+8)$                                         \\ \hline \hline
UKF Prediction     & $\frac{13}{3}d^3+17d^2+4d+2+(2d+1)F$                                                \\ \hline
UKF Update         & $\frac{1}{3}d^3+\frac{1}{3}m^3+6d^2m+8dm^2+9d^2+4m^2+13dm+$          \\ & \qquad $2d+6m+2+(2d+1)H$ \\\hline 
UKF Total          & $n\bigg(\frac{14}{3}d^3+m^3+6d^2m+8dm^2+26d^2+6m^2+13dm+$ \\                     & \qquad $6d+9m+13+(2d+1)(F+H)\bigg)+18$ \\ \hline \hline
DMD                & $\frac{7}{3}m^3+4m^2n-7m^2$                                                               \\ \hline \hline
Sparse Regression  & $\frac{1}{3}\frac{p^3}{m^3}+4\frac{p^2n}{m^2}-5\frac{p^2}{m^2}-\frac{pn}{m}+2pn+\frac{p}{m}-3p$ \\ \hline
\end{tabular}
\end{table}

The computational costs of the Bayesian algorithms are on the order $\mathcal{O}(n(d^3+m^3))$.  Typically the dimension $\dimy$ of the observations is small, so this algorithm is primarily limited by the dimension $\dimx$ of the state vector. Furthermore, the dimension $\dimpar$ of the parameter vector only affects the evaluation of the prior, which is usually chosen so as to be easy to compute.  Therefore, this algorithm is most efficient for problems where the state dimension is low and the parameter dimension is high, such as in nonlinear regression problems.

\section{Numerical experiments}\label{sec:experiments}

In this section, we provide a set of empirical results that demonstrate a lack of robustness amongst methods that do not account for all three sources of uncertainty. We then show that our proposed approach is able to perform well under a greater variety of experimental conditions.  The conditions of each experiment are designed to highlight and exaggerate the specified limitation of some specific methods. We will show that in many cases only small changes to the setting, for instance a slightly larger noise or slower sampling frequency, can yield significant difference in learning with these existing methods -- demonstrating their lack of robustness.

Our evaluations of the methodology examine two quantities: reconstruction errors and prediction/forecasting errors. Reconstruction error compares how well the learned parameter is able to match the trajectory from which the data were generated. This is essentially training error, and used more to verify that the algorithms are working properly. Prediction/forecasting compares our estimate to some trajectory that is not contained in the data. These trajectories could be a continuation of the system into the future from the last point at which data were taken, or it could be starting the estimated dynamics at a different initial condition.  This comparison is of greater interest because it tests the extrapolatory power of the learned dynamics.

\subsection{Algorithmic settings}
To perform the following experiments, MATLAB 2019b was used. For our MCMC algorithm, we selected the delayed rejection adaptive Metropolis (DRAM) algorithm~\cite{Haario_2006}.  The tuning parameters of this algorithm are $n_0$ the number of samples to draw before beginning the AM algorithm, and $\gamma$ the scaling factor used by DR to scale the second-tier proposal covariance. In this paper, we used $n_0=200$ and $\gamma=0.01$ for each experiment.  Also throughout the algorithm, whenever a covariance matrix was calculated, a nugget $\varepsilon I$ was added where $\varepsilon=10^{-10}$ to help ensure positive definiteness. Furthermore, the algorithm requires selection of a starting sample and initial proposal covariance; we used the MAP point $\theta^\tmap$ as our initial sample $\theta^{(0)}$, and the inverse Hessian of the negative log posterior evaluated at $\theta^\tmap$ to be the initial covariance of our proposal distribution:
\begin{align}
    \pi_0(\theta) = \mathcal{N}\left(\theta^\tmap, \left(-\frac{\partial^2\log\probd(\theta^\tmap;\theta|\yn)}{\partial\theta^2}\right)^{-1}\right).
\end{align}
Both of these values were found using MATLAB's \texttt{fminunc} function.  For nonlinear systems, we must additionally select parameters $\alpha$, $\kappa$, and $\beta$ for the UKF.  In this paper, we followed a common choice of parameter selection where $\alpha=10^{-3}$, $\kappa = 0$, and $\beta=1$.

Unless otherwise specified, an improper uniform prior is used for all dynamics model parameters $\theta \in \thetdyn$ 
\begin{align}
    p(\theta)=\mathcal{U}(-\infty, \infty),
\end{align}
and half normal priors are specified on the variance parameters $\theta \in \thetsig$ and $\theta \in \thetgam$ as suggested in~\cite{Gelman_2006}
\begin{align}
    p(\theta)=\text{half-}\mathcal{N}(0,1).
\end{align}

The code used to implement the Bayesian algorithms can be found on the author's GitHub \url{https://github.com/ngalioto}.  To execute DMD, MATLAB's right matrix division operator `/' was used, which returns the least squares solution. TDMD was performed using a script taken from MATLAB file exchange~\cite{Houtzager_2019} that solves the total least squares problem.  Lastly, SINDy was run using code from~\cite{Brunton_2016}, which utilizes code from~\cite{Chartrand_2011} to compute the total variation regularized derivatives.

\subsection{Linear pendulum, linear model}
\label{subsec:linpend_param}

In this section, we consider learning a linear model under an identity observation operator $h=I$ when the truth model is also linear. We show that that the proposed probabilistic approach is more robust to sparse observations and measurement noise than the least squares-based DMD and TDMD.

Consider the linear model~\eqref{eq:pend_conts_param} for which the exact propagator is
\begin{align}
    x_k = 
    \text{exp}\Bigg(\begin{bmatrix}0 & 1 \\ -\frac{g}{L} & 0\end{bmatrix}\Delta t\Bigg)
    x_{k-1}, \qquad  x_0=\begin{bmatrix}0.1 \\ -0.5\end{bmatrix}
    \label{eq:pend_conts}
\end{align}
where $g=9.81$ is the acceleration due to gravity and $L=1$ is the length of the pendulum.  

We are learning an unknown linear model $A(\thetdyn)$ and assume that the process noise and measurement noise is also uncertain. Under this setting, System~\eqref{eq:sys} becomes
\begin{align}
\begin{split}
    x_{k} = A(\thetdyn)x_{k-1} + \xi_k, \quad \xi_k\sim\mathcal{N}(0,\Sigma(\thetsig)) \\
    y_k = x_k + \eta_k, \quad \eta_k\sim\mathcal{N}(0,\Gamma(\thetgam)),
\end{split}
\qquad  \text{for } k = 1,\ldots,\numObs,
\end{align}
where
\begin{align}
\begin{split}
    \begin{array}{lcr}
    A(\thetdyn) = \begin{bmatrix} \theta_1 & \theta_2 \\ \theta_3 & \theta_4 \end{bmatrix}, &
    \Sigma(\thetsig) = \theta_{5}I_{2 \times 2}, &
    \Gamma(\thetgam) = \theta_{6}I_{2 \times 2}.
    \end{array}
\end{split}
\label{eq:linpend_param}
\end{align}

Because this setup is precisely the one corresponding to DMD, we seek to compare the performance of our approach to DMD and TDMD. Our comparison takes the form of average performance over 500 different realizations of the data sets for different combinations of training data sizes $n$ and true measurement noise standard deviation $\sigma$. The data points are spread out over a simulation period of four seconds, so increasing $n$ indicates increasing density of data per time.  

The results, shown in Figure~\ref{fig:MSE_contours}, provide (log base 10) ratios of the expected error of the posterior predictive mean (computed with 1000 posterior samples) to the (T)DMD estimators. The squared errors were calculated only at the times of observations, and the largest MSE from each data set for each algorithm was discarded to prevent biasing from outliers.   We see that the biggest gains in using the probabilistic Bayesian approach come in the low noise regime.  At first this seems surprising, but in the low noise regime, this is likely the result of the scale of the errors being so small.  As the noise increases, we see the ratio increasing even though we'd expect DMD to break down much more quickly than the Bayesian approach.  The reason this occurs is because DMD predictions decay to zero after a certain level of noise (shown in Figure~\ref{fig:pred_DMDcontour}), effectively placing an upper bound on the MSE of the algorithm. Regardless, the contour plots show that the Bayesian algorithm outperforms both DMD and TDMD at every measurement frequency and noise pair considered.

\begin{figure}
    \centering
    \begin{subfigure}{0.33\linewidth}
        \includegraphics[height=30mm]{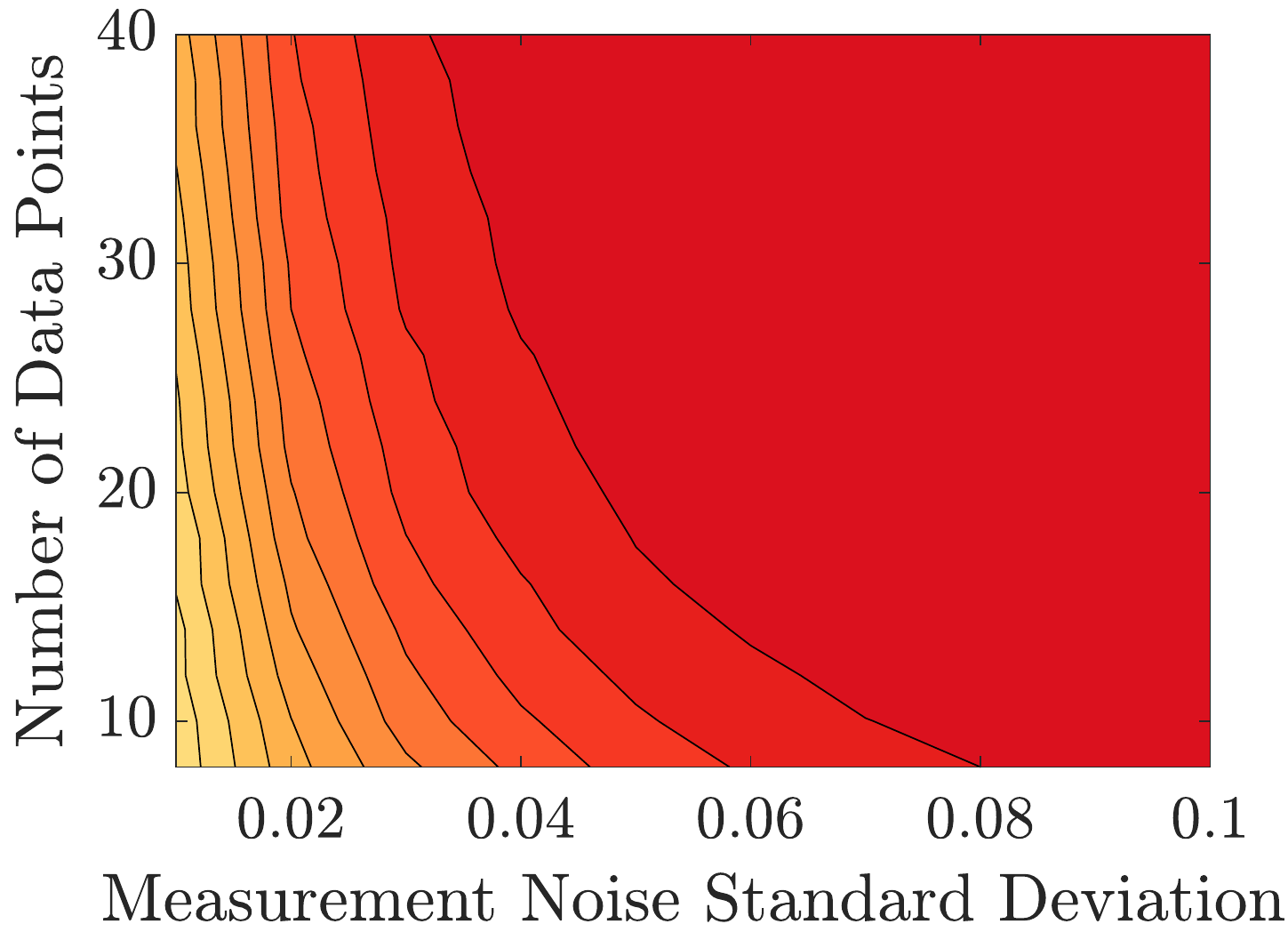}
        \caption{DMD Prediction MSE}
        \label{fig:pred_DMDcontour}
    \end{subfigure}
    \begin{subfigure}{0.32\linewidth}
        \includegraphics[height=30mm]{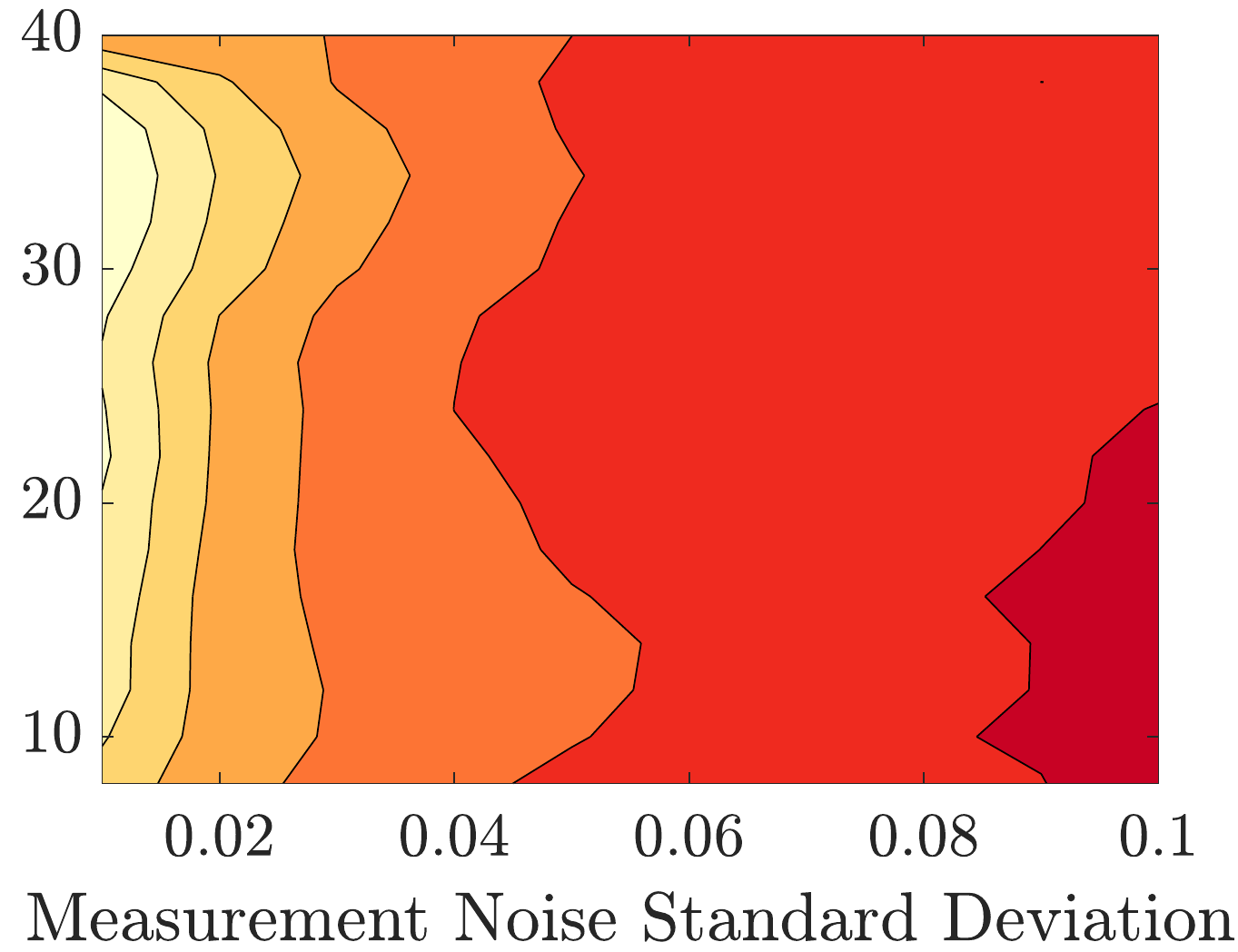}
        \caption{Bayes/DMD Prediction}
        \label{fig:pred_ratioxmeanDMD}
    \end{subfigure}
    \begin{subfigure}{0.32\linewidth}
        \includegraphics[height=30mm]{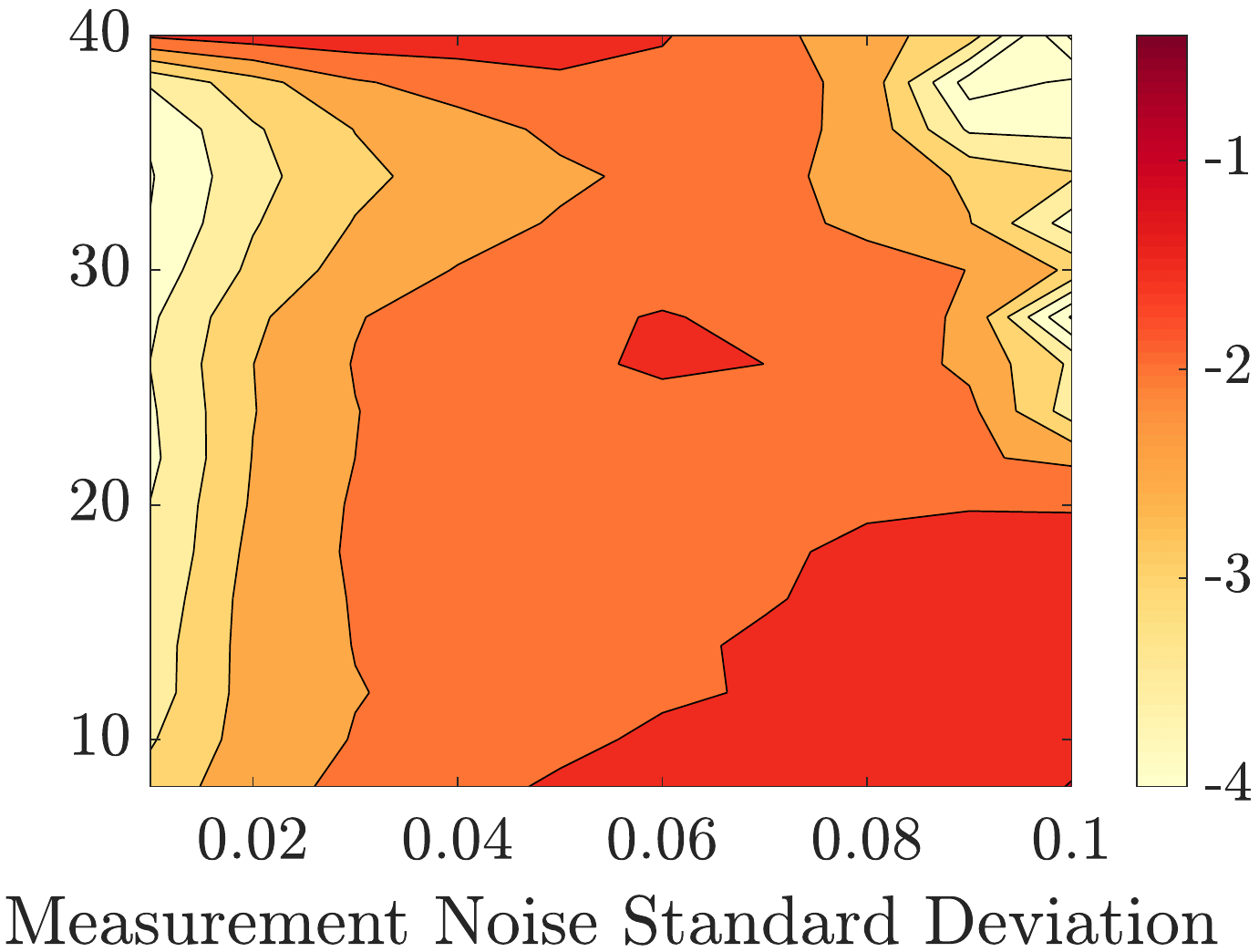}
        \caption{Bayes/TDMD Prediction}
        \label{fig:pred_ratioxmeanTDMD}
    \end{subfigure}
    \caption{Log base 10 ratio of the MSE obtained by the proposed Bayesian approach to that obtained by (T)DMD for the linear pendulum model. In all cases, this value is less than zero signifying that our proposed approach outperforms (T)DMD in all cases considered.  Also observe in the high noise regime, TDMD can begin to lose stability.}
    \label{fig:MSE_contours}
\end{figure}

Next we provide a detailed look at two specific points on these contour plots to demonstrate the mechanism by which DMD/TDMD decline. The first case is a low-noise/sparse-data case of $\sigma=10^{-2}$ and $\numObs=8$, and the second case is for a higher noise case $\sigma=10^{-1}$ with more data $\numObs=40.$ 

The reconstruction results for each state are compared in Figure~\ref{fig:recon_linear}. The prediction (forecasting) results for just the second state are shown in Figure~\ref{fig:predict_linear}. The shaded area represents the region between the 97.5th and 2.5th quantiles of the Bayesian posterior. In the low noise case, we see that all three algorithms perform essentially equally -- though the DMD-based approaches slightly underestimate the amplitude. In other words, even in the case for which DMD was designed to perform well, the Bayesian approach performs slightly better. In the high noise case, we see that the TDMD predictions become completely out of phase with increasingly small amplitude, and the DMD estimator smooths out the data too much and rapidly converges to zero.  Not only does the Bayesian approach provide the most accurate estimate, but it also gives a quantification of the certainty of its estimate in the form of its posterior, which (T)DMD is unable to provide.
 
 \begin{figure}
    \centering
    \begin{subfigure}{0.43\linewidth}
        \includegraphics[width=\linewidth]{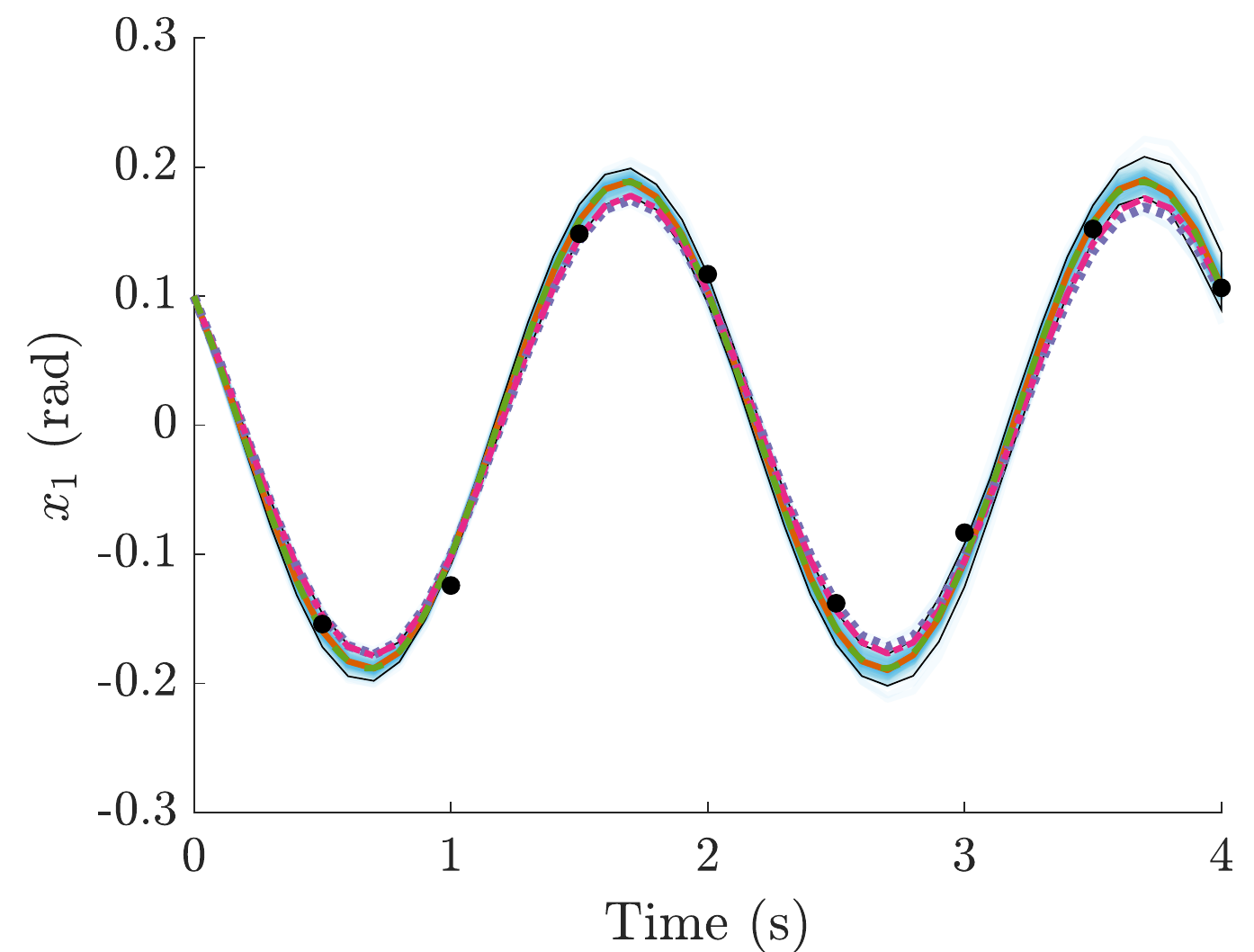}
        \caption{$x_1, \sigma=10^{-2}, \numObs=8$}
        \label{fig:8_1reconx1}
    \end{subfigure}
    \begin{subfigure}{0.56\linewidth}
        \centering
        \includegraphics[width=\linewidth]{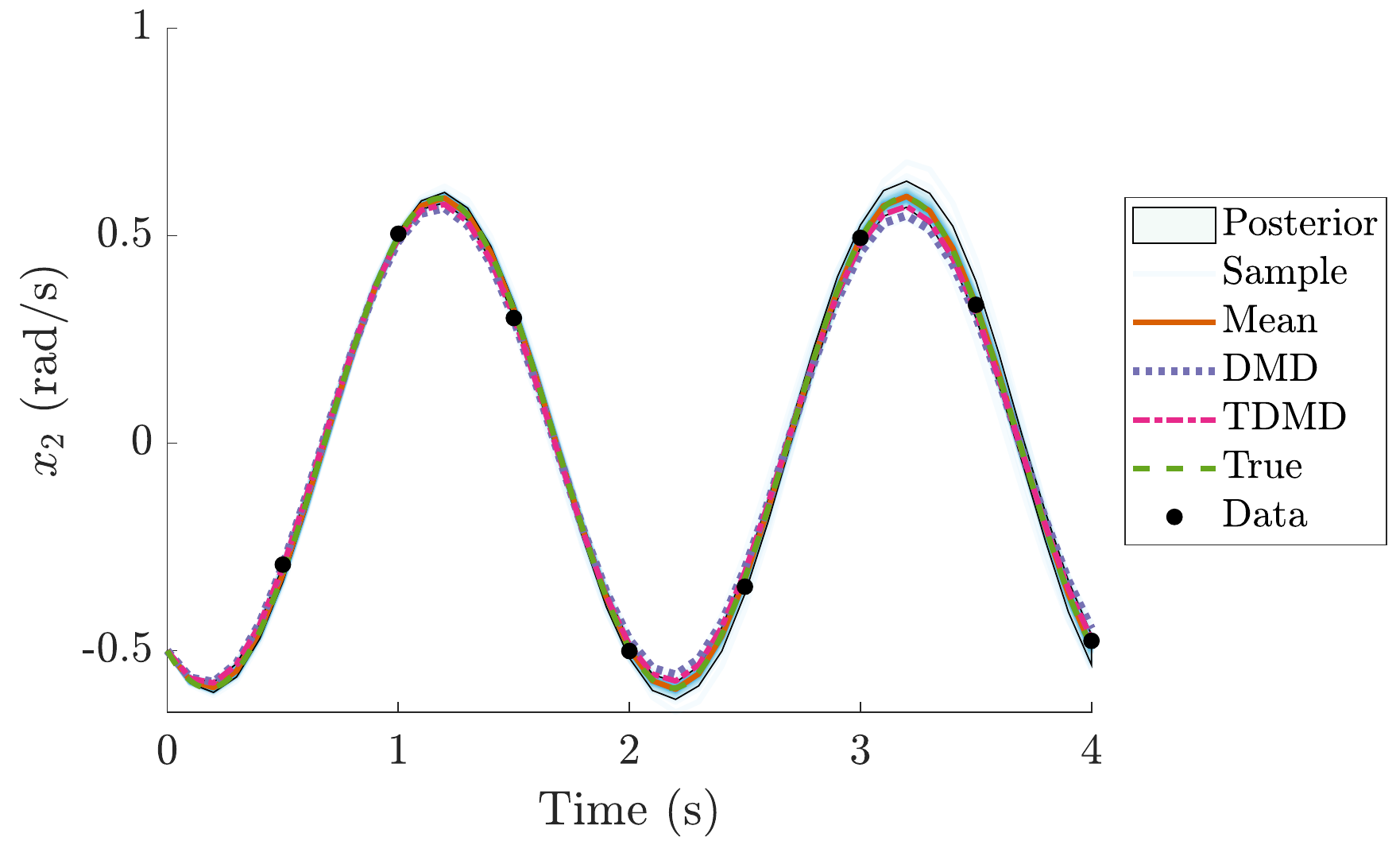}
        \caption{$x_2, \sigma=10^{-2}, \numObs=8$}
        \label{fig:8_1reconx2}
    \end{subfigure}
    \begin{subfigure}{0.43\linewidth}
        \centering
        \includegraphics[width=\linewidth]{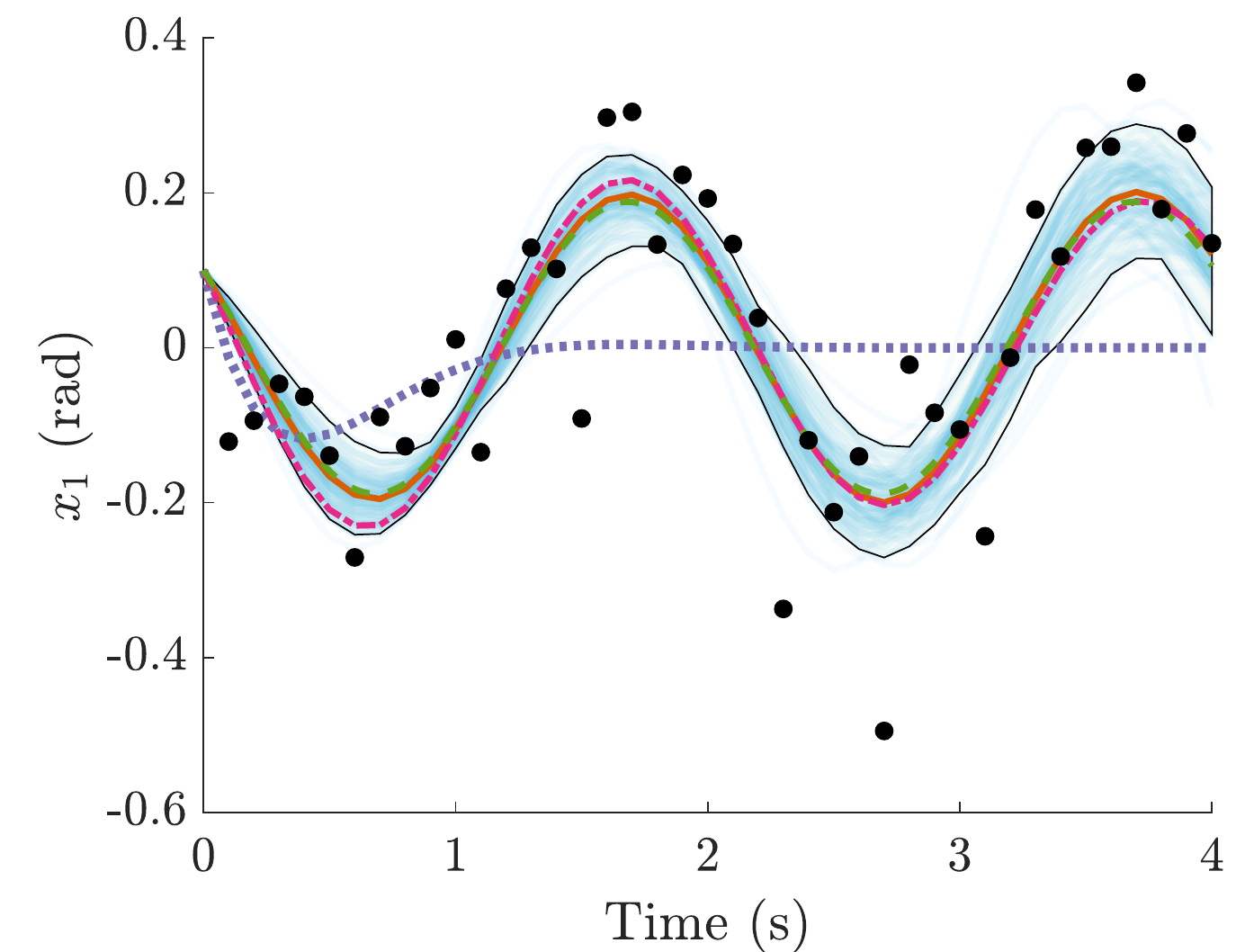}
        \caption{$x_1, \sigma=10^{-1}, \numObs=40$}
        \label{fig:40_10reconx1}
    \end{subfigure}
    \begin{subfigure}{0.56\linewidth}
        \centering
        \includegraphics[width=\linewidth]{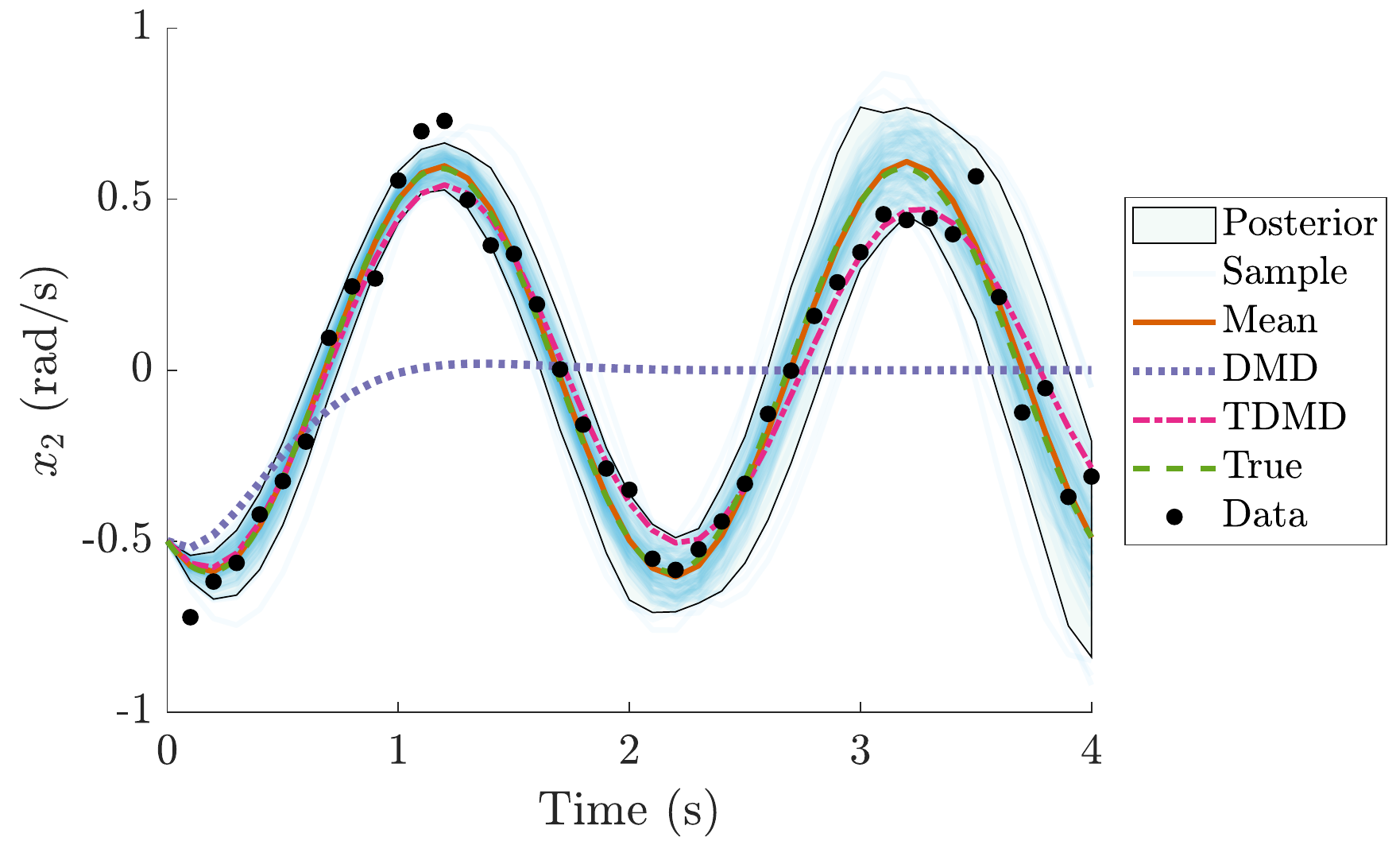}
        \caption{$x_2, \sigma=10^{-1}, \numObs=40$}
        \label{fig:40_10reconx2}
    \end{subfigure}
    \caption{Comparison of reconstruction error amongst the Bayesian and (T)DMD algorithms for different levels of noise and measurements for the linear pendulum truth model. Top row corresponds to a low-noise/sparse-data case and the bottom row corresponds to a high-noise/dense-data case. Left column corresponds to the first state (angular position) and right column corresponds to the second state (angular velocity). For low-noise, the algorithms perform similarly; however, the (T)DMD approaches underestimate the amplitude. For the high-noise case, DMD fails and TDMD misfits the amplitude. The Bayesian approach is able to recognize greater uncertainty for the high-noise case.}
    \label{fig:recon_linear}
\end{figure}
 
 \begin{figure}
    \centering
    \begin{subfigure}{0.43\linewidth}
        \centering
        \includegraphics[width=\linewidth,clip,trim=0 0 100 0]{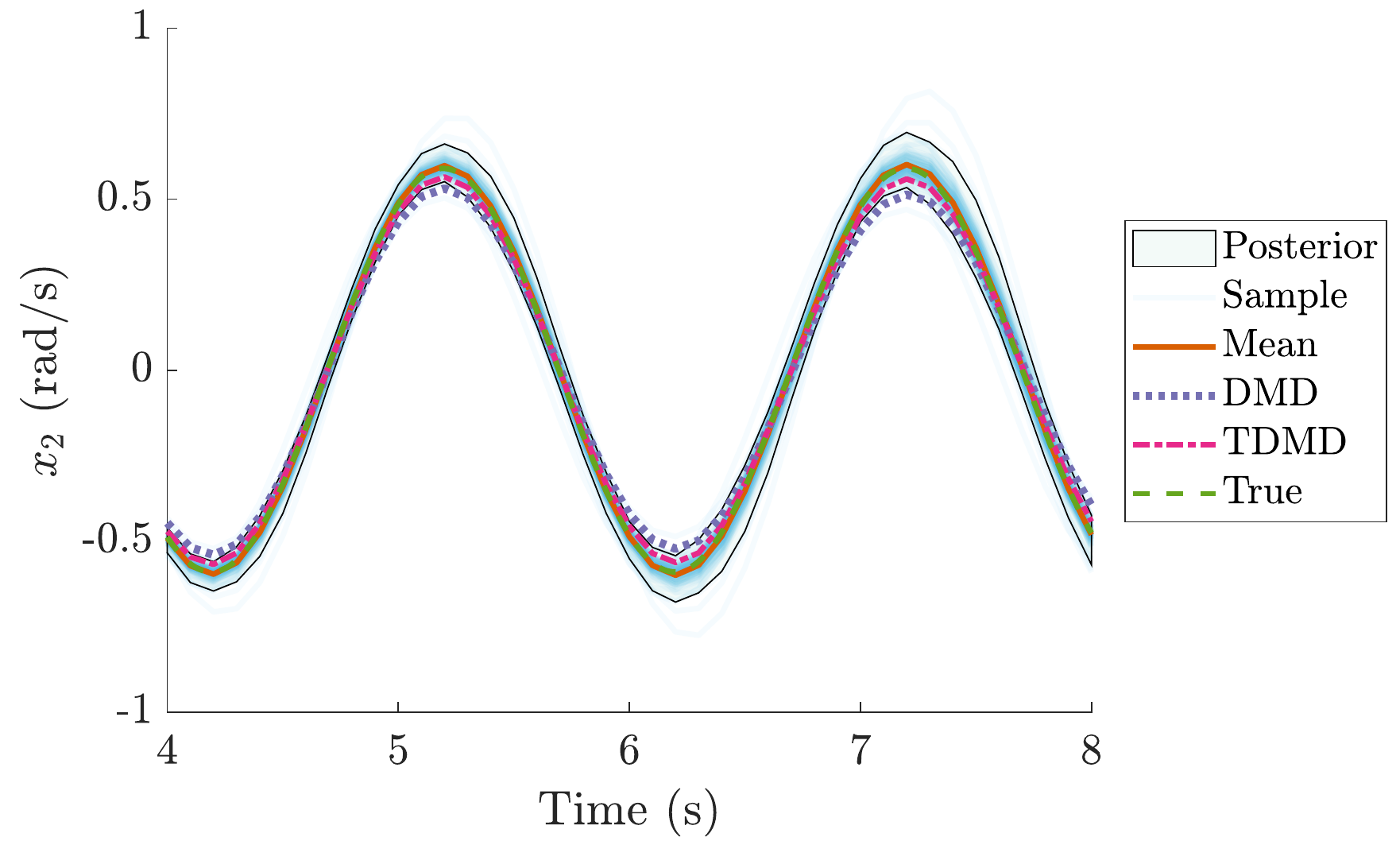}
        \caption{$x_2, \sigma=10^{-2}, \numObs=8$}
        \label{fig:8_1predictx2}
    \end{subfigure}
    \begin{subfigure}{0.56\linewidth}
        \centering
        \includegraphics[width=\linewidth]{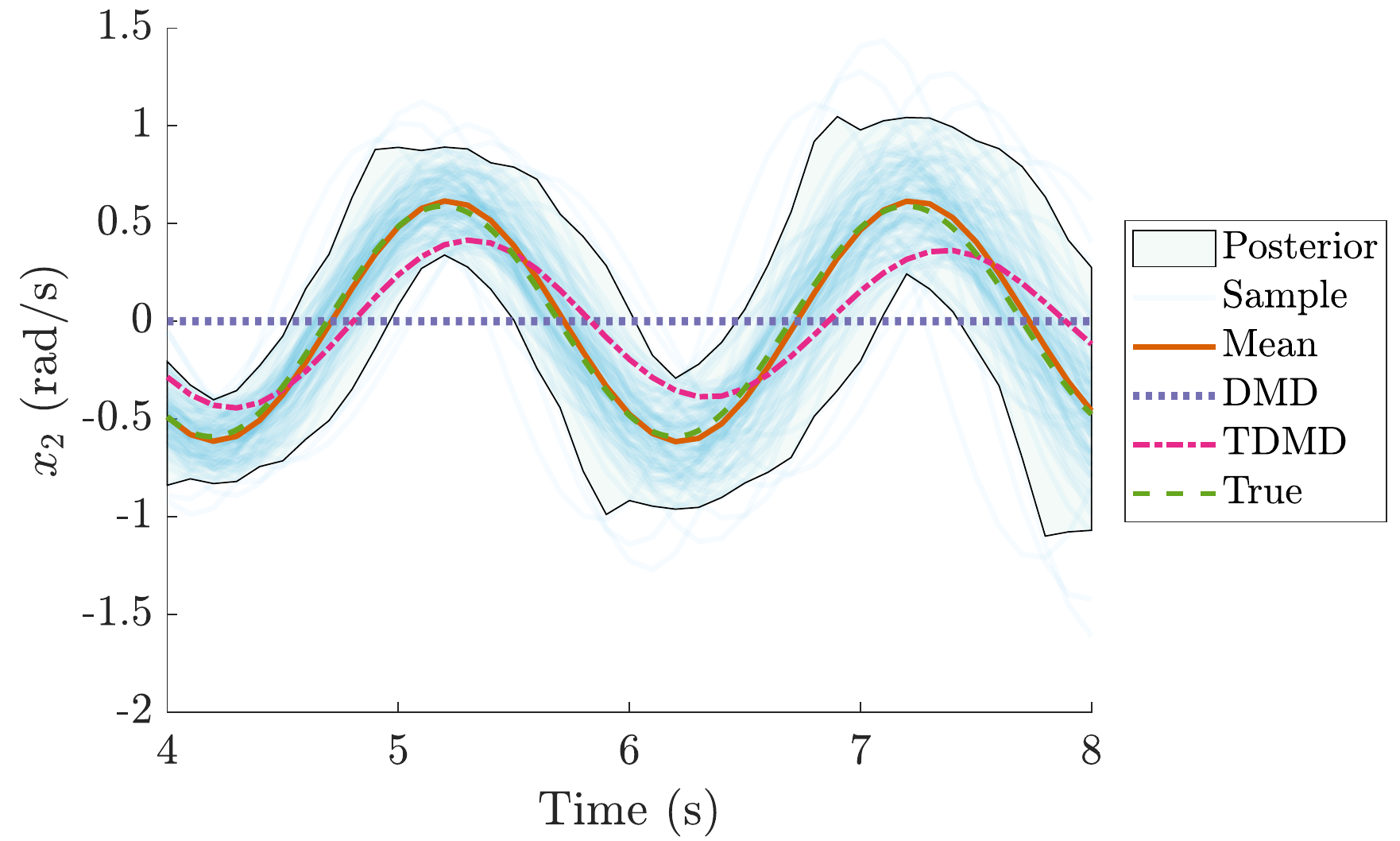}
        \caption{$x_2, \sigma=10^{-1}, \numObs=40$}
        \label{fig:40_10predictx2}
    \end{subfigure}
    \caption{Comparison of prediction error amongst the Bayesian and (T)DMD algorithms for different levels of noise and measurements for the linear pendulum truth model. Left panel corresponds to a low-noise/sparse-data case and the right panel corresponds to a high-noise/dense-data case. Both panels show the angular velocity of the pendulum. For low-noise, the algorithms perform similarly. For the high-noise case, DMD fails and TDMD can be seen to be out of phase and have a smaller amplitude. The Bayesian approach is able to recognize greater uncertainty for the high-noise case.}
    \label{fig:predict_linear}
\end{figure}

Figure~\ref{fig:eigs} shows the estimated eigenvalues of the system by the Bayesian and (T)DMD algorithms.  In the low noise case, Figure~\ref{fig:8_1eig} shows that the Bayesian approach is slightly more accurate than the (T)DMD approaches, though they all perform well. For the high noise case, Figure~\ref{fig:40_10eig} shows that DMD is unable to provide a reasonable estimate of the eigenvalues. TDMD gives a close estimate, but the estimated eigenvalues are too far in the left-hand plane, causing the gradual decay seen in Figure~\ref{fig:predict_linear}. The Bayesian estimate lies almost exactly on top of the truth.  

\begin{figure}
\centering
    \begin{subfigure}{0.45\textwidth}
    \centering
    \includegraphics[width=0.9\linewidth]{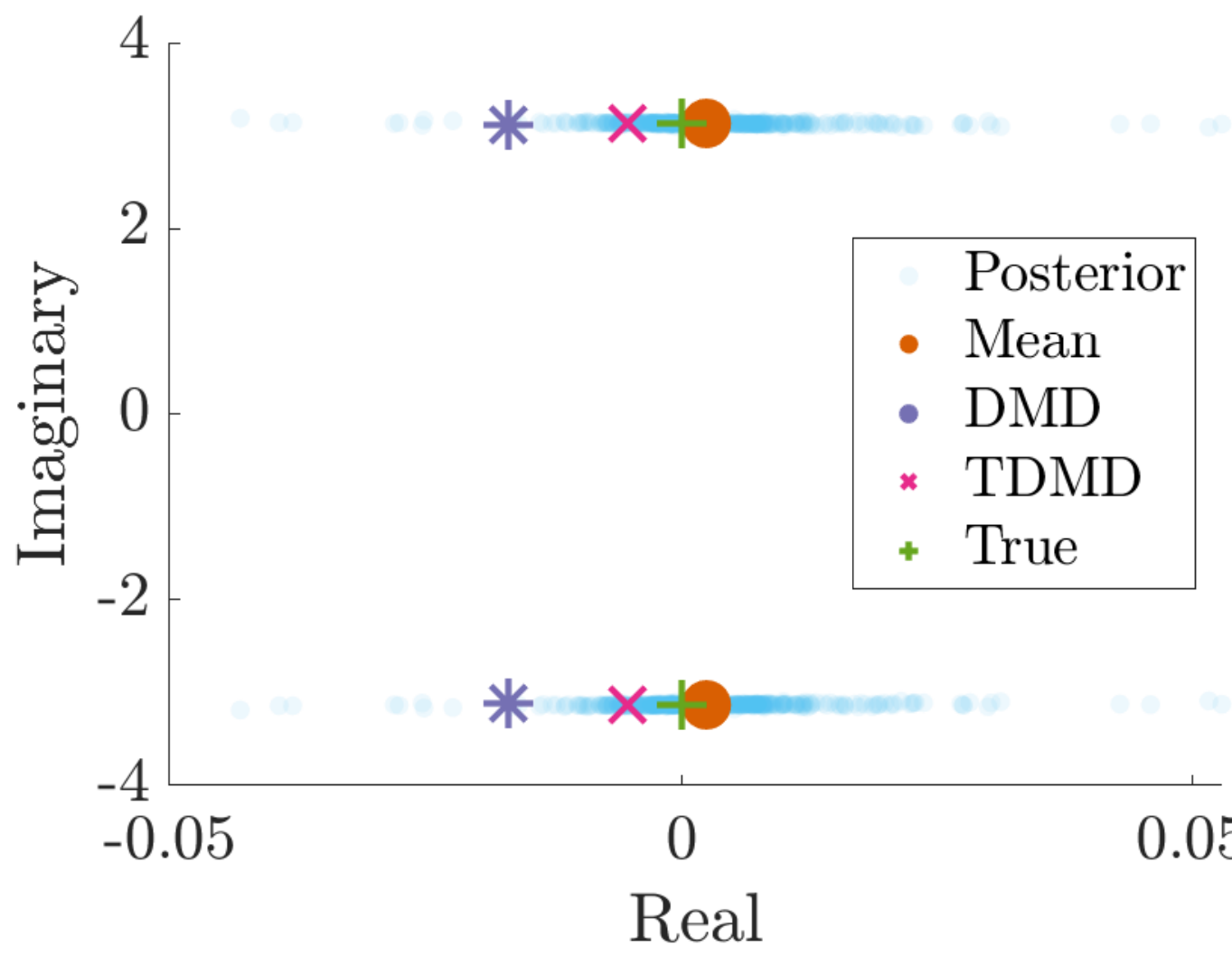}
    \caption{Lower noise/measurement frequency}
    \label{fig:8_1eig}
    \end{subfigure}
    \qquad 
    \begin{subfigure}{0.45\textwidth}
    \centering
    \includegraphics[width=0.9\linewidth]{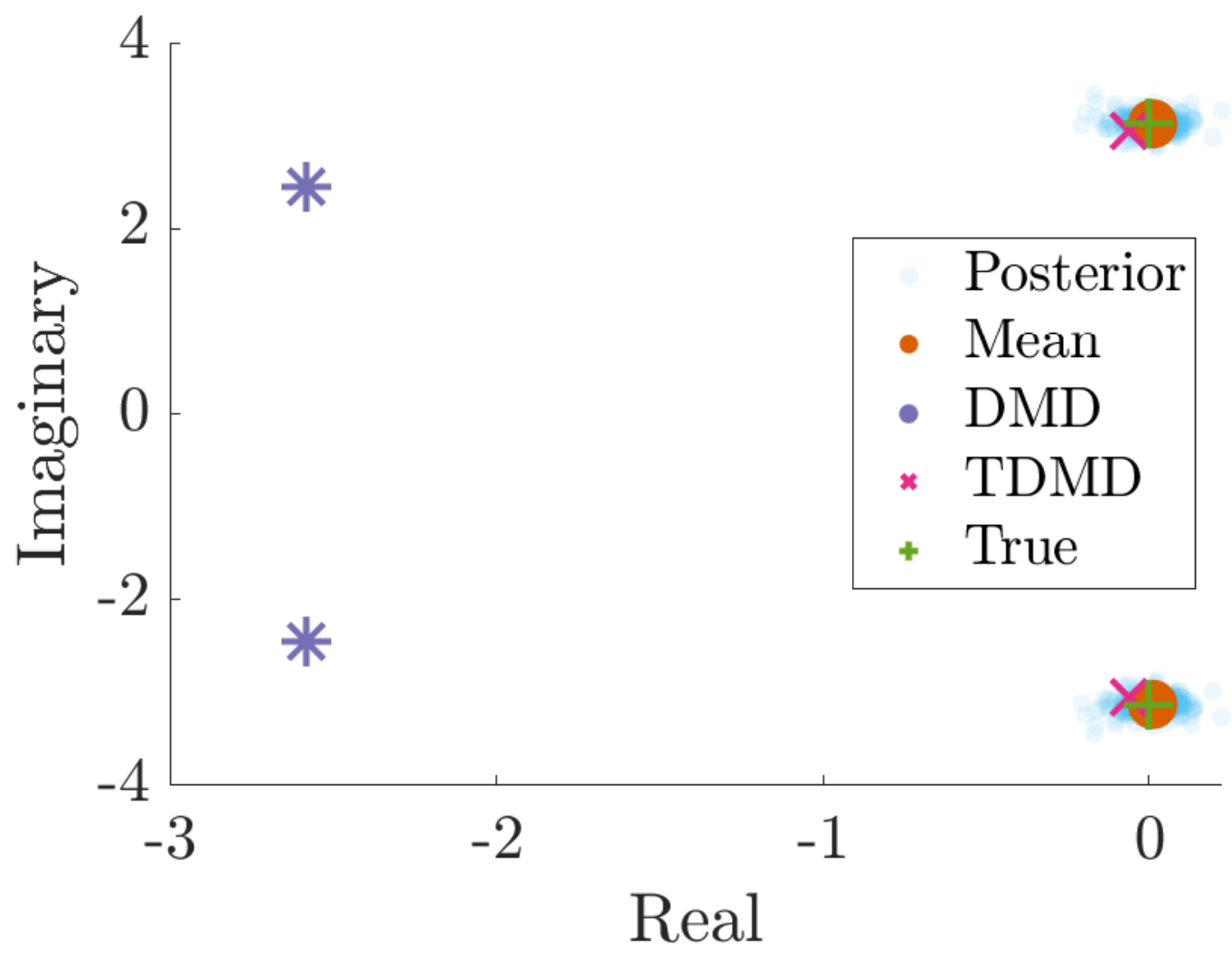}
    \caption{Higher-noise/measurement frequency}
    \label{fig:40_10eig}
    \end{subfigure}
    \caption{Eigenvalue distributions for the estimators of the linear pendulum. All three algorithms come very close to learning the true eigenvalues in the low noise case, but Bayes is able to outperform the other two in both the high and low noise cases. DMD achieves significant error when the data are noisy. The mean value here represents the mean of the eigenvalues.}
    \label{fig:eigs}
\end{figure}

Finally, Figure~\ref{fig:varHistLin} shows the marginal and joint distributions of the process and measurement noise variances for these two cases.  The process noise is very close to zero because we are using a linear model for a linear system, and thus the system learns that the dynamics can be captured exactly. These plots also indicate that we have learned the measurement noise, as the mode aligns closely with the true value shown in red.  Note also that the joint distribution in this figure shows that the two noise variances are negatively correlated, conveying the fact that the estimator does not yet have enough data to determine if the model is off and the measurements are accurate, or if the model is accurate and the measurements are noisy. As more data come in, however, one of these scenarios can usually be ruled out and the distribution becomes unimodal.

\begin{figure}
    \begin{subfigure}{0.49\linewidth}
        \centering
        \includegraphics[width=\linewidth]{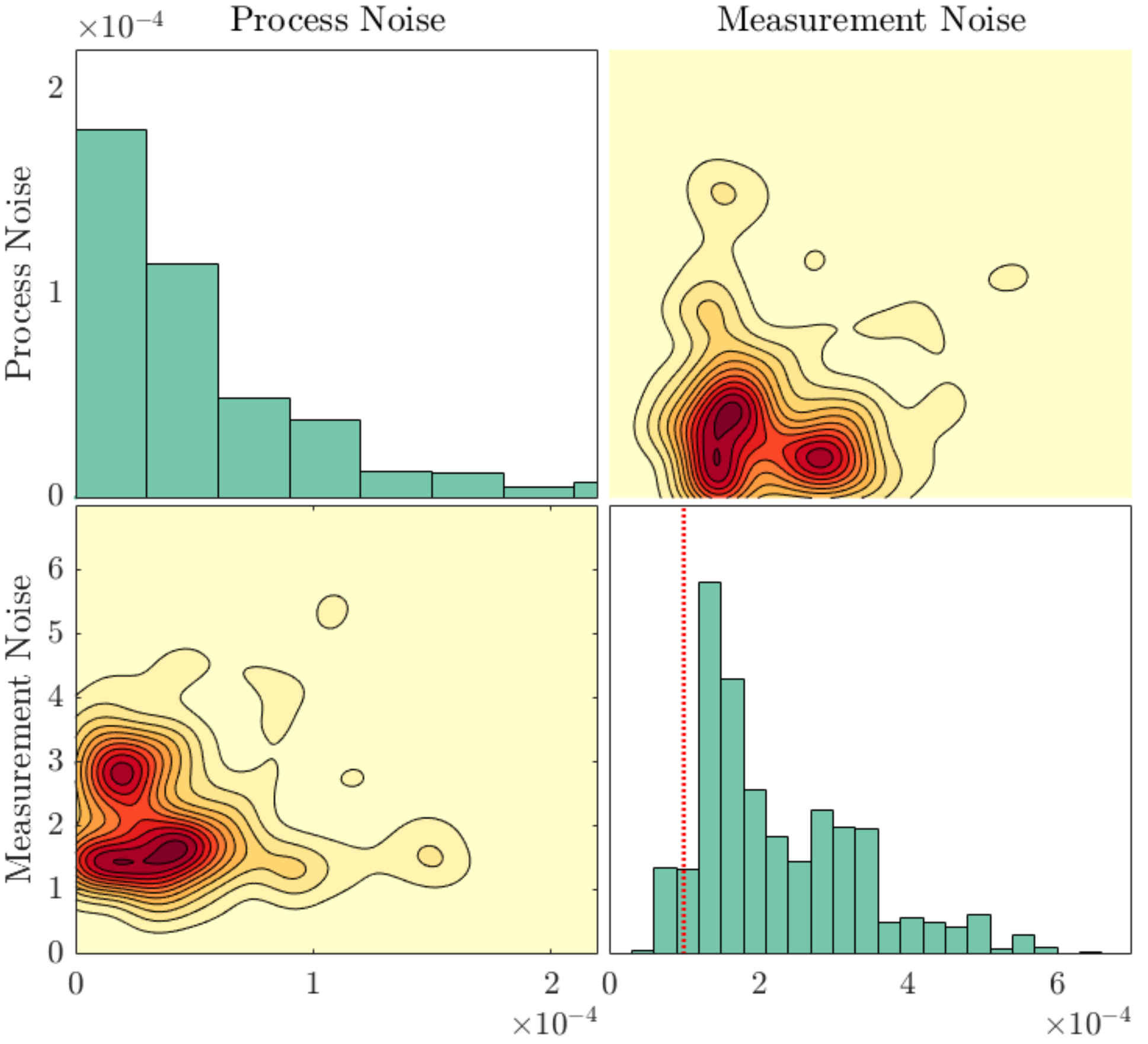}
        \caption{Lower noise/measurement frequency}
        \label{fig:nlscatterHist8_1}
    \end{subfigure}
    \begin{subfigure}{0.49\linewidth}
        \centering
        \includegraphics[width=\linewidth]{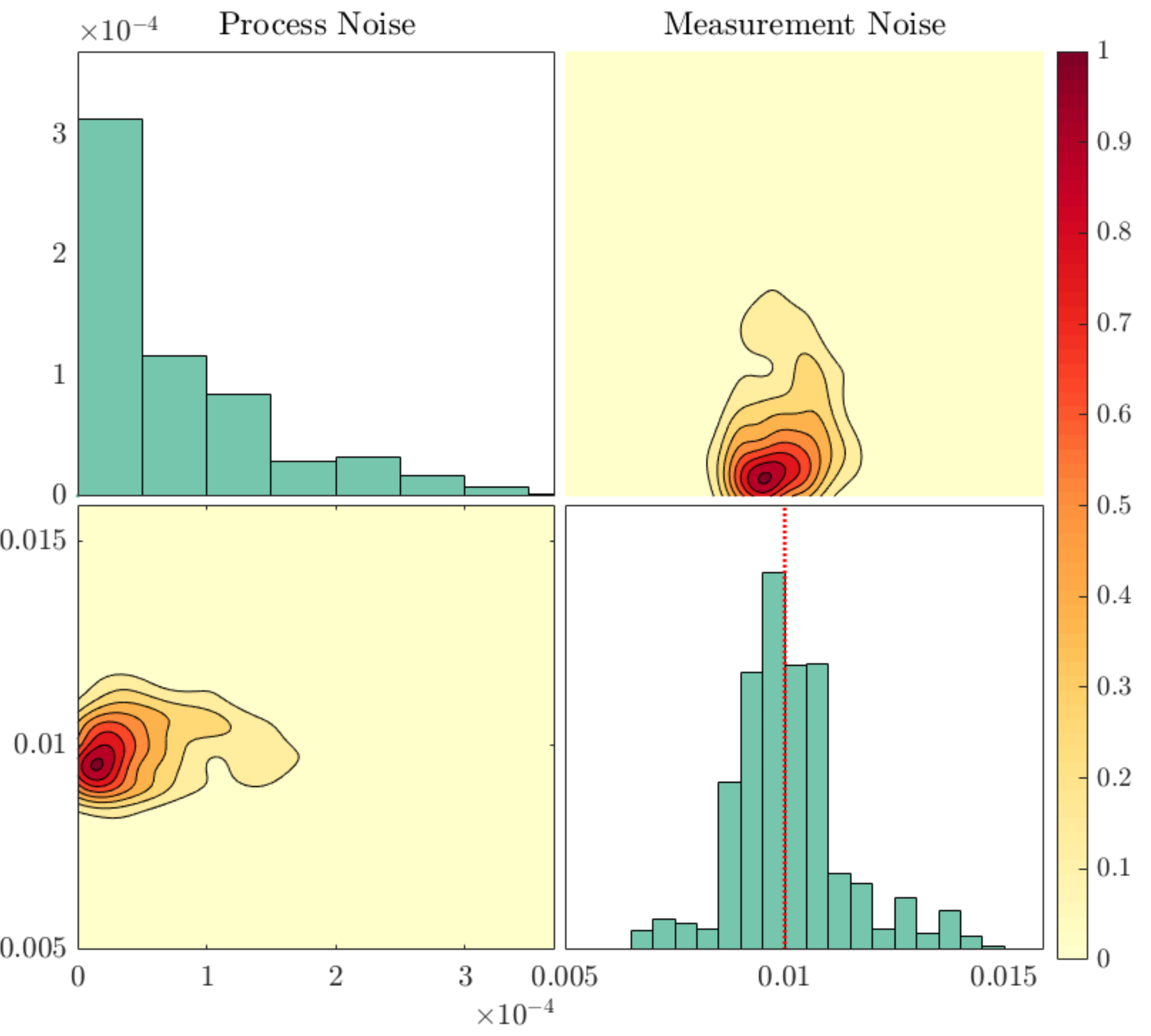}
        \caption{Higher-noise/measurement frequency}
        \label{fig:scatterHist40_10}
    \end{subfigure}
    \caption{Marginal and joint posterior distributions of the process and measurement noise variance parameters during the recovery of the linear pendulum.  In the left panel, 8 measurements are not enough for the Bayesian estimator to unambiguously determine the measurement noise, but its best guess (the mode) aligns with the truth.  On the right, we see that 40 measurements are enough to define a distinct mode within the joint distribution, which also aligns with the truth.}
    \label{fig:varHistLin}
\end{figure}

\subsection{Nonlinear pendulum, linear model}

Next we consider a problem where the model class within which we are learning does not encompass the true underlying dynamical system. This is the most realistic situation that would be encountered in practice, and avoids the so-called ``inverse crime''~\cite{Colton_1992, Wirgin_2004}.

Consider a nonlinear pendulum
\begin{align}
    \begin{bmatrix}\dot{x}_1\\\dot{x}_2\end{bmatrix} = 
    \begin{bmatrix}x_2 \\ -\frac{g}{L}\text{sin}(x_1) \end{bmatrix},
    \quad x_0=\begin{bmatrix}2.5 \\ 0\end{bmatrix}
\end{align}
to be the truth model. We have changed the initial condition to ensure that we are operating in the nonlinear regime.

The learning setup is identical to that provided in Section~\ref{subsec:linpend_param}; we learn a linear model, and the same validation experiments are performed. These experimental results are shown in Figure~\ref{fig:nlMSE_contours}.  We are able to clearly see here that, although the three algorithms are comparable in the low noise regime, the strength of the Bayesian approach increases with the measurement noise.  A discussion on why (T)DMD may outperform the mean estimator from the Bayesian approach in the low noise regime is provided later in Section~\ref{subsubsec:diagnostics}.

\begin{figure}
    \centering
    \begin{subfigure}{0.49\linewidth}
        \includegraphics[height=45mm]{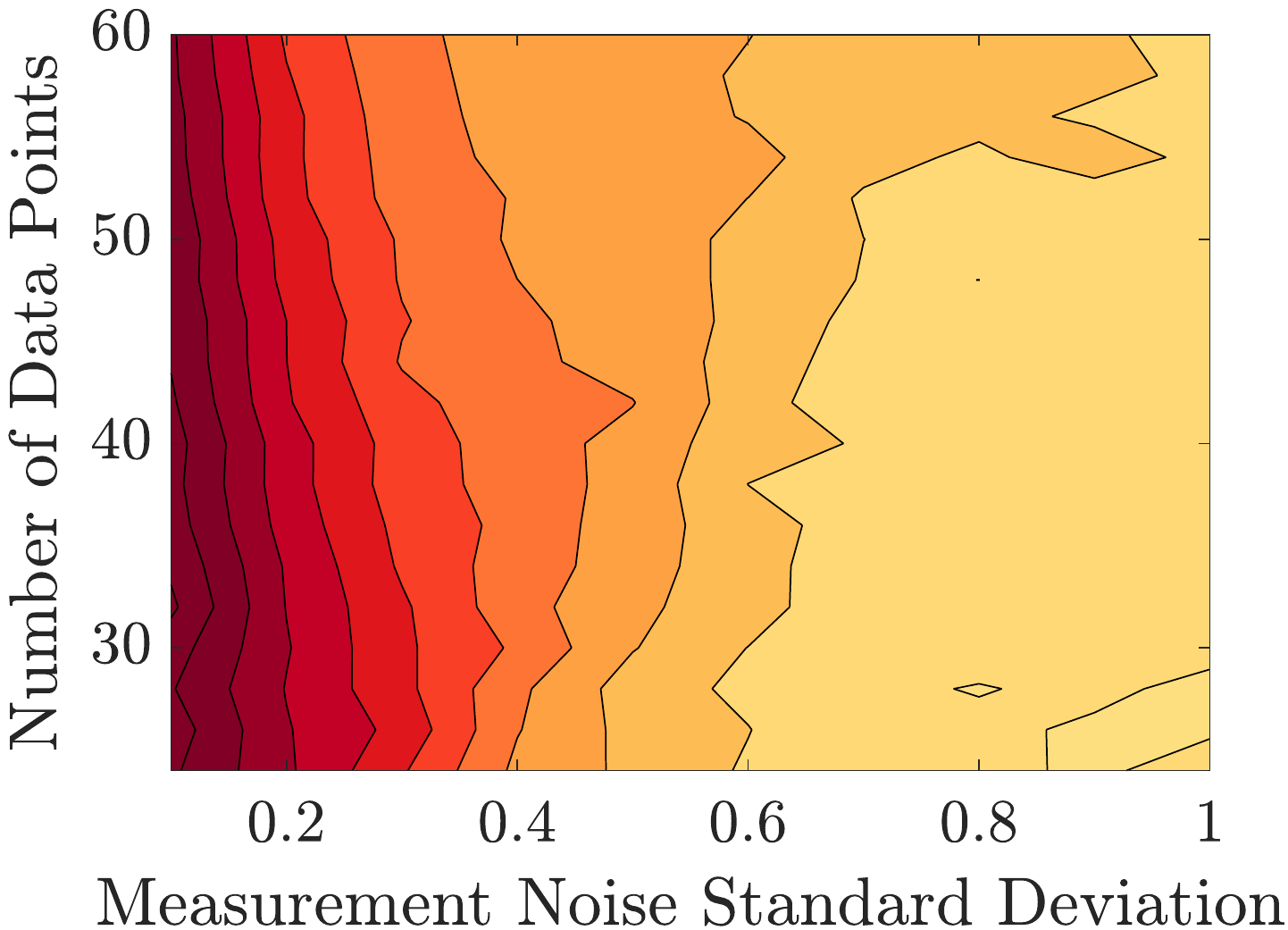}
        \caption{Bayes to DMD Prediction Ratio}
        \label{fig:nlpred_ratioxmeanDMD}
    \end{subfigure}
    \begin{subfigure}{0.49\linewidth}
        \includegraphics[height=45mm]{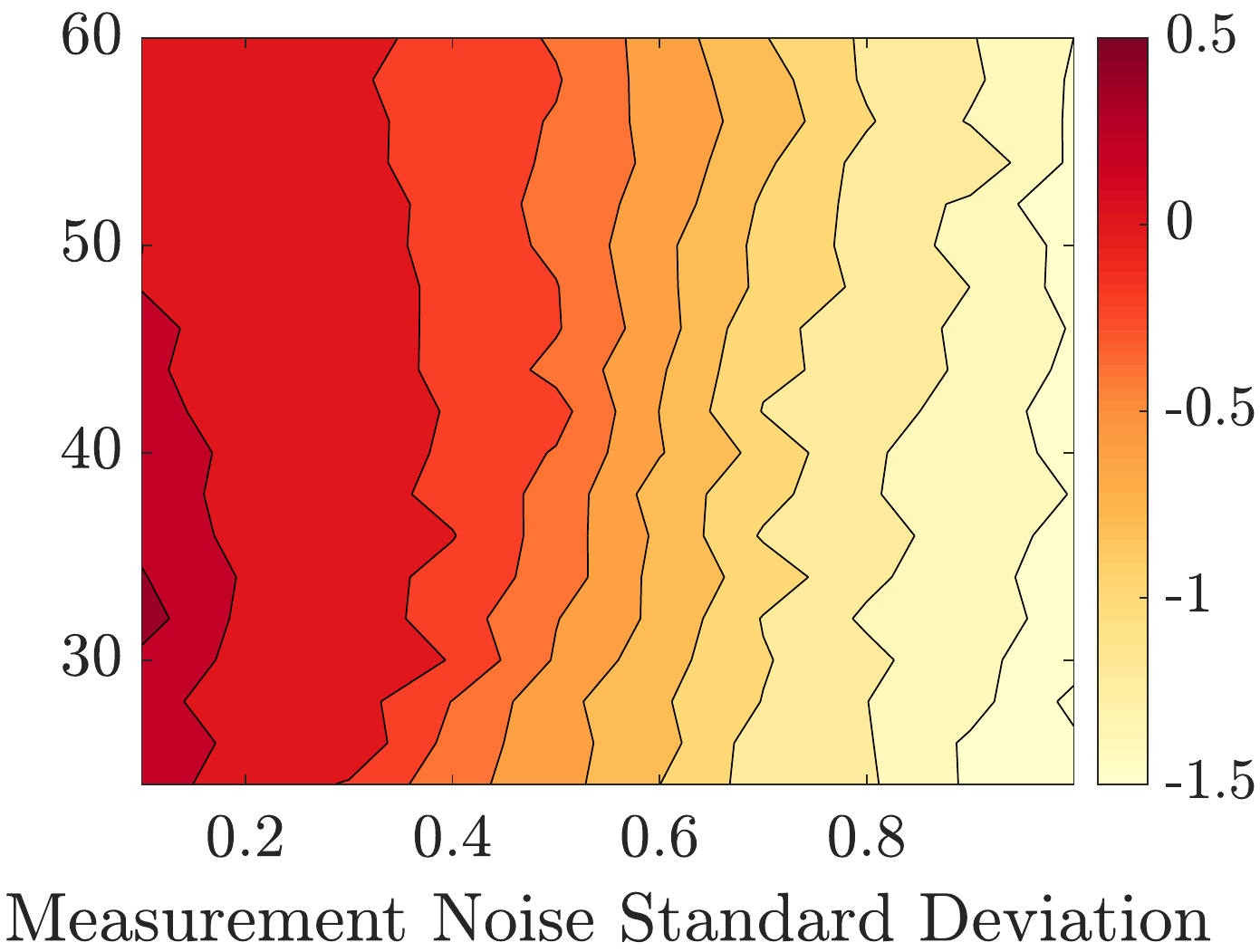}
        \caption{Bayes to TDMD Prediction Ratio}
        \label{fig:nlpred_ratioxmeanTDMD}
    \end{subfigure}
    \caption{The experiment conducted here is the same as described in Figure~\ref{fig:MSE_contours}, but this time with the nonlinear pendulum. Note that as the measurement noise increases, the log ratios decrease, reflecting the robustness of the Bayesian approach in the face of noisy measurements.  A detailed explanation for the low noise regime where it appears (T)DMD outperforms Bayes is given in Section~\ref{subsubsec:diagnostics}.}
    \label{fig:nlMSE_contours}
\end{figure}

We again present more detailed results for two representative cases. Both cases have $n=24$ data points, but the first case is a low noise case of $\sigma=10^{-1}$ and the second case is a higher noise case of $\sigma=1$.

The resulting reconstructions are shown in Figure~\ref{fig:recon_nonlinear}, and the predictions are given in Figure~\ref{fig:predict_nonlinear}. Note that the variances of the posterior distributions in both cases grow much more quickly than in either of the linear pendulum examples as a consequence of increased model uncertainty (process noise).  The posterior distribution can therefore be used to qualitatively assess not only how informative the data are, but also how appropriate the chosen model is for the system at hand.  In the low noise case, the performances of the three estimates are virtually indistinguishable, once again demonstrating that even in systems that are ideal for (T)DMD, there is no loss of performance when using the Bayesian estimator. In the high noise case, DMD struggles with noisy measurements and settles on quickly decaying to zero, similar to what we observed in the linear case.  TDMD, on the other hand, comes closer but is noticeably out of phase with the truth.  The Bayesian approach is able to reconstruct the signal very closely, at least within the constraints imposed by using a linear model.

\begin{figure}
    \centering
    \begin{subfigure}{0.43\linewidth}
        \includegraphics[width=\linewidth]{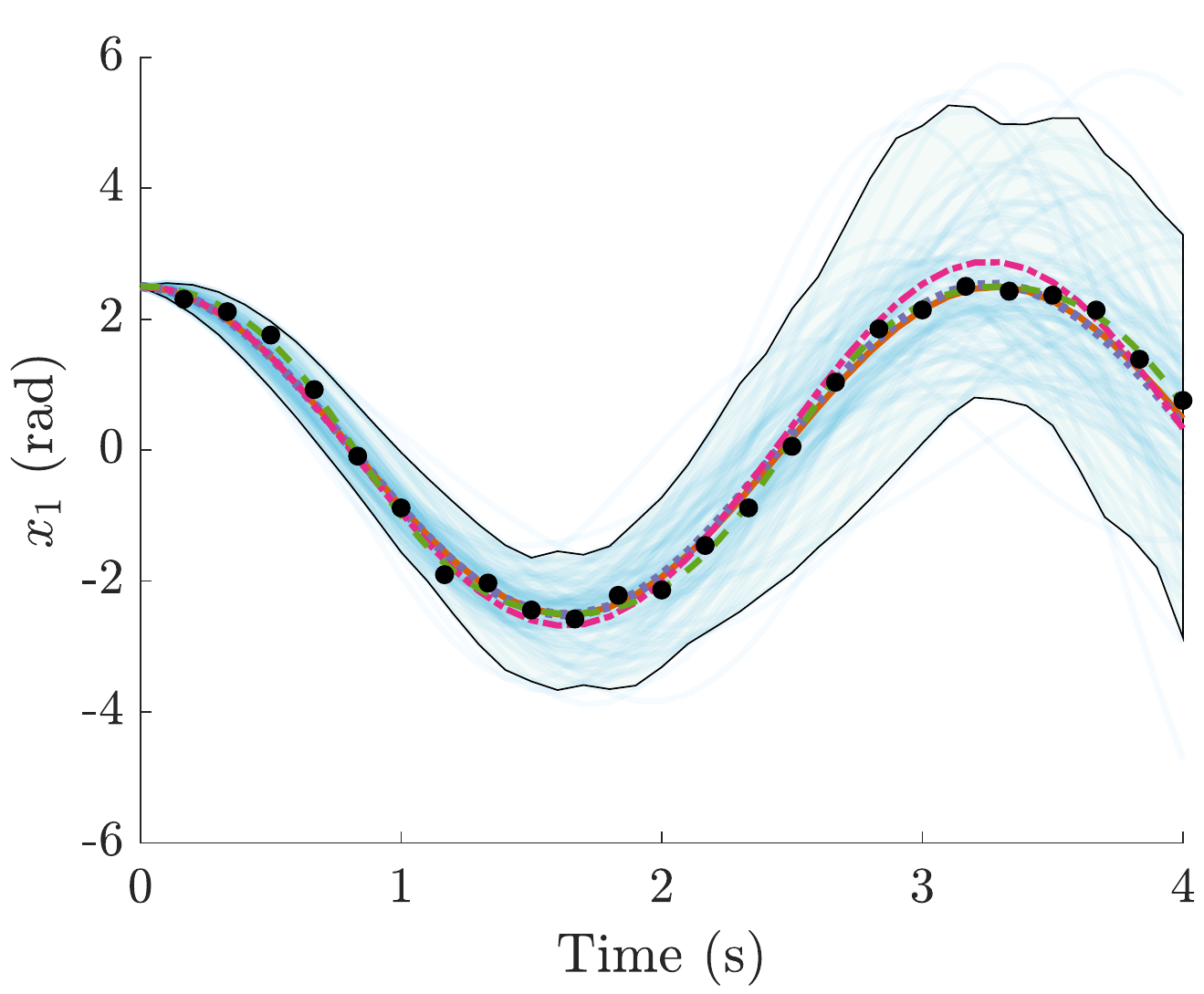}
        \caption{$x_1$, $n=24$, $\sigma=10^{-1}$ }
        \label{fig:nl24_1reconx1}
    \end{subfigure}
    \begin{subfigure}{0.56\linewidth}
        \includegraphics[width=\linewidth]{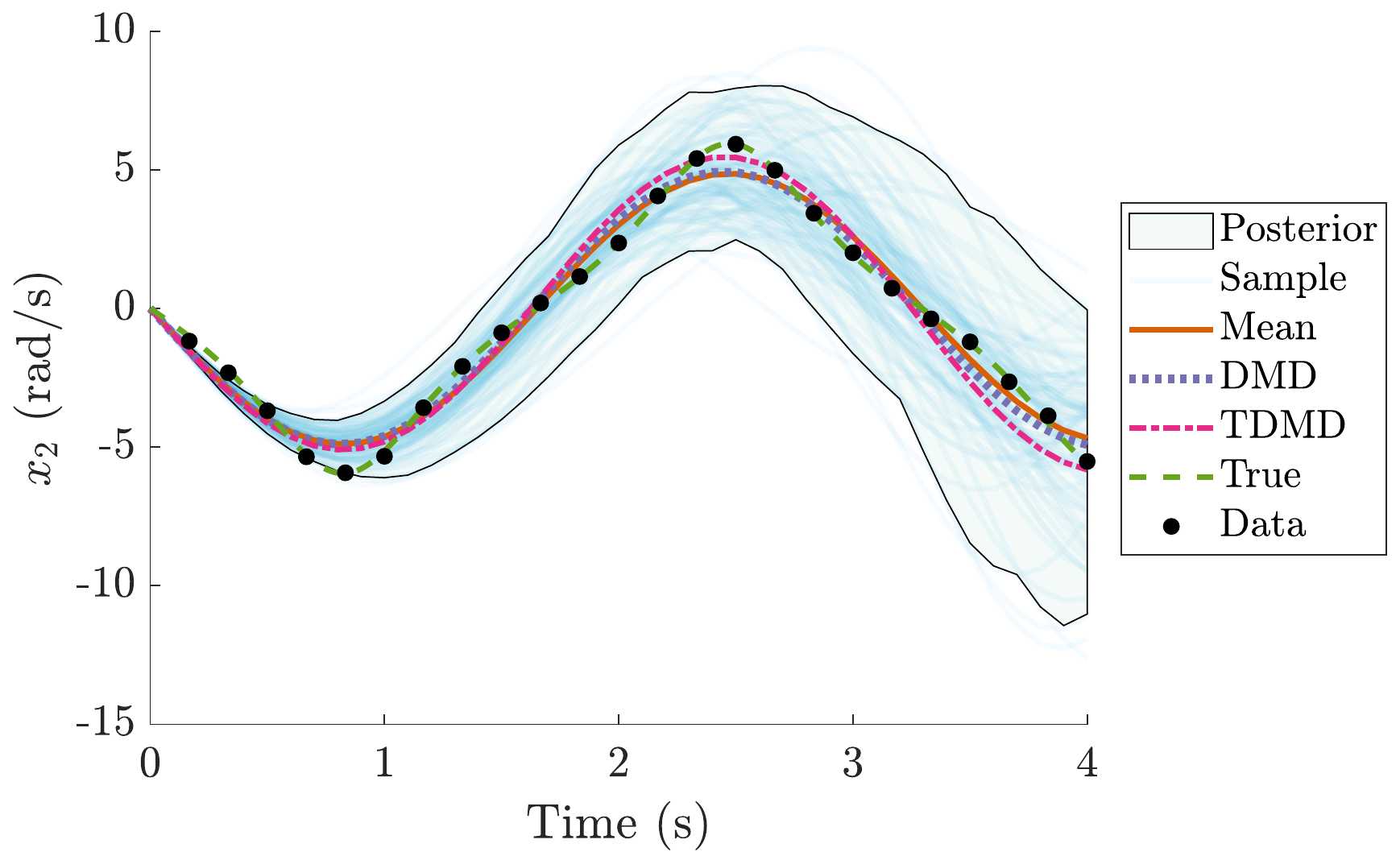}
        \caption{$x_2$, $n=24$, $\sigma=10^{-1}$}
        \label{fig:nl24_1reconx2}
    \end{subfigure}
    \begin{subfigure}{0.43\linewidth}
        \includegraphics[width=\linewidth]{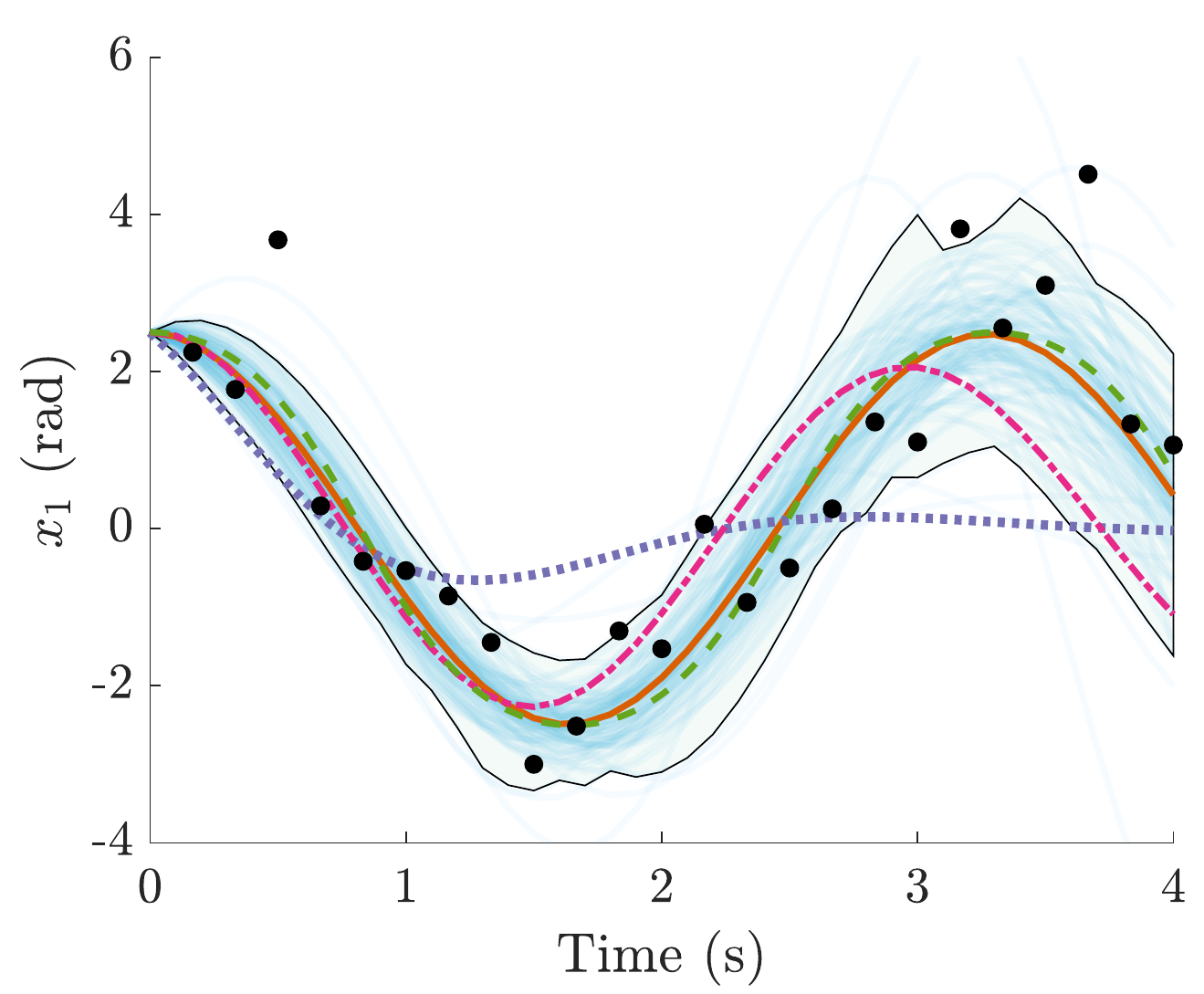}
        \caption{$x_1$, $n=24$, $\sigma=1$}
        \label{fig:nl24_10reconx1}
    \end{subfigure}
    \begin{subfigure}{0.56\linewidth}
        \includegraphics[width=\linewidth]{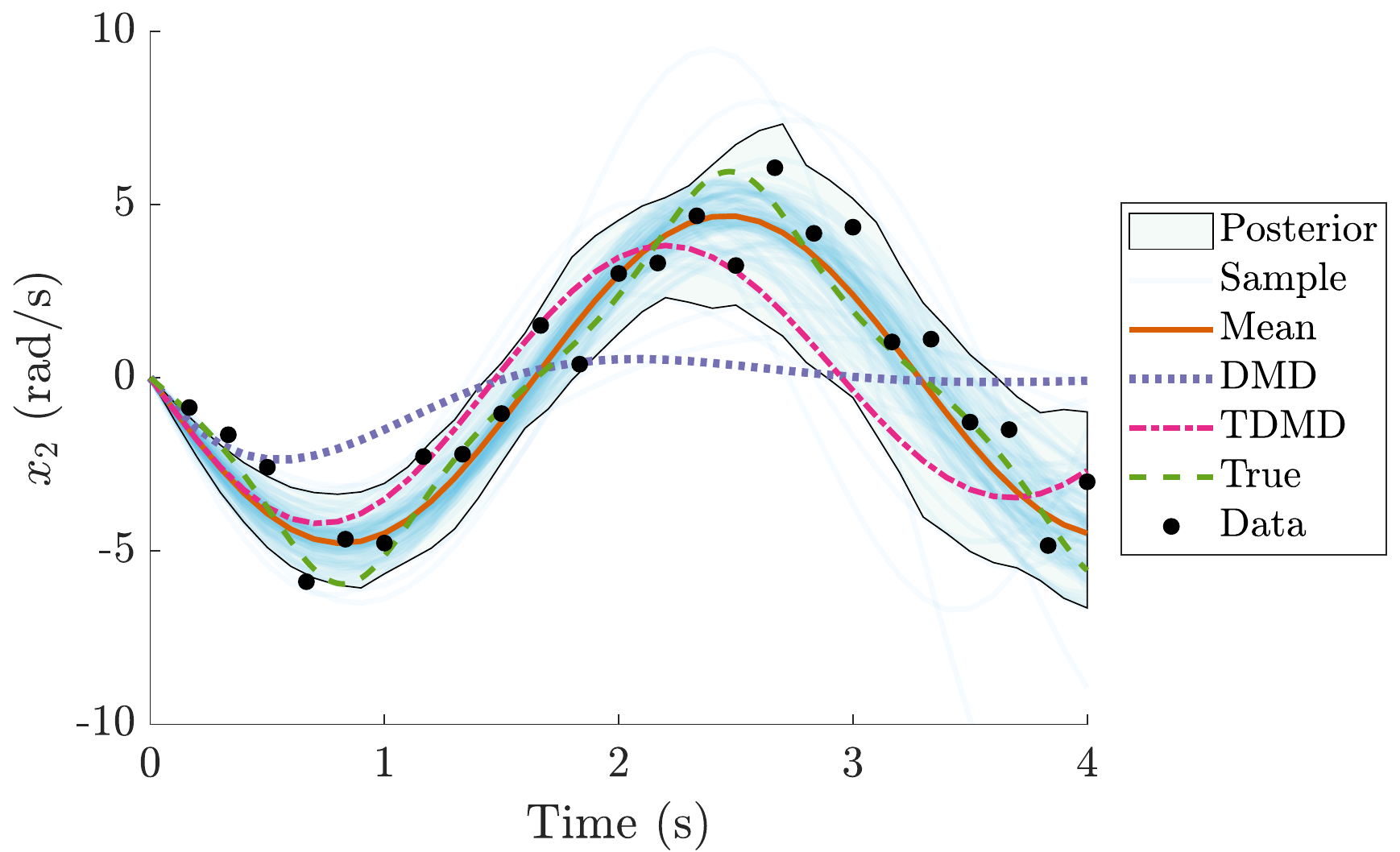}
        \caption{$x_2$, $n=24$, $\sigma=1$}
        \label{fig:nl24_10reconx2}
    \end{subfigure}
    \caption{Reconstruction performance for low-noise (top row) and high-noise (bottom row) data sets for the nonlinear pendulum using a linear model. All three estimates capture the truth closely in the low noise case, but only the Bayesian algorithm performs well (it is in phase and approximately the correct amplitude) for the high noise case. Furthermore, it reflects the additional uncertainty resulting from the simplistic linear model through its posterior distribution as compared to the results in Figure~\ref{fig:recon_linear}.  In the high noise case, DMD fails as it did for the linear pendulum, and TDMD underestimates both the period and amplitude of the pendulum's oscillations.}
    \label{fig:recon_nonlinear}
\end{figure}

\begin{figure}
    \centering
    \begin{subfigure}{0.43\linewidth}
        \includegraphics[width=\linewidth,clip,trim=0 0 100 0]{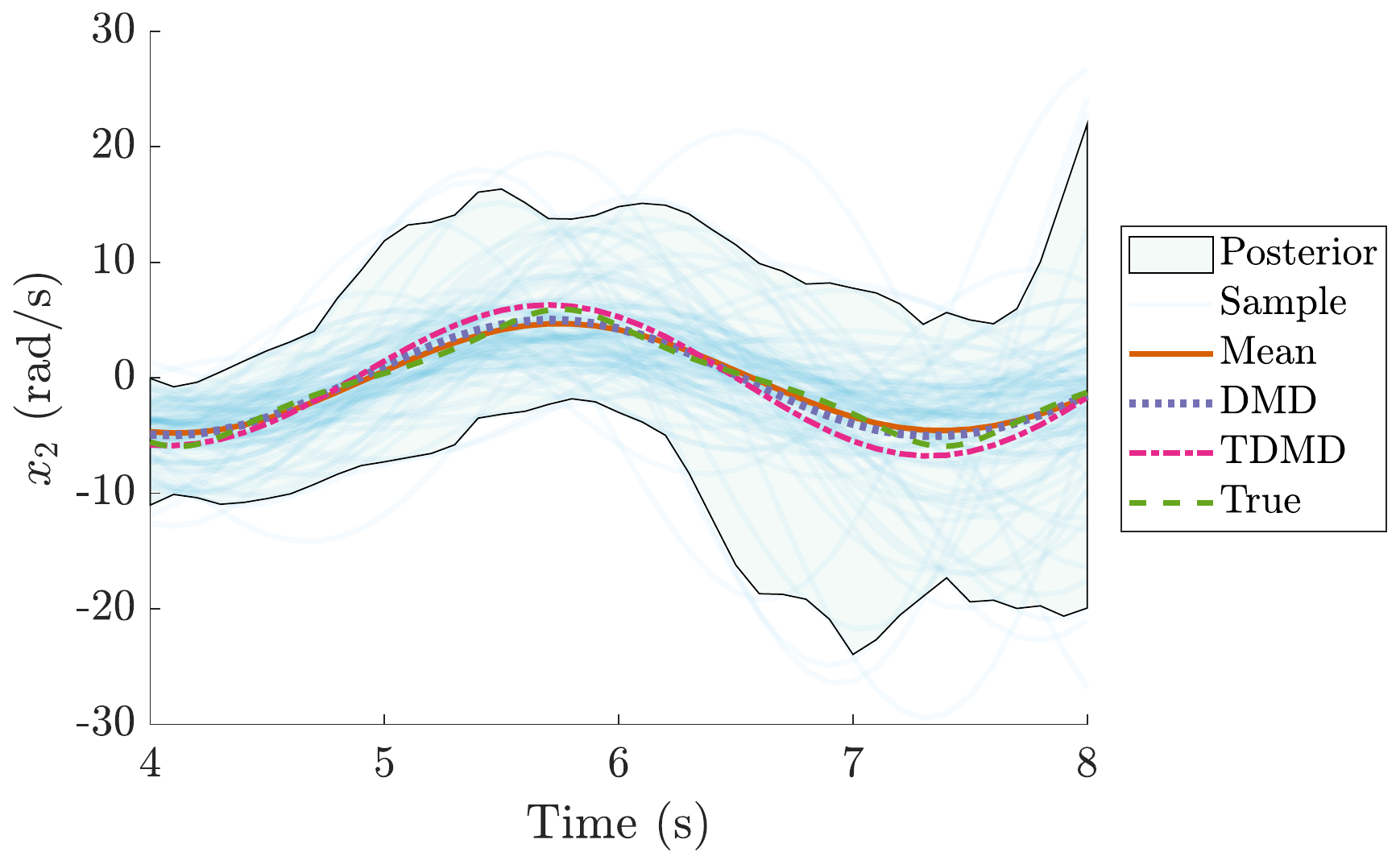}
        \caption{$x_2$ Prediction}
        \label{fig:nl24_1predictx2}
    \end{subfigure}
    \begin{subfigure}{0.56\linewidth}
        \includegraphics[width=\linewidth]{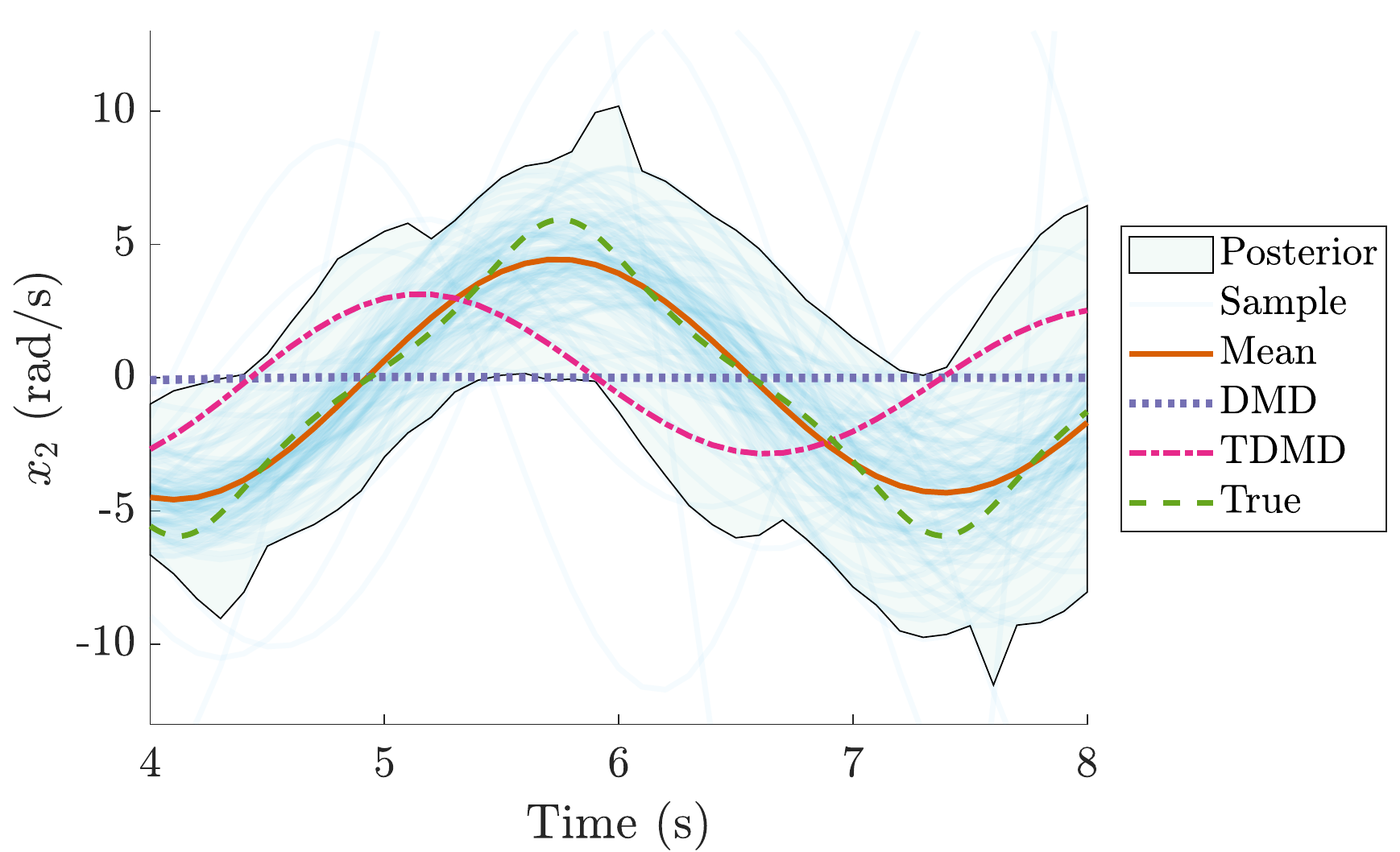}
        \caption{$x_2$ Prediction}
        \label{fig:nl24_10predictx2}
    \end{subfigure}
    \caption{Comparison shown here is the same as in Figure~\ref{fig:predict_linear}, but this time for a nonlinear pendulum truth model. In the low noise case, the estimates are all visually aligned with the truth. In the high noise case, DMD fails and TDMD falls out of phase, but the Bayesian algorithm remains robust and produces an accurate estimate.}
    \label{fig:predict_nonlinear}
\end{figure}

Next we investigate what the Bayesian approach learns for the process and measurement noise in the case where there is a model error. The  marginal and joint posterior distributions for both measurement noise cases are shown in Figure~\ref{fig:varHistNL}.   We observe that in the low noise case~\ref{fig:nlscatterHist24_1}, the joint distribution is bimodal.  The smaller mode corresponds to a model with low process noise and high measurement noise, and the larger mode corresponds to a model with high process noise and low measurement noise. Bayes has effectively uncovered that the data can be explained in one of two ways: either the model fits the true system well, but the data are very noisy, or the measurement noise is low and the model is not capable of properly capturing the dynamics.  In this case, the latter is true and is also the option that Bayes found to be much more likely.  For the high noise case~\ref{fig:nlscatterHist24_10}, the joint distribution is unimodal, conveying the possibility of only one process-measurement noise pairing.  Once again, the modes of both the measurement noise marginal distributions align closely with the truth shown in red.  Finally,  we see that the process noise magnitudes in both cases are much larger than those seen in the linear pendulum examples (Figure~\ref{fig:varHistLin}) as a consequence of trying to capture nonlinear dynamics with a linear model.

\begin{figure}
    \centering
    \begin{subfigure}{0.50\linewidth}
        \centering
        \includegraphics[width=\linewidth]{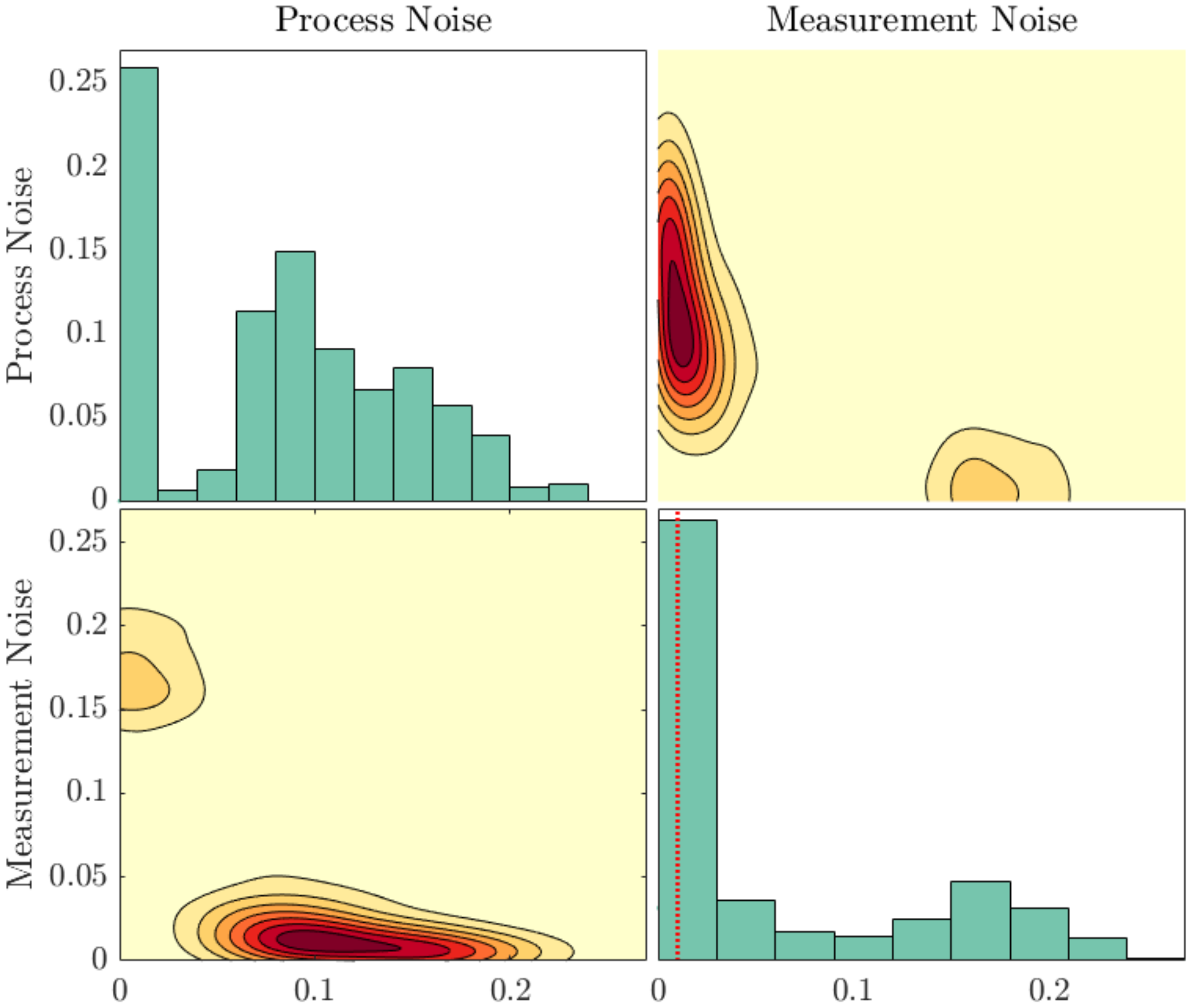}
        \caption{True measurement noise $\sigma= 0.1$}
        \label{fig:nlscatterHist24_1}
    \end{subfigure}
    \begin{subfigure}{0.48\linewidth}
        \centering
        \includegraphics[width=\linewidth]{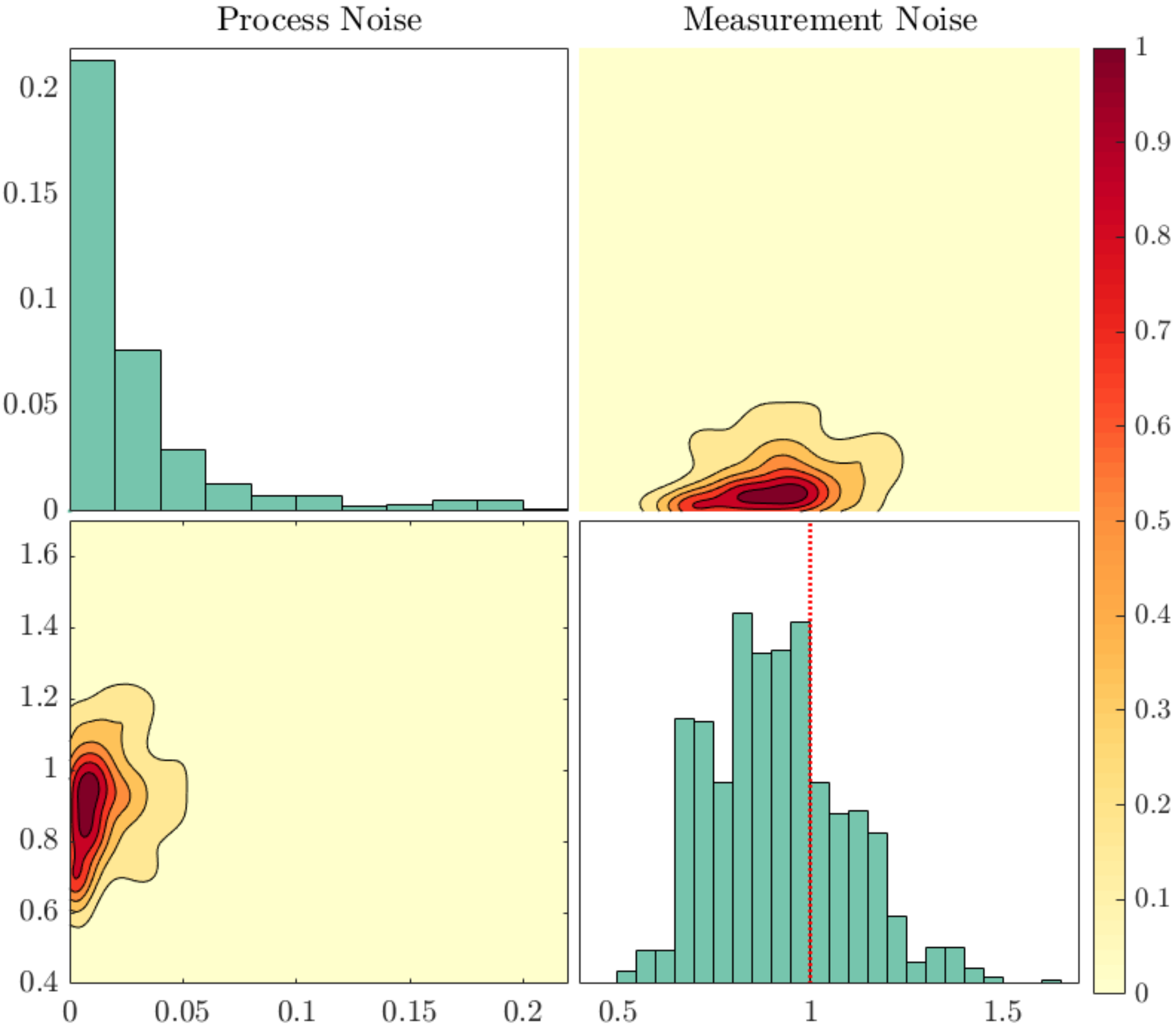}
        \caption{True measurement noise $\sigma=1.0$}
        \label{fig:nlscatterHist24_10}
    \end{subfigure}
    \caption{Marginal and joint posterior distributions of the process and measurement noise variance parameters during the recovery of the nonlinear pendulum. In the left panel, the joint distribution is bimodal, offering two possible models with the true case being strongly preferred.  In the right panel, all of the distributions are unimodal and in alignment with the truth.}
    \label{fig:varHistNL}
\end{figure}

\subsubsection{Discussion on diagnostics}
\label{subsubsec:diagnostics}
One of the strengths of the Bayesian approach is that it separates the learning stage from the decision making stage, so if the initial decision rule yields an unsatisfactory estimate, one can go back and analyze the posterior distribution to devise an improved decision rule.  It was noted earlier in Figure~\ref{fig:nlMSE_contours} that the average MSE of (T)DMD is lower than that of the average MSE of the Bayesian estimator over 500 data sets when the measurement noise is low.  This observation likely implies that there is a better decision rule that can be used to achieve performance at least as strong as DMD.  To understand how to best select a point from the posterior to be our estimate, we first look at the posterior over the states. Figure~\ref{fig:nlbadMean26_1} shows samples from the posterior predictive distribution for a single data set containing $n=26$ measurements with noise standard deviation of $\sigma=0.1$. The mean deviates from the truth near the peaks and valleys of the trajectory between about 2.5 and 4 seconds. This is the same location in which the posterior appears to be significantly spread in possible predictions. This presence of significant outliers is a result of the bimodal noise distribution previously discussed. Furthermore, it is clear that the mean is not a good estimator in the case of bimodal distributions; however, we see that there exists a mode in alignment with the truth.  Upon this realization, we can then craft a decision rule that selects this mode rather than the mean for improved performance. In this case, the mode-based rule would result in the Bayesian approach being 1.3 times better than the TDMD estimator. Moreover, this entire analysis can be done \textit{a posteriori}, and therefore uses no additional assumptions or requirements on our approach.

\begin{figure}
    \centering
    \includegraphics[width=0.5\linewidth]{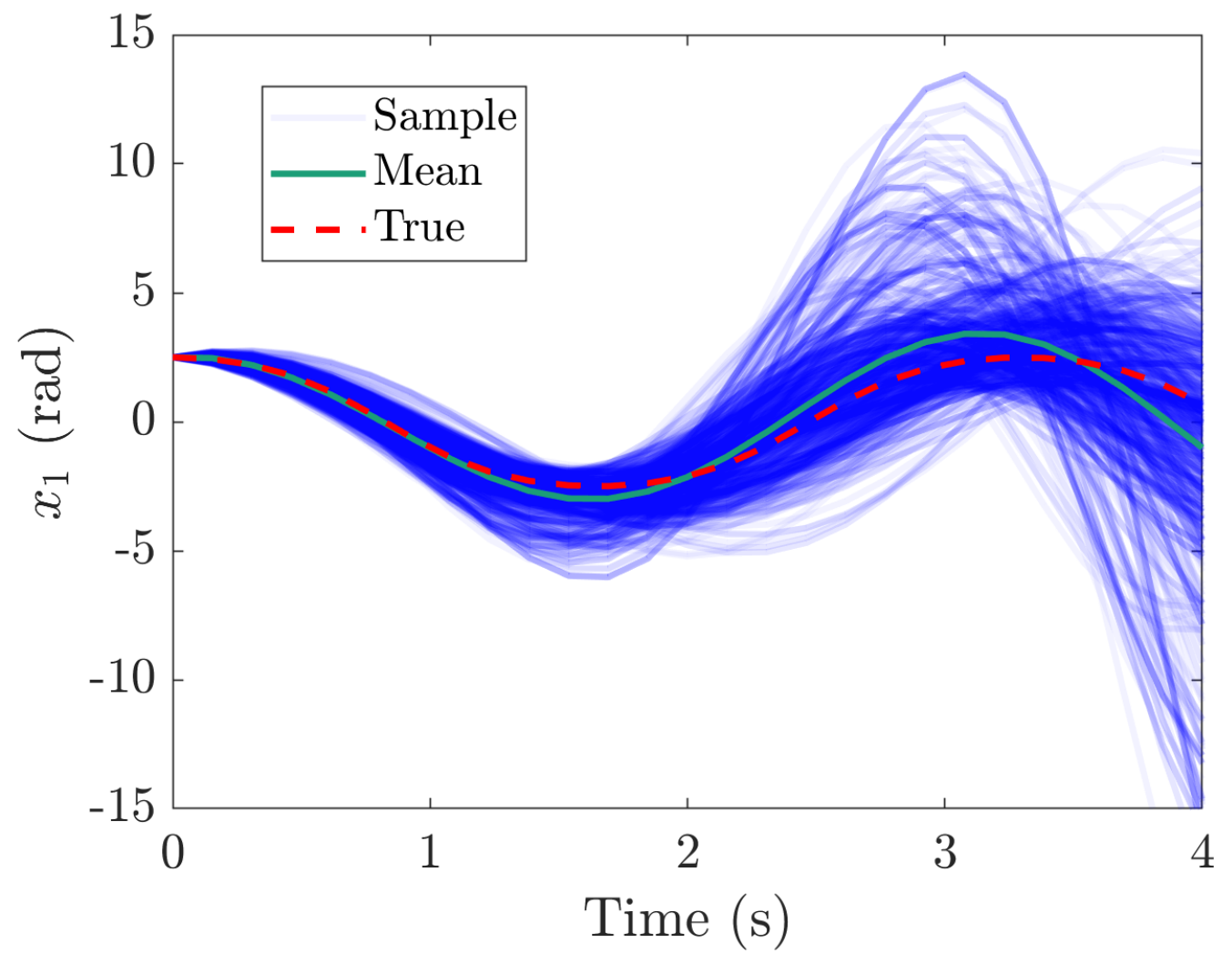}
    \caption{Posterior samples from a data set with $n=26$ data points with noise standard deviation of $\sigma=0.1$.  This data set produced the worst mean estimate out of the 500 with respect to MSE.  This figure illustrates that the mean deviates from the truth at the extrema of the curve where samples are skewed toward larger magnitudes.  Using a decision rule that selects the mode here would give a much better estimate.}
    \label{fig:nlbadMean26_1}
\end{figure}

We also note that the effect this has on the MSE ratio appears more strongly in this nonlinear case for two reasons.  The first reason is that the higher process noise due to the model error and low measurement noise can create a bimodal distribution because of the alternate possibility of a good model with noisy data as shown earlier.  The second reason is that the ratio of process noise to measurement noise is higher than that in the linear case. As we have shown in Theorem~\ref{th:DMD}, the (T)DMD approaches effectively assume the existence of process noise but no measurement noise. In cases where the linear and nonlinear models are mismatched this becomes a better assumption.

In summary, for cases where the model error can be significant, a non-mean estimator should be extracted from the Bayesian posterior. This estimator should be chosen by considering the bimodality of the learned process/measurement noise estimator, and can often be the peak of one of the modes. If this is done (it is an \textit{a posteriori} procedure), we have seen that it yields improved performance compared to (T)DMD.

\subsection{Optimal estimators and the Van der Pol oscillator}
\label{subsec:vanderpol}

Next we consider learning a sparse representation of a nonlinear system so that we can compare the Bayesian algorithm directly to SINDy. Here, it will once again be shown that factoring the process and measurement noise into our estimator will allow it to be robust even for noisy measurements.

Consider the nonlinear Van der Pol oscillator 
\begin{align}
    \begin{bmatrix}\dot{x}_1 \\ \dot{x}_2 \end{bmatrix} = 
    \begin{bmatrix} x_2 \\ \mu(1-x_1^2)x_2 - x_1 \end{bmatrix}, \qquad 
    x_0=\begin{bmatrix}0 \\ 2\end{bmatrix},
\end{align}
where $\mu=3$.  In this case, we use the SINDy algorithm rather than DMD to account for the nonlinear dynamics. For both the Bayesian and SINDy algorithms we therefore consider a subspace of right hand sides that is spanned by a set of candidate functions. We choose monomial candidates up to third degree and their interacting terms. As a result, each algorithm seeks to learn  20 dynamics parameters (10 for each state). The Bayesian algorithm is additionally tasked with learning the covariance matrices parameterized as follows:
\begin{equation}
    \Sigma(\thetsig) = \begin{bmatrix}\theta_{21} & 0 \\ 0 & \theta_{22}\end{bmatrix} \quad
    \Gamma(\thetgam) = \theta_{23}I_{2\times2}.
\end{equation}
The priors on the dynamics parameters are Laplace distributions with zero mean and on the variance parameters are once again half-normal distributions.

We consider two cases: one where SINDy shows strong performance, and one in which SINDy struggles, and we show that the Bayesian algorithm yields an accurate estimate in both cases.  The case in which SINDy excels is frequent and low noise data.  Here, $n=2,000$ measurements were taken over the course of 20 seconds with measurement noise standard deviation of $\sigma=10^{-3}$.  In the opposite case, we collect only $n=200$ measurements over 20 seconds with measurement noise standard deviation of $\sigma=2.5\times 10^{-1}$.

The reconstructions from these experiments are shown in Figure~\ref{fig:recon_vanderpol}, predictions are given in Figure~\ref{fig:predict_vanderpol}, and the phase plots over 200 seconds are given in Figure~\ref{fig:phase_vanderpol}.  Here, the mode represents the mode of the posterior predictive distribution.  In the low noise case, we see that the Bayesian algorithm and SINDy both capture the dynamics very closely.  We see that SINDy agrees slightly more closely with the trajectory as a result of its hard threshold regularization.  Note that the posterior in this case is very small because the high number of data points and low measurement noise gives us high certainty in our estimate.  In the high noise case, we see that SINDy gives a similar result to what DMD gave when the measurements were noisy: the trajectory immediately flatlines.  When the data are noisy like this, the procedure for SINDy is to denoise the data using total variation (TV) regularization~\cite{Chartrand_2011} before executing the algorithm.  However, the increased timestep between data makes it difficult to accurately denoise the data, and when the TV regularization is performed, SINDy ends up giving an unstable estimate. The Bayesian approach, however, is still able to identify the dynamics of the Van der Pol system.  The posterior in this high noise case is wider, signifying that the estimate holds more uncertainty than the low noise and frequent measurements case.

\begin{figure}
    \centering
    \begin{subfigure}{0.43\linewidth}
        \includegraphics[width=\linewidth]{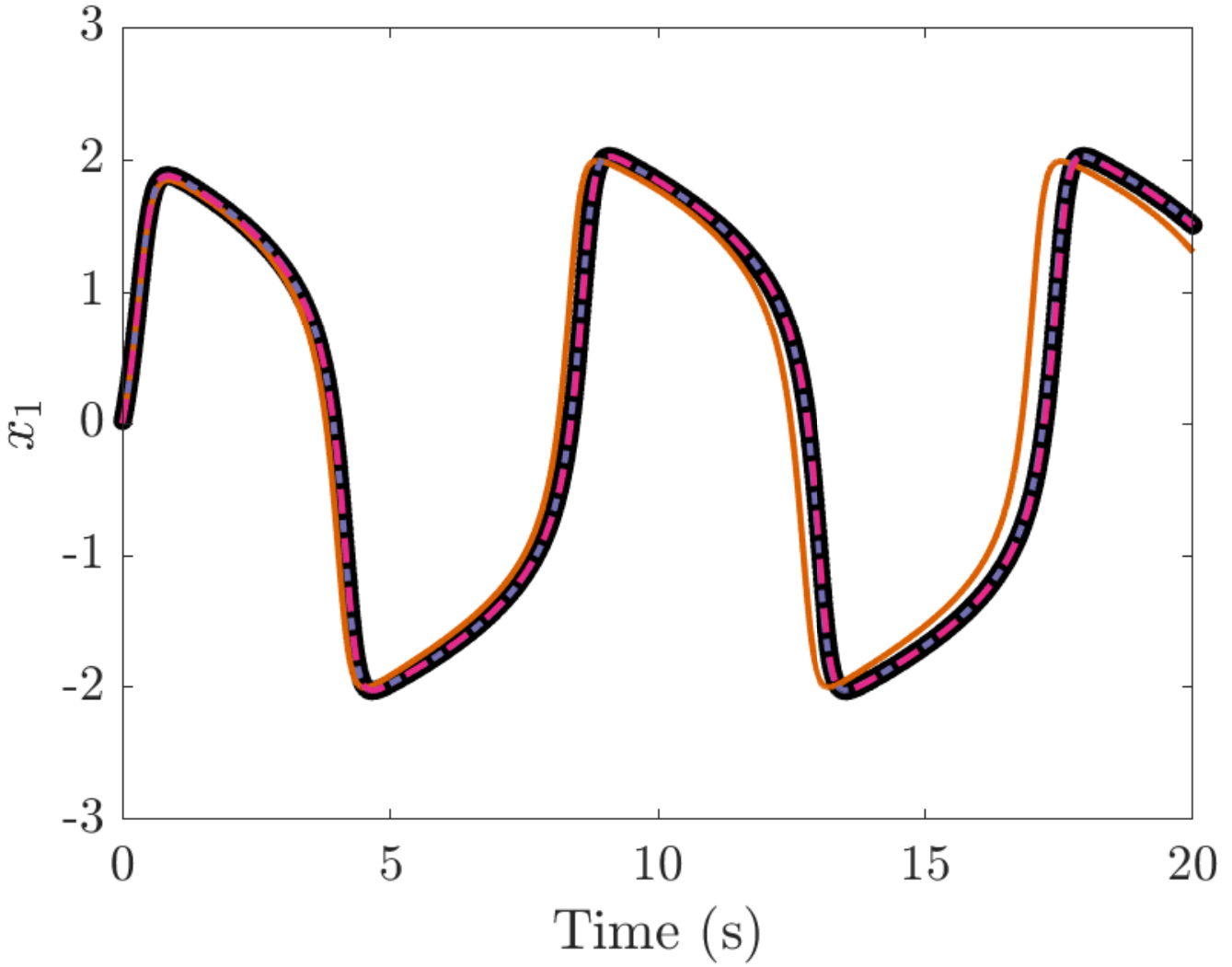}
        \caption{Reconstruction of $x_1$}
        \label{fig:vanderpolReconLowx1}
    \end{subfigure}
    \begin{subfigure}{0.56\linewidth}
        \includegraphics[width=\linewidth]{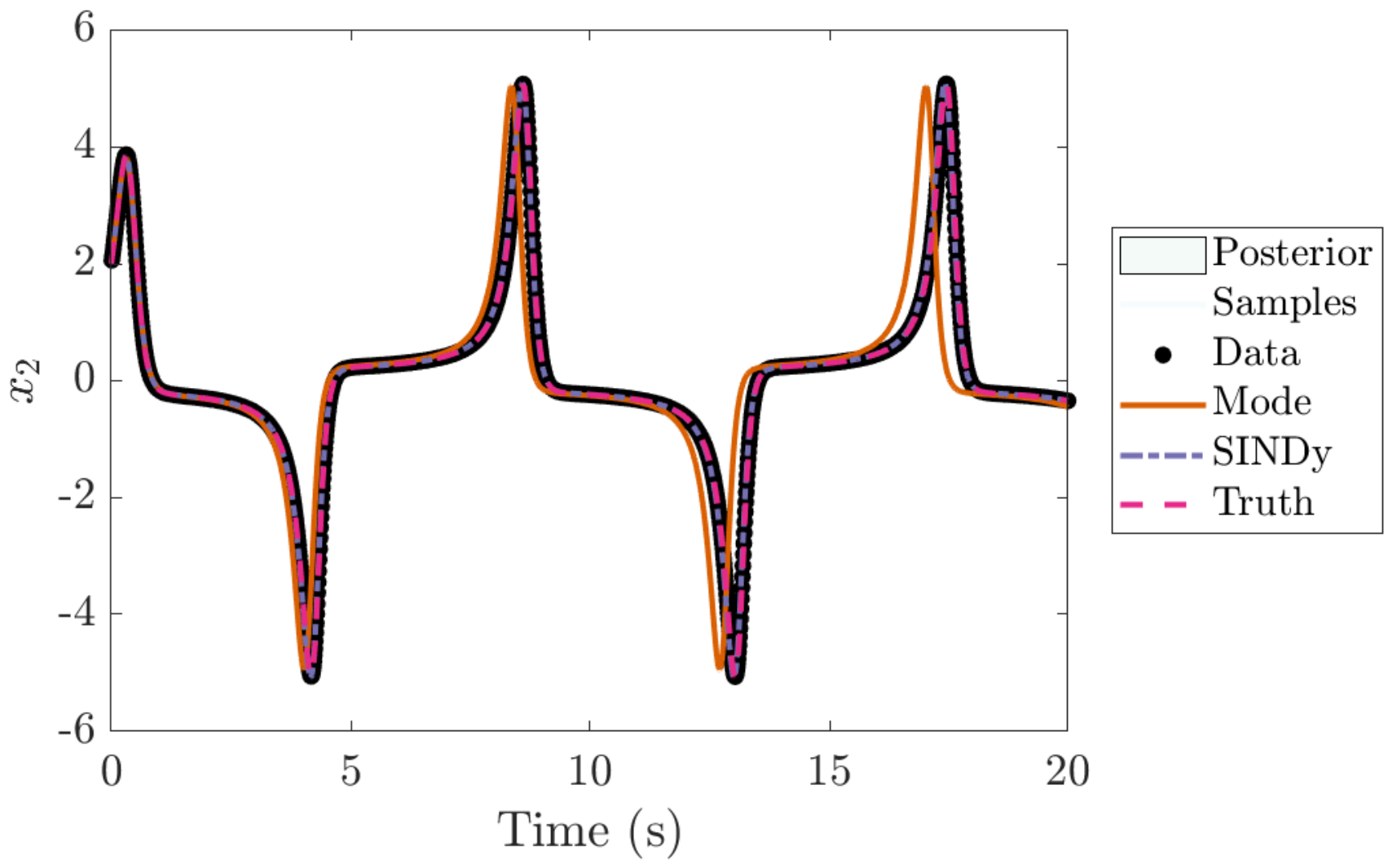}
        \caption{Reconstruction of $x_2$}
        \label{fig:vanderpolReconLowx2}
    \end{subfigure}
    \begin{subfigure}{.43\linewidth}
        \includegraphics[width=\linewidth]{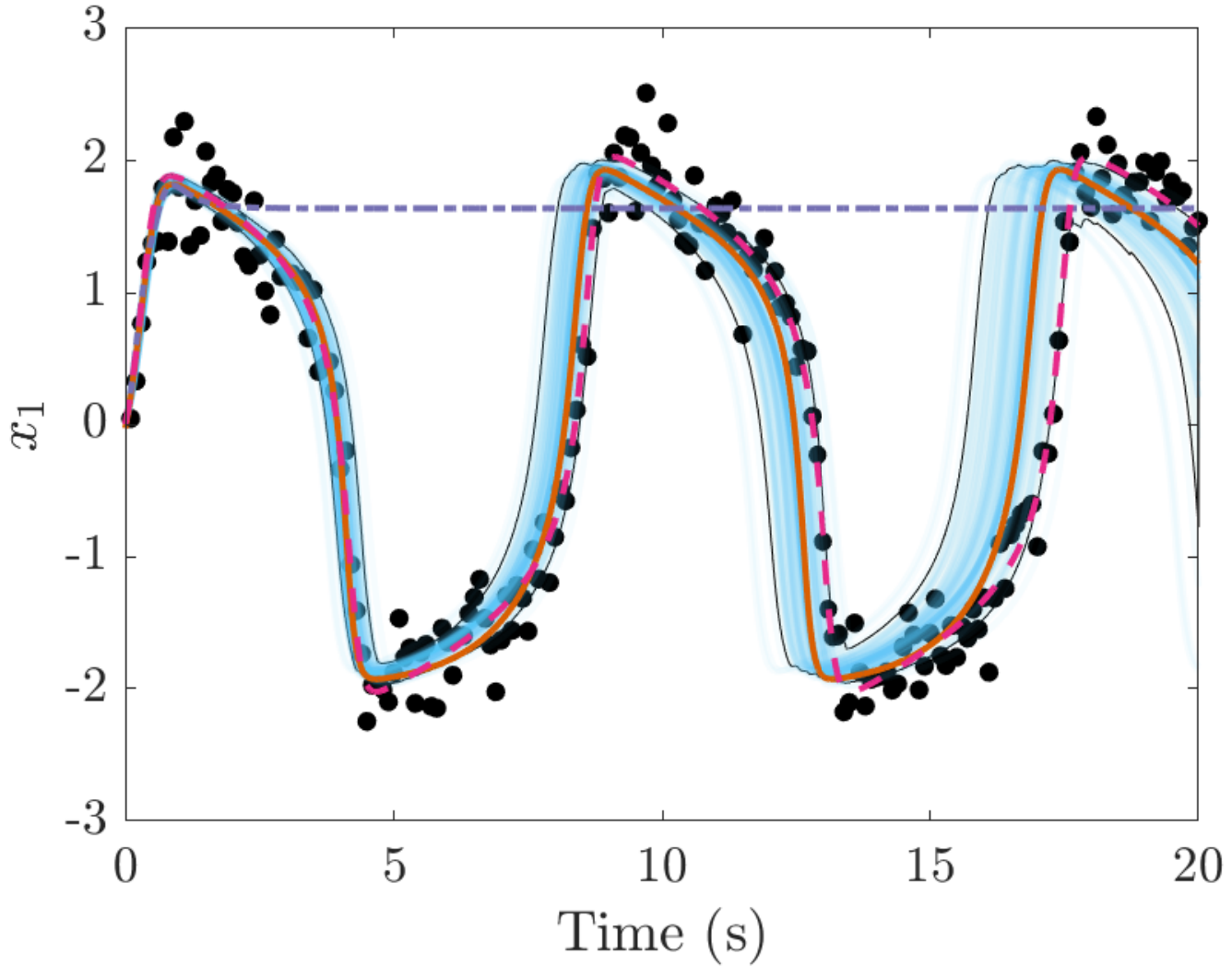}
        \caption{Reconstruction of $x_1$}
        \label{fig:vanderpolReconHighx1}
    \end{subfigure}
    \begin{subfigure}{.56\linewidth}
        \includegraphics[width=\linewidth]{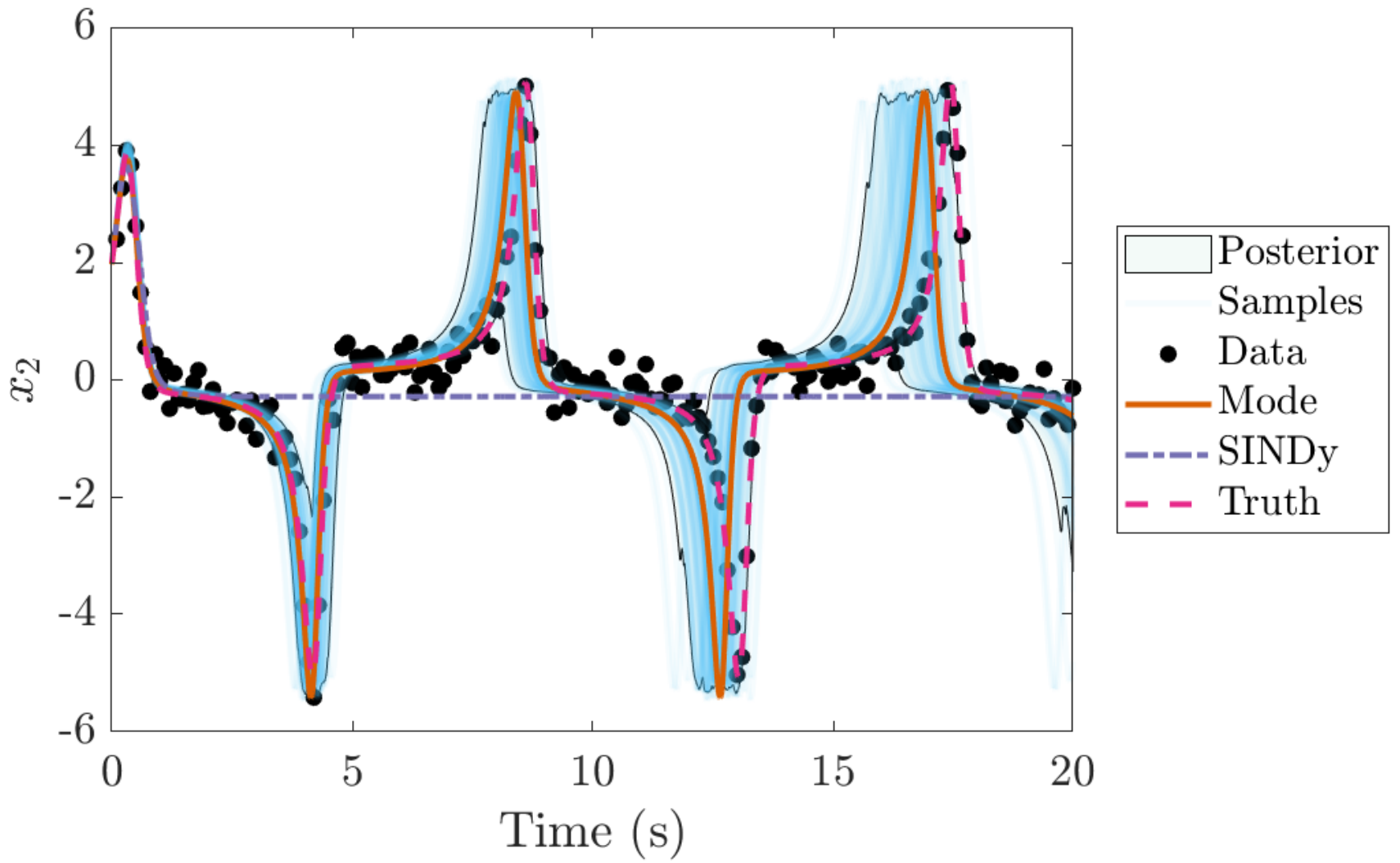}
        \caption{Reconstruction of $x_2$}
        \label{fig:vanderpolReconHighx2}
    \end{subfigure}
    \caption{Comparison of reconstruction error amongst the Bayesian and SINDy algorithms for different levels of noise and measurements for the Van der Pol system.  Top row corresponds to a low-noise/dense-data case, and the bottom row corresponds to a high-noise/sparse-data case.  Left column corresponds to the first state (position), and right column corresponds to the second state (velocity).  The Bayesian estimator is able to accurately reconstruct the dynamics, even in the presence of high noise.}
    \label{fig:recon_vanderpol}
\end{figure}

\begin{figure}
    \centering
    \begin{subfigure}{0.43\linewidth}
        \includegraphics[width=\linewidth]{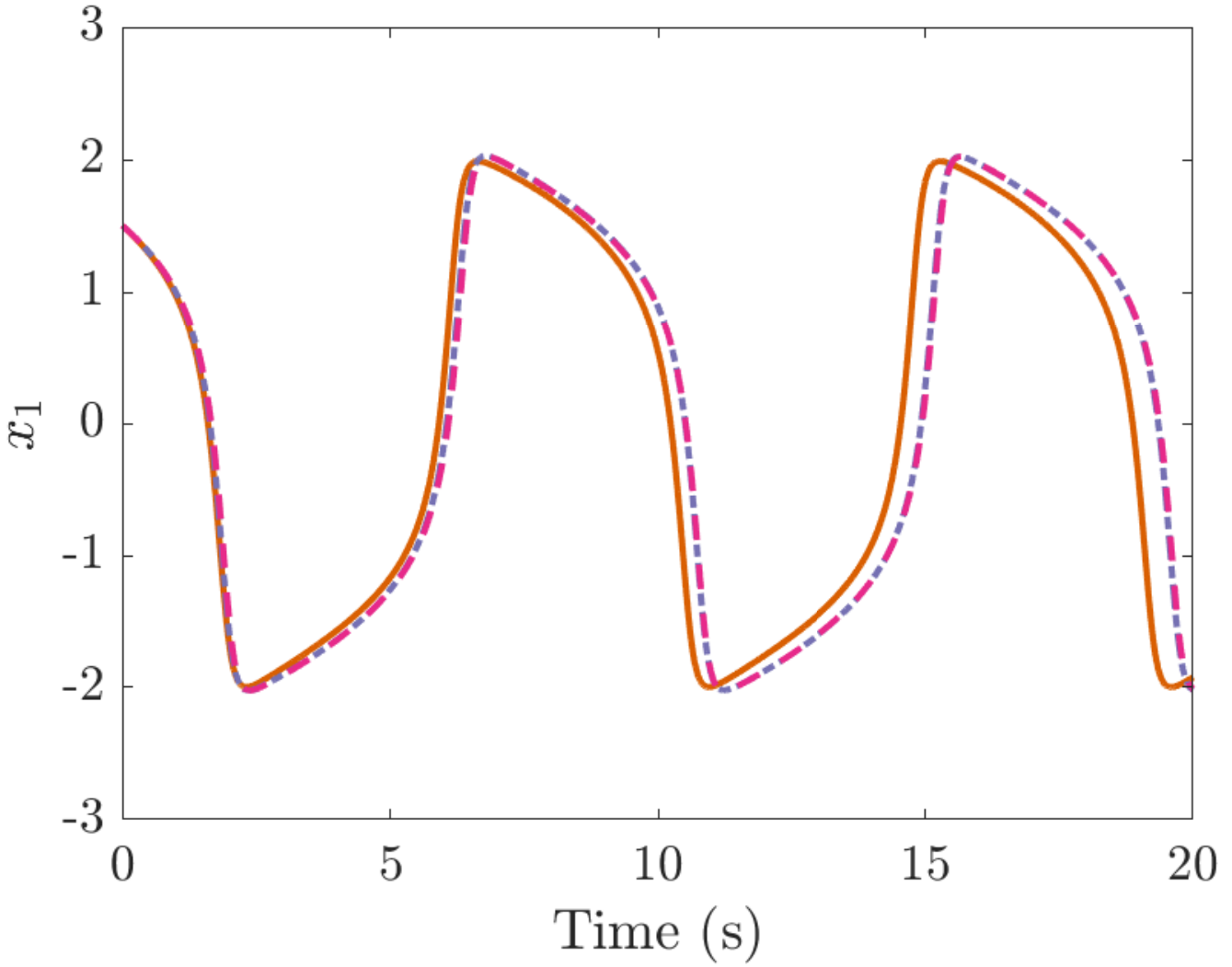}
        \caption{Prediction of $x_1$}
        \label{fig:vanderpolPredictLowx1}
    \end{subfigure}
    \begin{subfigure}{0.56\linewidth}
        \includegraphics[width=\linewidth]{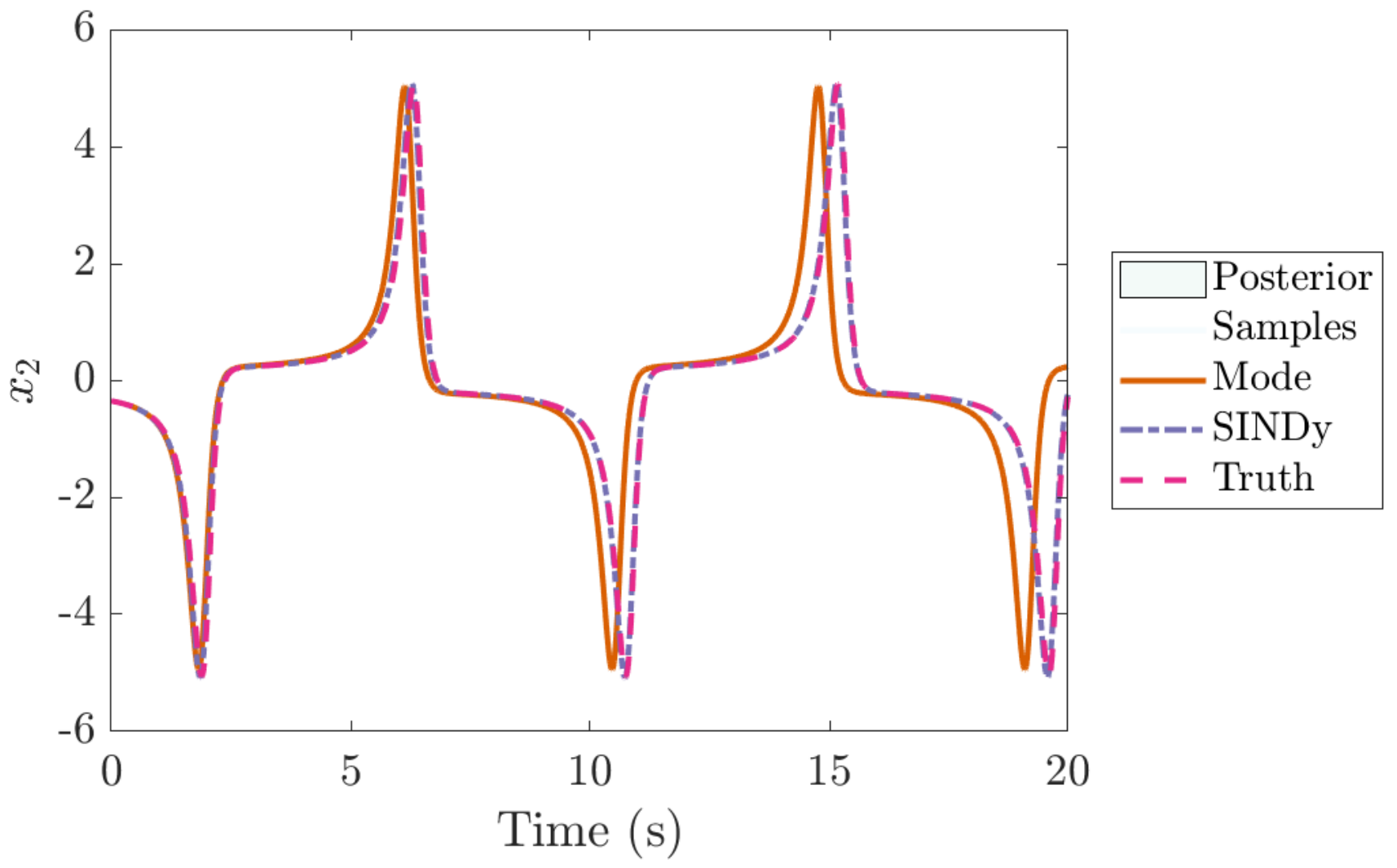}
        \caption{Prediction of $x_2$}
        \label{fig:vanderpolPredictLowx2}
    \end{subfigure}
    \begin{subfigure}{0.43\linewidth}
        \includegraphics[width=\linewidth]{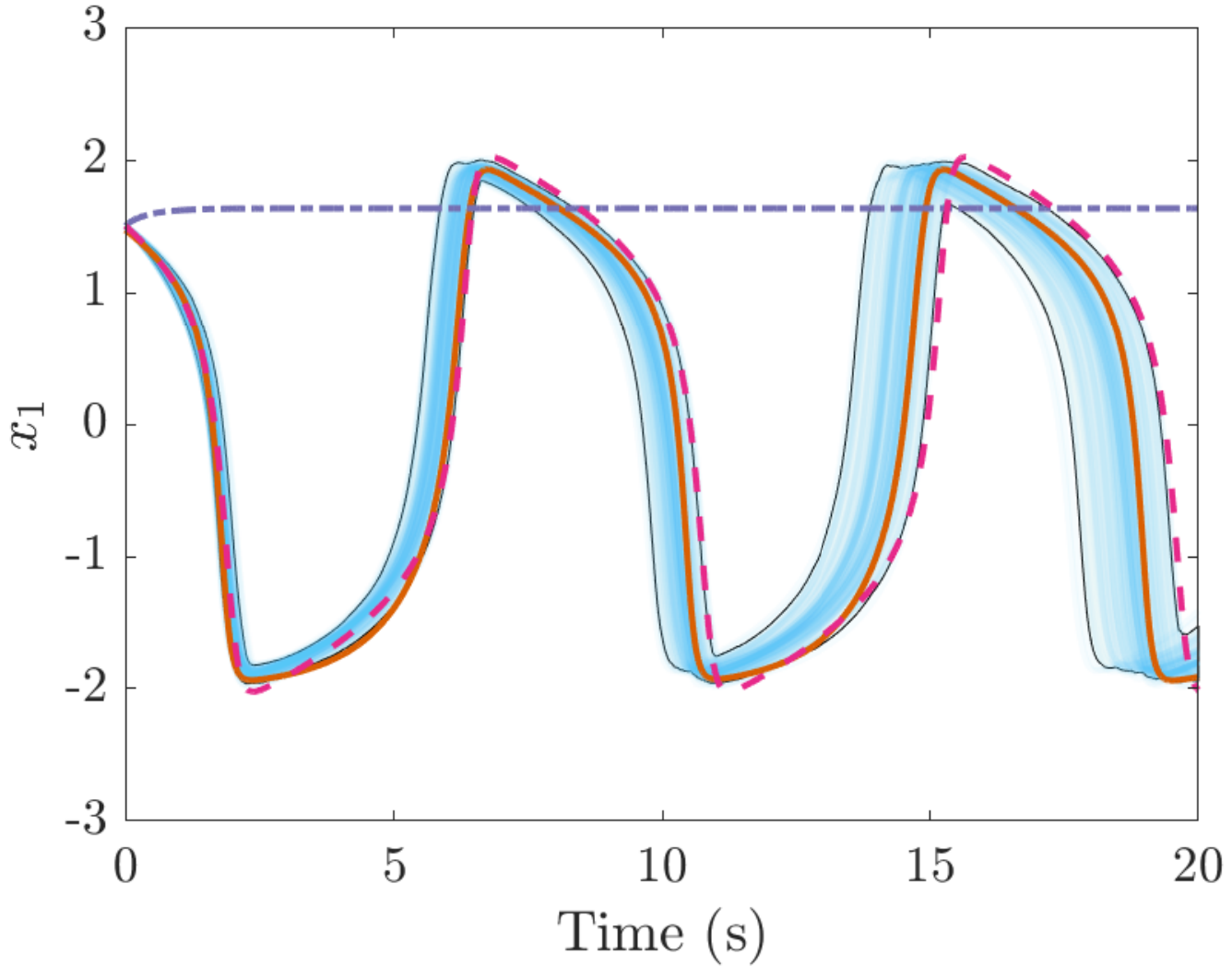}
        \caption{Prediction of $x_1$}
        \label{fig:vanderpolPredictHighx1}
    \end{subfigure}
    \begin{subfigure}{0.56\linewidth}
        \includegraphics[width=\linewidth]{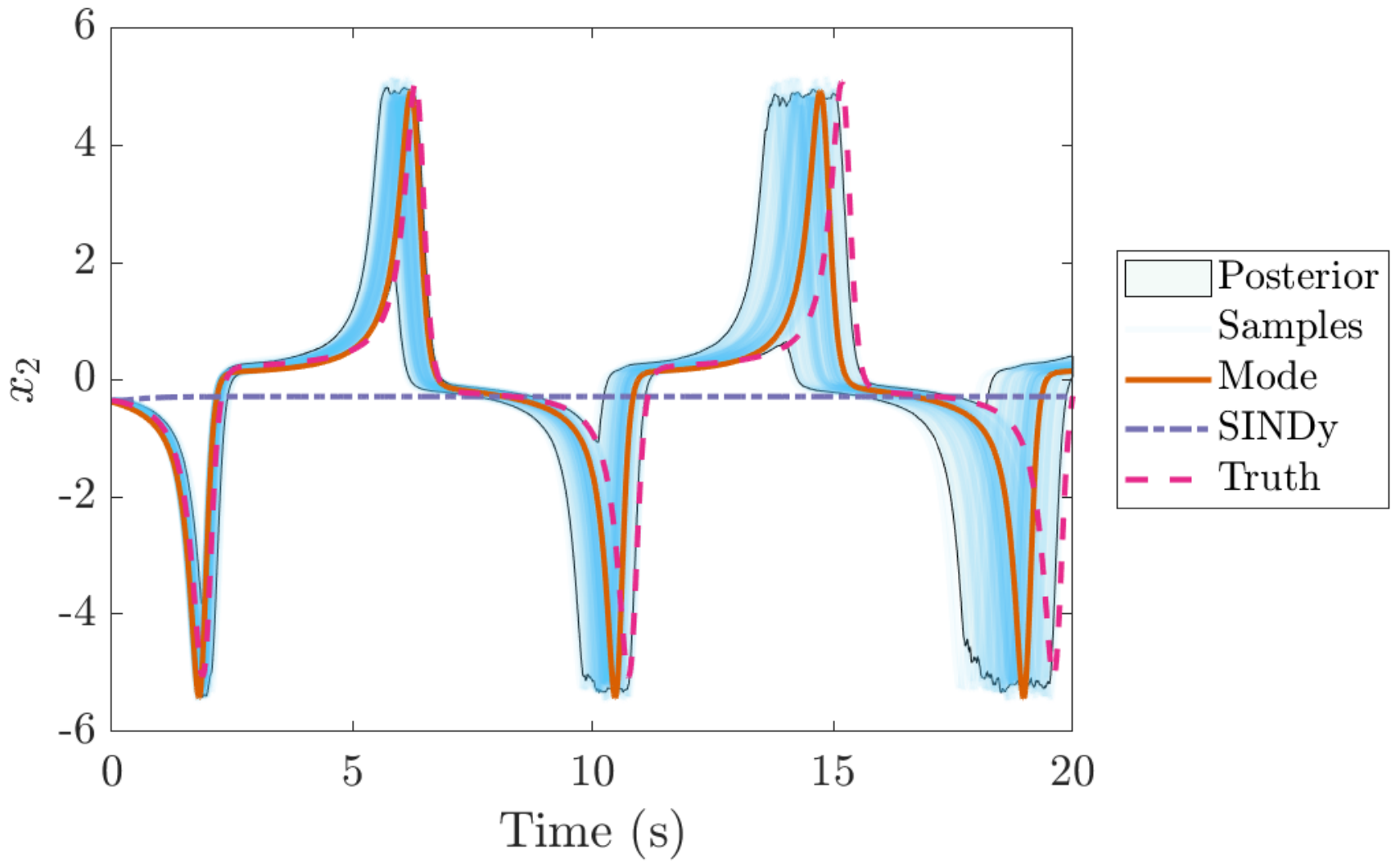}
        \caption{Prediction of $x_2$}
        \label{fig:vanderpolPredictHighx2}
    \end{subfigure}
    \caption{Comparison of prediction error amongst the Bayesian and SINDy algorithms for different levels of noise and measurements for the Van der Pol system.  The meaning of the figures is the same as described in Figure~\ref{fig:recon_vanderpol}.  The model learned by the Bayesian estimator is still accurate at a different initial condition.}
    \label{fig:predict_vanderpol}
\end{figure}

\begin{figure}
    \centering
    \begin{subfigure}{0.45\linewidth}
    \centering
    \includegraphics[width=\linewidth]{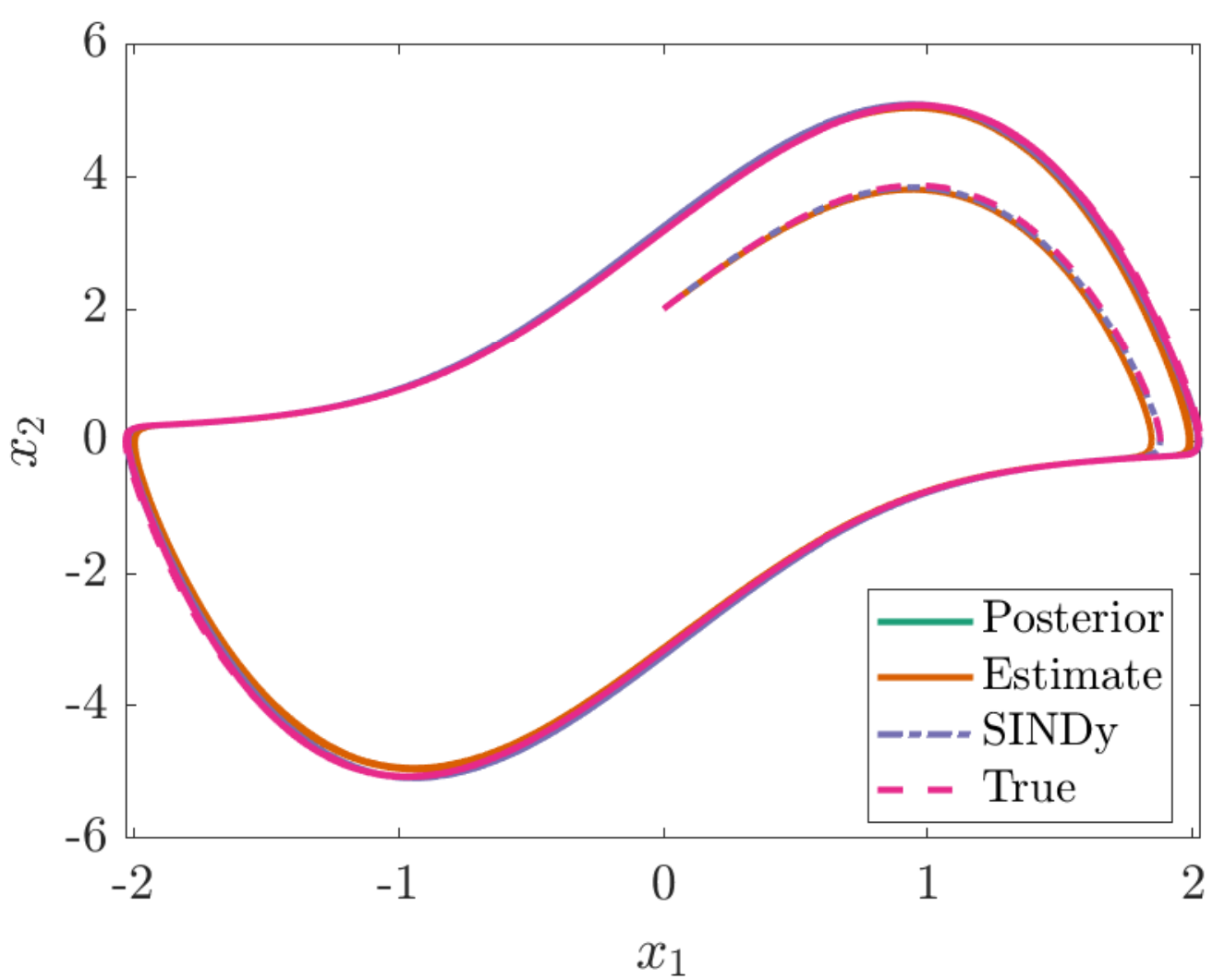}
    \caption{$\sigma=10^{-3}$, $n=2000$}
    \label{fig:vanderpolPhaseLow}
    \end{subfigure}
    \begin{subfigure}{0.45\linewidth}
        \centering
        \includegraphics[width=\linewidth]{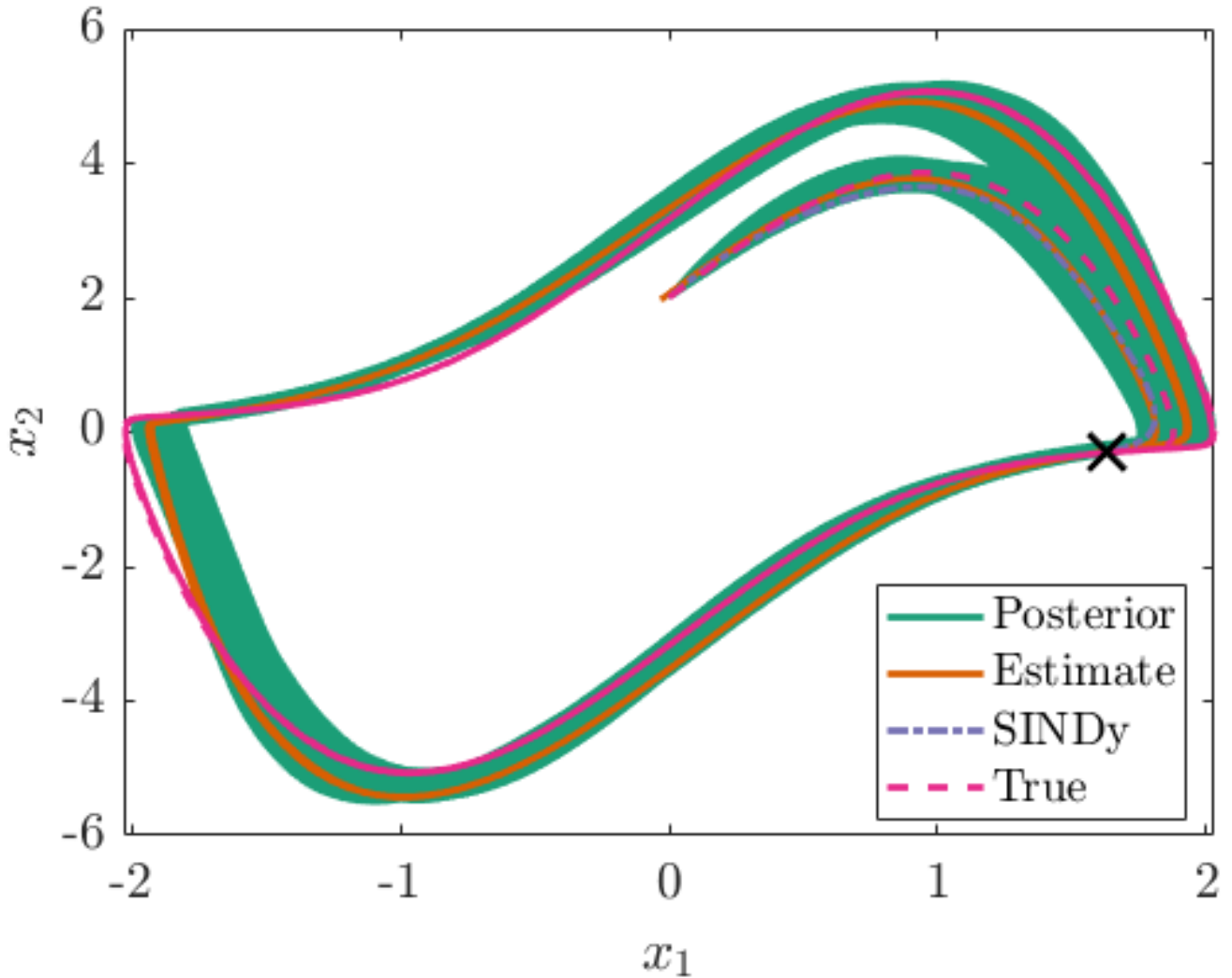}
        \caption{$\sigma=2.5\times10^{-1}$, $n=200$}
        \label{fig:vanderpolPhaseHigh} 
    \end{subfigure}
    \caption{Phase-diagram reconstruction for the Van der Pol oscillator under the two indicated data conditions. In the low-noise and frequent data domain, both the Bayesian and SINDy estimates lie directly on the truth. In the high noise-case, the Bayesian posterior is wider, but is still visually aligned with the truth.  The SINDy estimate is unable to recover the limit cycle, and the large ``x'' marks the equilibrium point to which SINDy converges, as shown in Figure~\ref{fig:predict_vanderpol}.}
    \label{fig:phase_vanderpol}
\end{figure}

\subsection{Known model form}

Finally, we consider the case where the model form is known, for instance from physical laws, but the parameters are uncertain. This is the classical inference setting and has seen a lot of development~\cite{Marzouk2007, Marzouk2009, Brastein_2018, Dokos_2004, Kivman_2003}, including in the computational physics community. However, much of this literature either only considers deterministic dynamics according to some variation of Equation~\eqref{eq:detobj} or only static problems. In this section, we consider both a chaotic system and a reaction-diffusion PDE in which we impose process noise to aid in the parameter estimation. For the reaction-diffusion PDE, this implies that the process noise is added to the discretized dynamics. Our results suggest that these methods are applicable to spatial problems and are able to effectively learn chaotic dynamics with a much smaller amount of data than observed in the literature.

\subsubsection{Lorenz 63}
We first consider the chaotic Lorenz 63 system~\cite{Lorenz_1963}
\begin{align}
    \begin{bmatrix}\dot{x}_1 \\ \dot{x}_2 \\ \dot{x}_3 \end{bmatrix} = 
    \begin{bmatrix} \theta_1(x_2-x_1) \\ x_1(\theta_2-x_3)-x_2 \\ x_1x_2-\theta_3 x_3 \end{bmatrix}, \qquad 
    x_0=\begin{bmatrix}2.0181 \\ 3.5065 \\ 11.8044\end{bmatrix}
\end{align}
The initial condition of this system was chosen so as to sit on the attractor. We attempt only to learn the parameters $\thetdyn = (\theta_1, \theta_2, \theta_3)$.  The difficulty with learning in chaotic systems is that the computation of the likelihood can be challenging. Since the likelihood involves running a filter, and filtering chaotic systems is well known to be challenging, it may seem that our approach would breakdown. Here we show that our Gaussian filtering approach is still able to learn an approximate dynamical system without resorting to more complicated likelihood building processes, e.g., using correlation integrals~\cite{Haario2015,Springer2019}. 

The priors on the dynamics parameters are once again improper and uniform. In addition to learning the model parameters in this example, we also learn the process noise variance for each state and the measurement noise variance for a total of seven parameters.  The parameterizations of the covariance matrices are shown:
\begin{align}
    \begin{array}{lr}\Sigma(\thetsig) = \begin{bmatrix}\theta_4 & 0 & 0 \\ 0 & \theta_5 & 0 \\ 0 & 0 & \theta_6\end{bmatrix}, & \Gamma(\thetgam) = \theta_7I_{3\times3}, \end{array}
\end{align}
with half-normal priors as before.

One hundred data points uniformly spaced over ten seconds are collected with a true measurement noise standard deviation of 2.0.  The predicted state trajectories after 10 seconds of simulation using the parameter posterior mode are shown in Figure~\ref{fig:LorenzStates}.  Similar to the Van der Pol oscillator, the dynamics exist on a low-dimensional attractor in phase space, and the wide, but constant, posterior distribution once again reflects this fact. Figure~\ref{fig:LorenzAttractors} shows the reconstructed and predicted attractors from the Bayesian algorithm. These figures show that while we cannot accurately capture the state, indeed all methods would eventually break down due to the chaotic nature of the system, we do predict a qualitatively similar attractor. As such, one would expect that most post-processing of these attractors, e.g., for control, would yield similar results.

\begin{figure}
    \centering
    \begin{subfigure}{0.3\linewidth}
        \includegraphics[width=\linewidth]{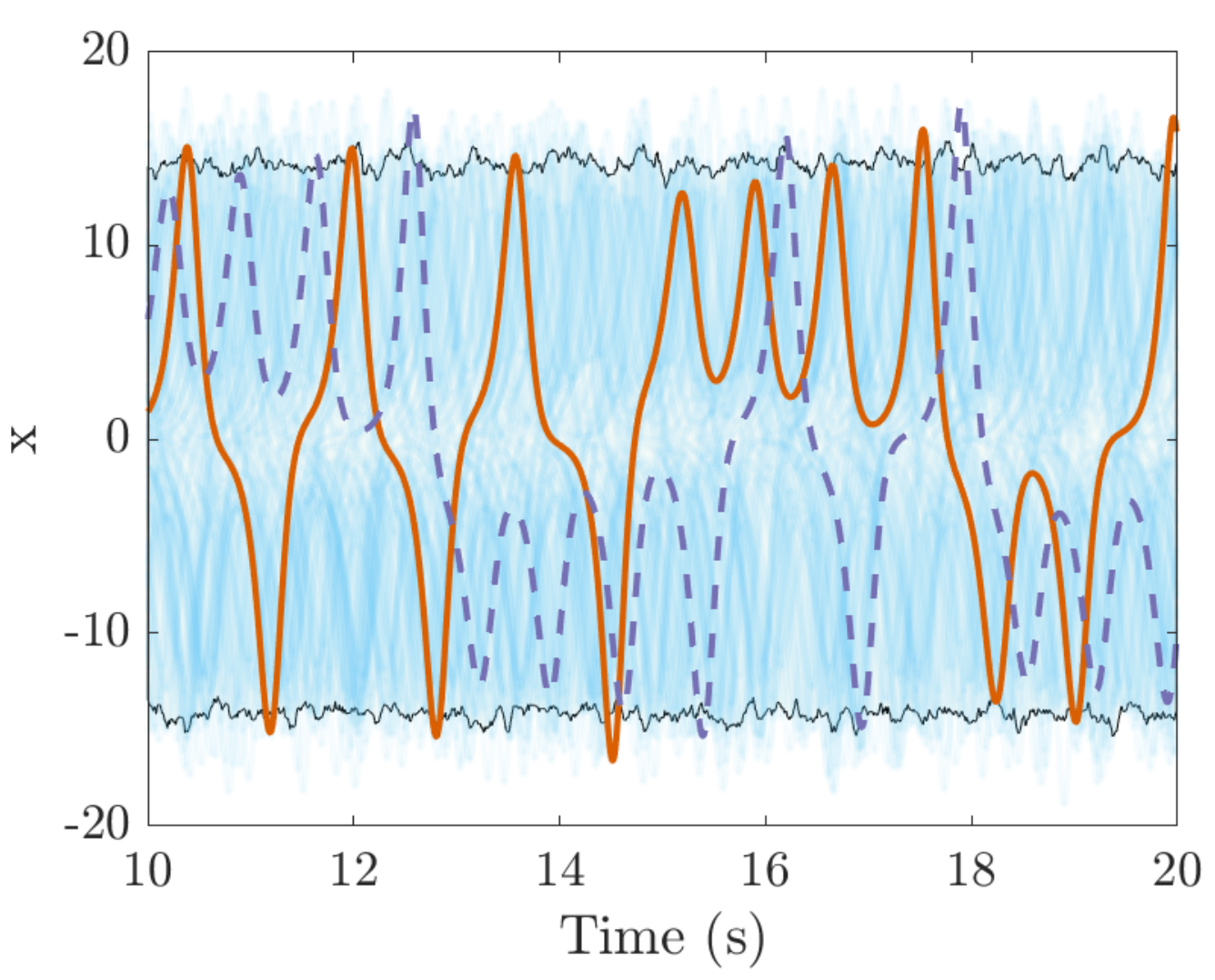}
        \caption{Prediction of $x$}
        \label{fig:LorenzPredictx}
    \end{subfigure}
    \begin{subfigure}{0.3\linewidth}
        \includegraphics[width=\linewidth]{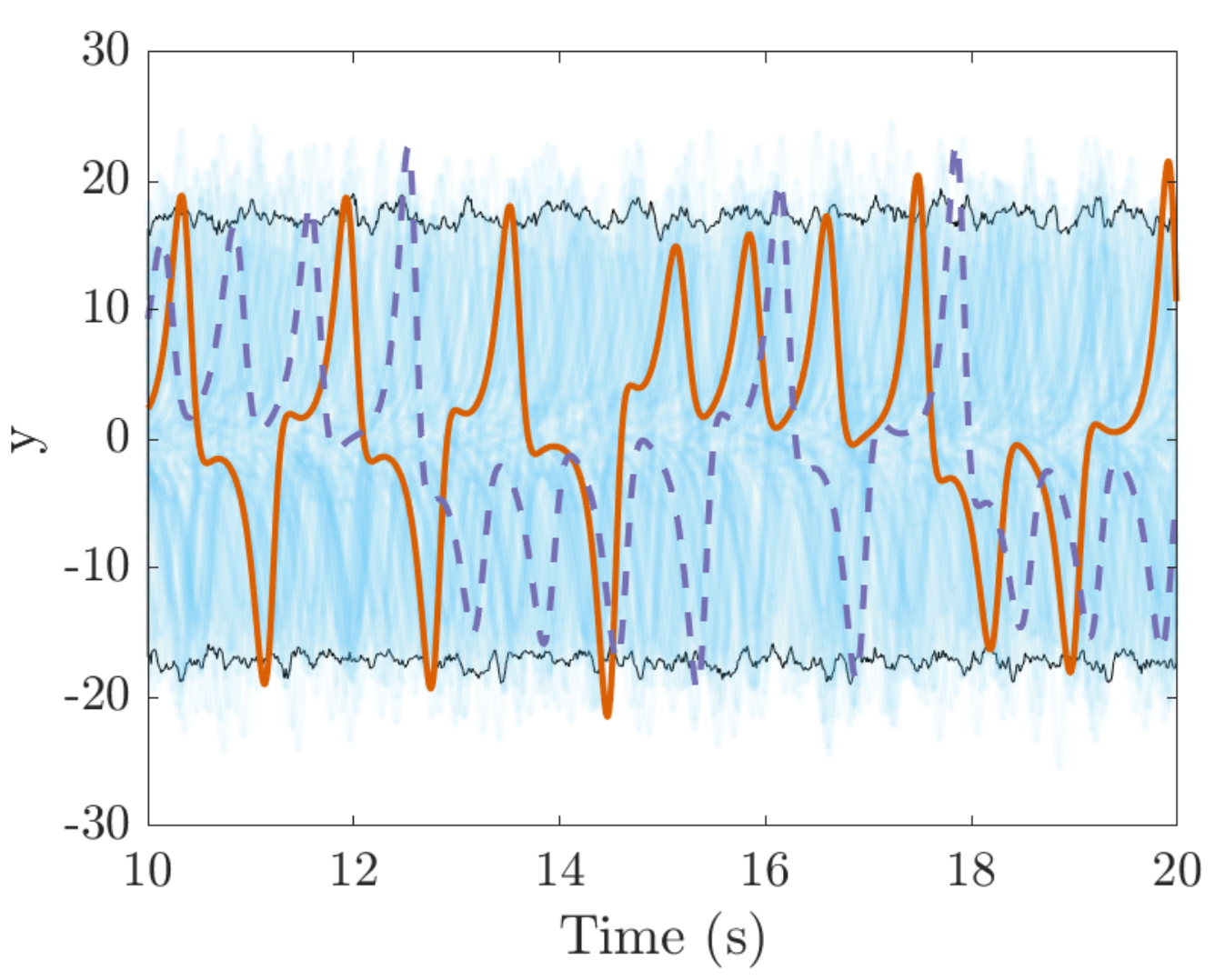}
        \caption{Prediction of $y$}
        \label{fig:LorenzPredicty}
    \end{subfigure}
    \begin{subfigure}{0.38\linewidth}
        \includegraphics[width=\linewidth]{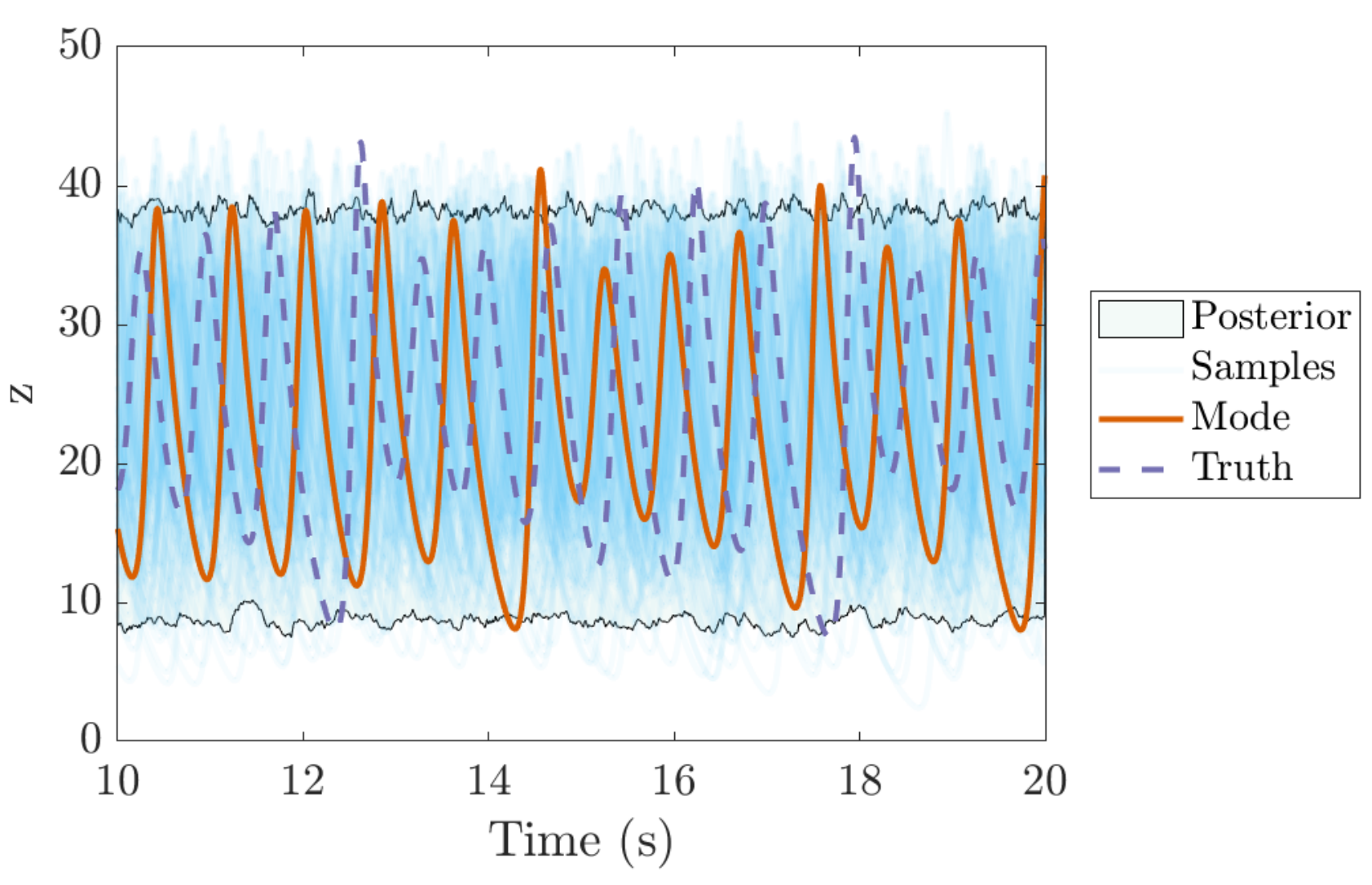}
        \caption{Prediction of $z$}
        \label{fig:LorenzPredictz}
    \end{subfigure}
    \caption{Lorenz '63 prediction posteriors. Measurements were taken every 0.1 seconds for 10 seconds with noise standard deviation of 2.0.  Although the trajectories become misaligned rather quickly due to the chaotic nature of the system, the posterior phase diagram~\ref{fig:LorenzAttractors} reveals that the algorithm has discovered that the dynamics exist on a low-dimensional attractor.}
    \label{fig:LorenzStates}
\end{figure}

\begin{figure}
    \centering
    \begin{subfigure}{0.49\linewidth}
        \includegraphics[width=\linewidth]{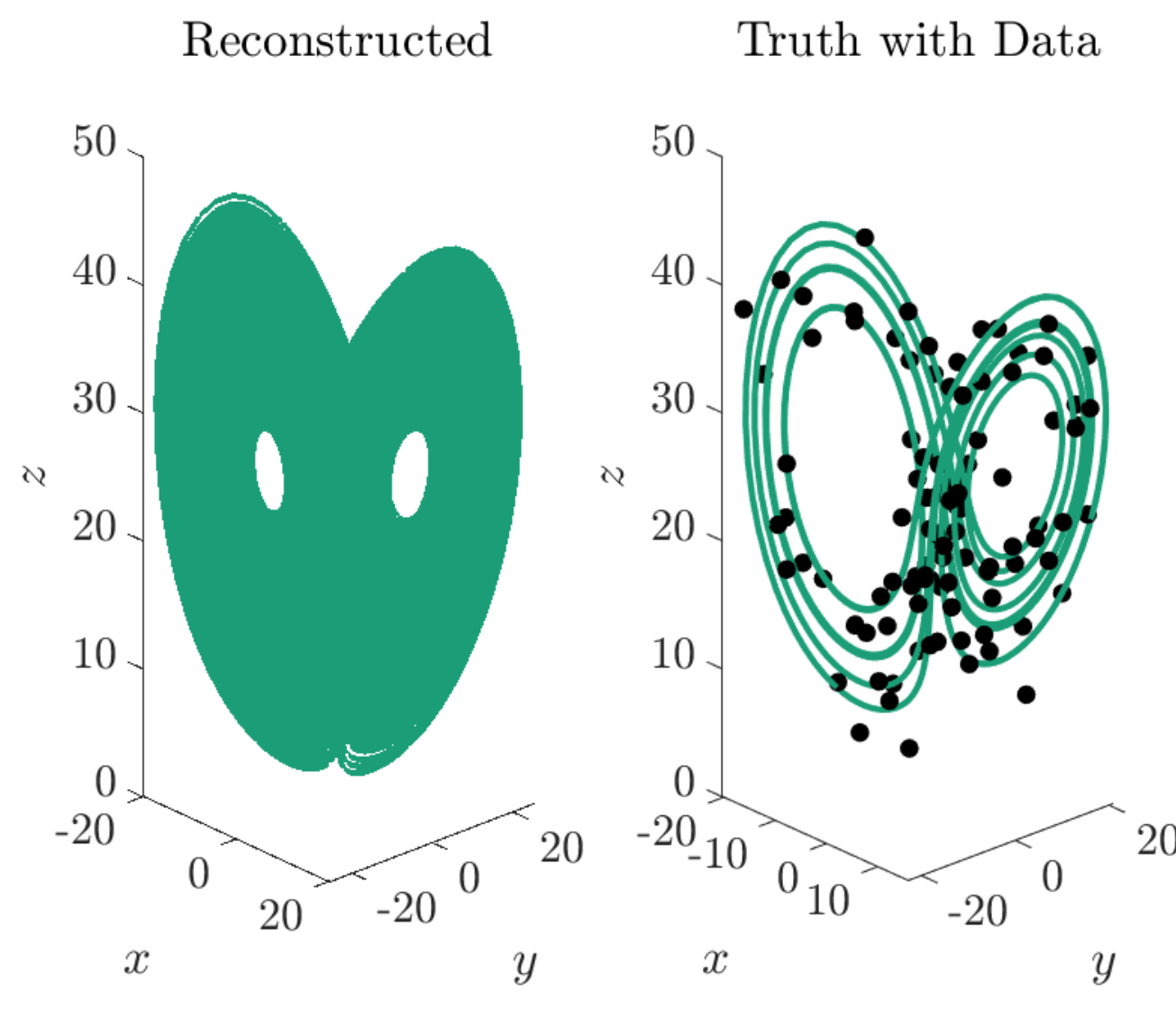}
        \caption{Reconstruction}
        \label{fig:LorenzAttractorRecon}
    \end{subfigure}
    \begin{subfigure}{0.49\linewidth}
        \includegraphics[width=\linewidth]{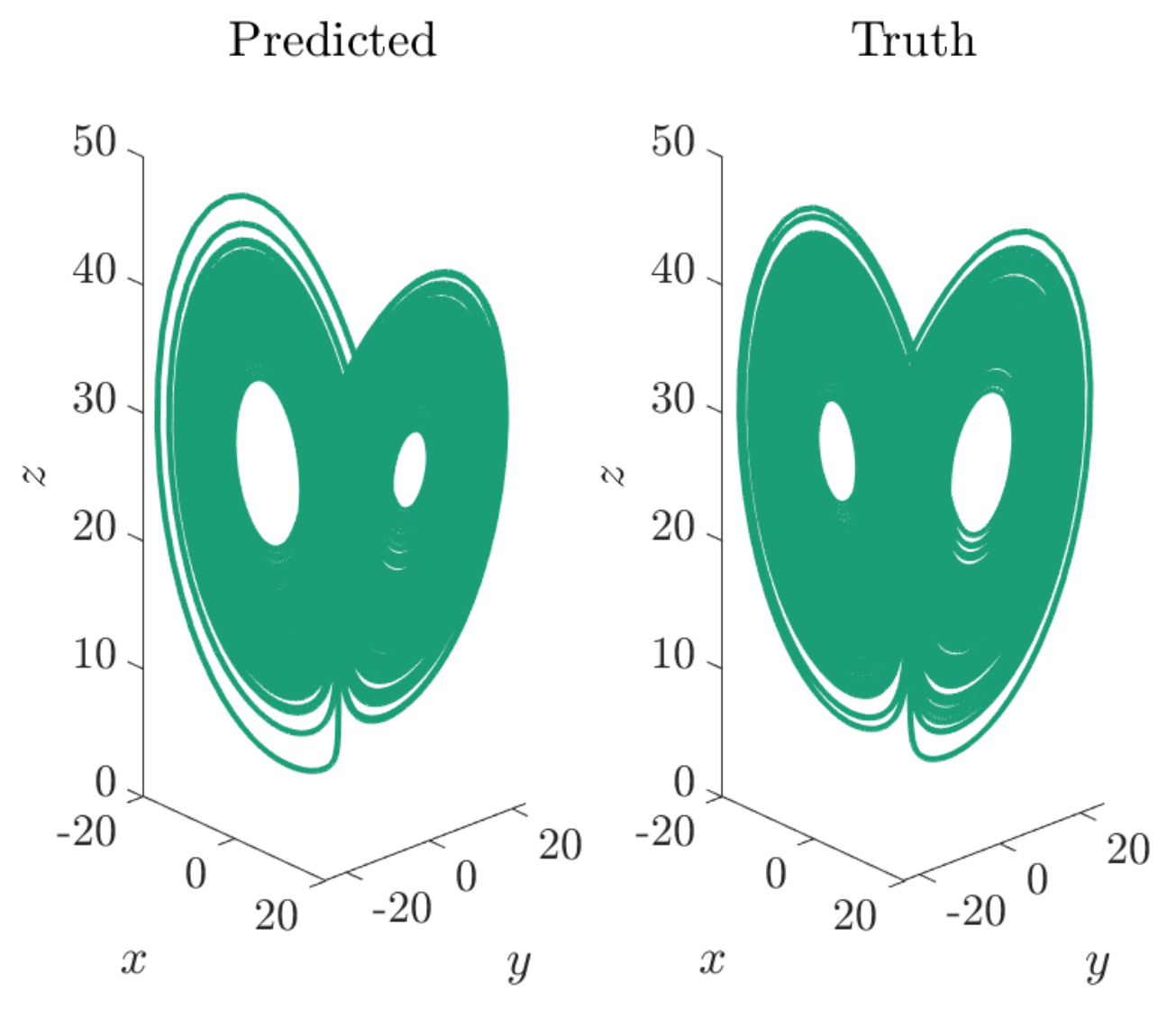}
        \caption{Prediction}
        \label{fig:LorenzAttractorPredict}
    \end{subfigure}
    \caption{Reconstruction and prediction of the Lorenz '63 attractor. The right panel compares the predicted and true trajectories up to 200 seconds using the mode of the parameter posterior distribution. The proposed approach is able to successfully discover the Lorenz attractor from sparse, noisy data.}
    \label{fig:LorenzAttractors}
\end{figure}

\subsubsection{Reaction diffusion}
\label{subsec:reacDiff}

In the final example we consider both a PDE and a case where the measurement operator $h$ is not the identity. The reaction diffusion PDE is given by 
\begin{align}
\begin{split}
    \frac{\partial C_1}{\partial t} &= \theta_1 \frac{\partial^2 C_1}{\partial x^2} + 0.1 - C_1 + \theta_3 C_1^2 C_2 \\
    \frac{\partial C_2}{\partial t} &= \theta_2 \frac{\partial^2 C_2}{\partial x^2}C_2 + 0.9 - C_1^2C_2
\end{split}
\end{align}
where $C_1,C_2$ specify the concentrations. A one dimensional spatial grid was selected to have regular intervals of 0.4 units between boundaries of -40 and 40 for a total of 201 grid points for each of the two states.  The boundary conditions at $x=\pm 40$ are 
\begin{align}
    \frac{\partial C_1}{\partial x} = \frac{\partial C_2}{\partial x} = 0,
\end{align}
and the initial condition of the system was drawn from a uniform distribution as shown
\begin{align}
    (C_i)_j\sim\mathcal{U}(0.4, 0.6), \quad \text{for }t=0;\;\forall i = 1,2;\; \forall j = 1,...,201.
\end{align}

Similar to the Lorenz example, for this system we attempt to learn only the model parameters, $\theta_1$, $\theta_2$, and $\theta_3$ rather than the complete model.  The measurement covariance matrix is assumed to be known, and the process noise covariance is fixed to be 1e-8 such that the total number of parameters that we are learning remains only three. The observation operator indirectly measures the concentration through only the first two moments of the concentration of the first species at certain time intervals
\begin{align}
\begin{split}
    y_1(t) &= \int_{-40}^{40}C_1(t)\,dx \\
    y_2(t) &= \int_{-40}^{40}C_1^2(t)\,dx.
\end{split}
\end{align}

We collect measurements every 0.5 seconds for 15 seconds with noise standard deviation of $10^{-2}$.  The reconstructions and predictions of the moments from these data using the mode of the parameter posterior distribution are shown in Figure~\ref{fig:ReacDiffObs}.  Additionally, the true and reconstructed contours of $C_1$ and $C_2$ are shown in Figure~\ref{fig:ReacDiffContour}.  The Bayesian estimate shows close agreement with the truth.

\begin{figure}
    \centering
    \begin{subfigure}{.42\linewidth}
        \includegraphics[height = 43mm]{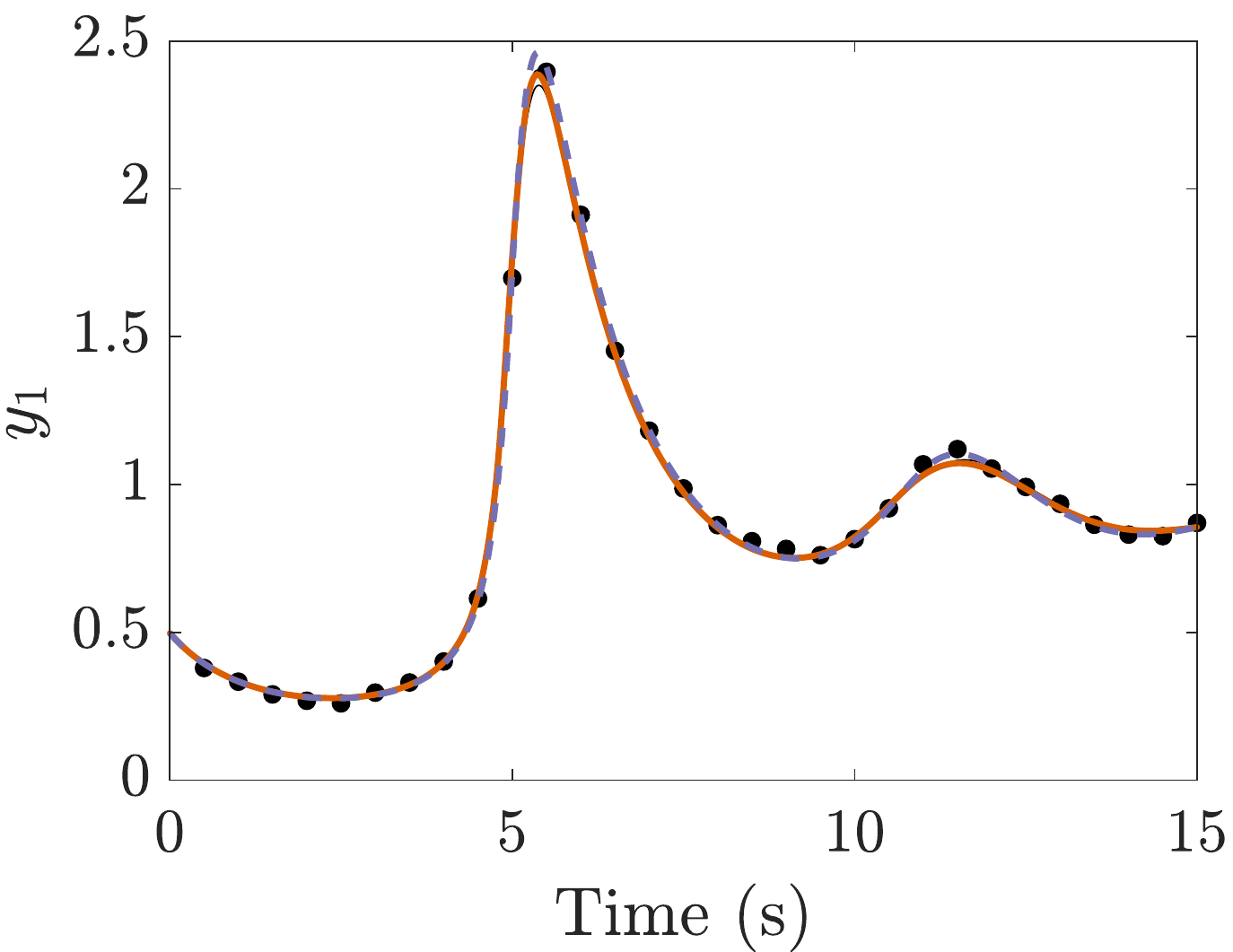}
        \caption{$y_1$, $n=30$, $\sigma=10^{-2}$}
        \label{fig:rdRecony1}
    \end{subfigure}
    \begin{subfigure}{.57\linewidth}
        \includegraphics[height=43mm]{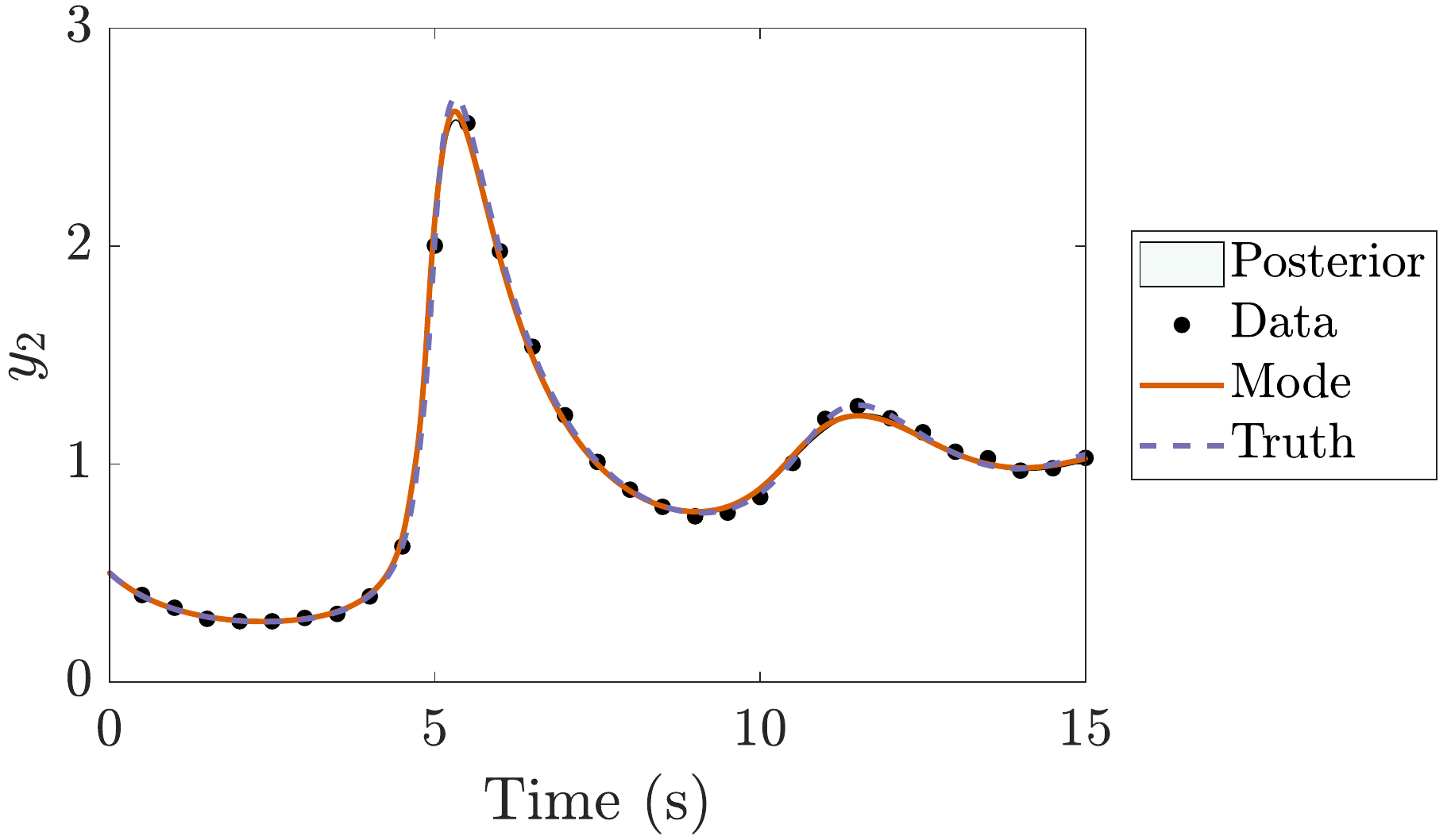}
        \caption{$y_2$, $n=30$, $\sigma=10^{-2}$}
        \label{fig:rdRecony2}
    \end{subfigure}
    \begin{subfigure}{.42\linewidth}
        \includegraphics[height=43mm]{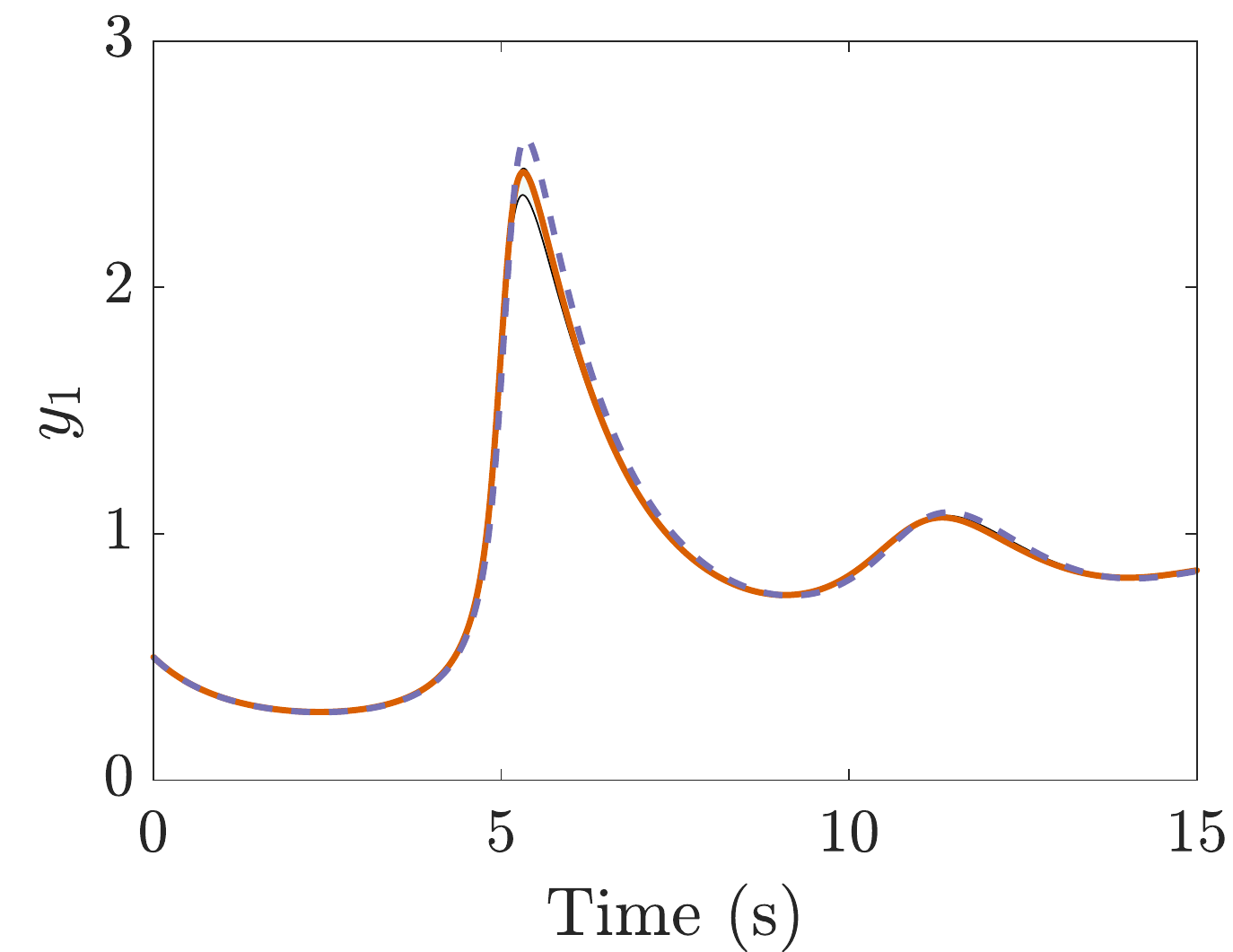}
        \caption{$y_1$ Prediction}
        \label{fig:rdRecony1_altIC}
    \end{subfigure}
    \begin{subfigure}{.57\linewidth}
        \includegraphics[height=43mm]{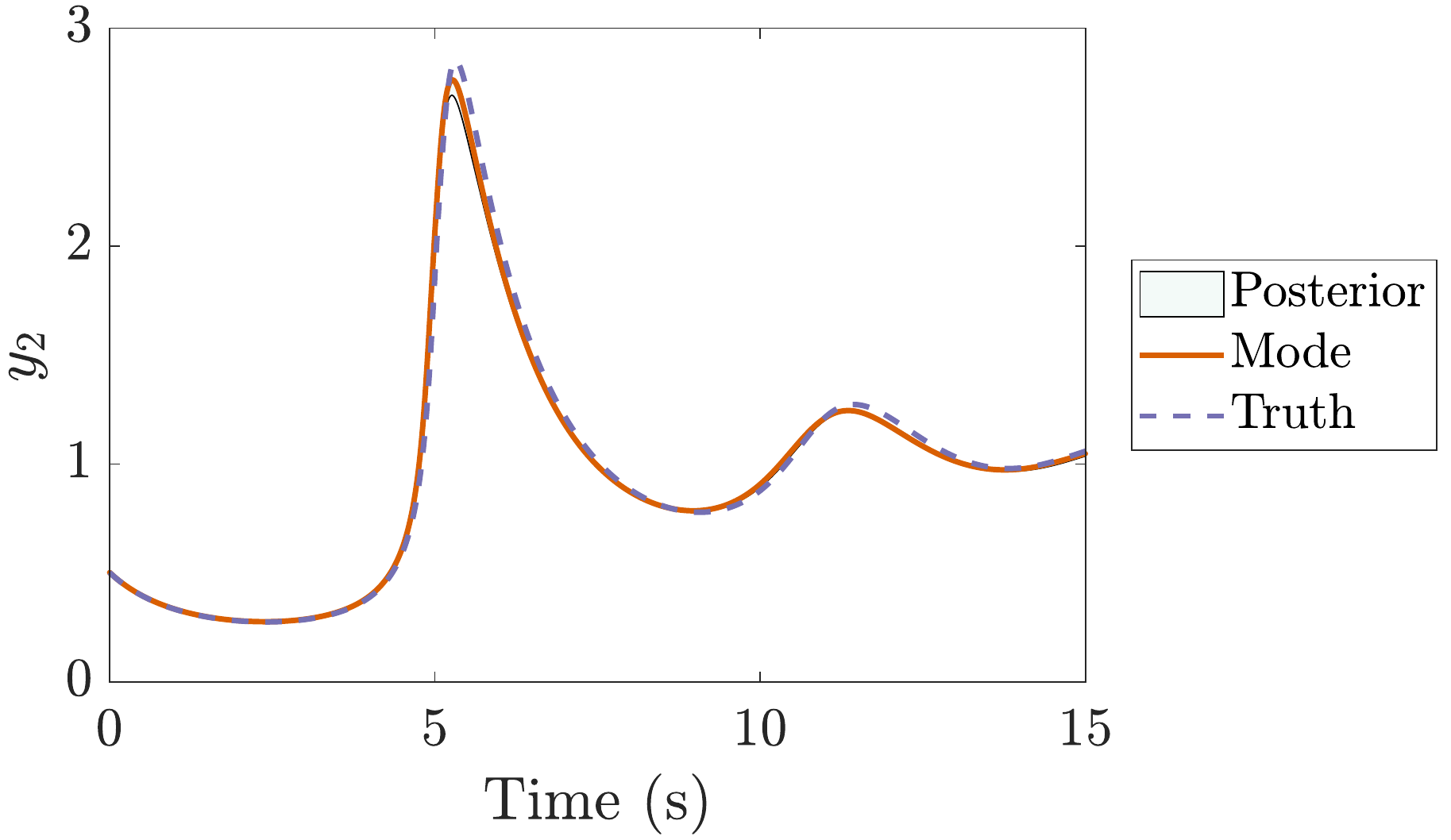}
        \caption{$y_2$ Prediction}
        \label{fig:rdRecony2_altIC}
    \end{subfigure}
    \caption{Reconstruction and prediction of the observables of the reaction diffusion system.  The top row shows the reconstruction, and the bottom row shows the prediction for an alternate initial condition.  The left column is the first measurement state (first moment), and the right column is the second measurement state (second moment).  The estimates are very close to the truth, demonstrating the generality of the learned model}
    \label{fig:ReacDiffObs}
\end{figure}

\begin{figure}
    \centering
    \begin{subfigure}{.49\linewidth}
        \includegraphics[width=\linewidth]{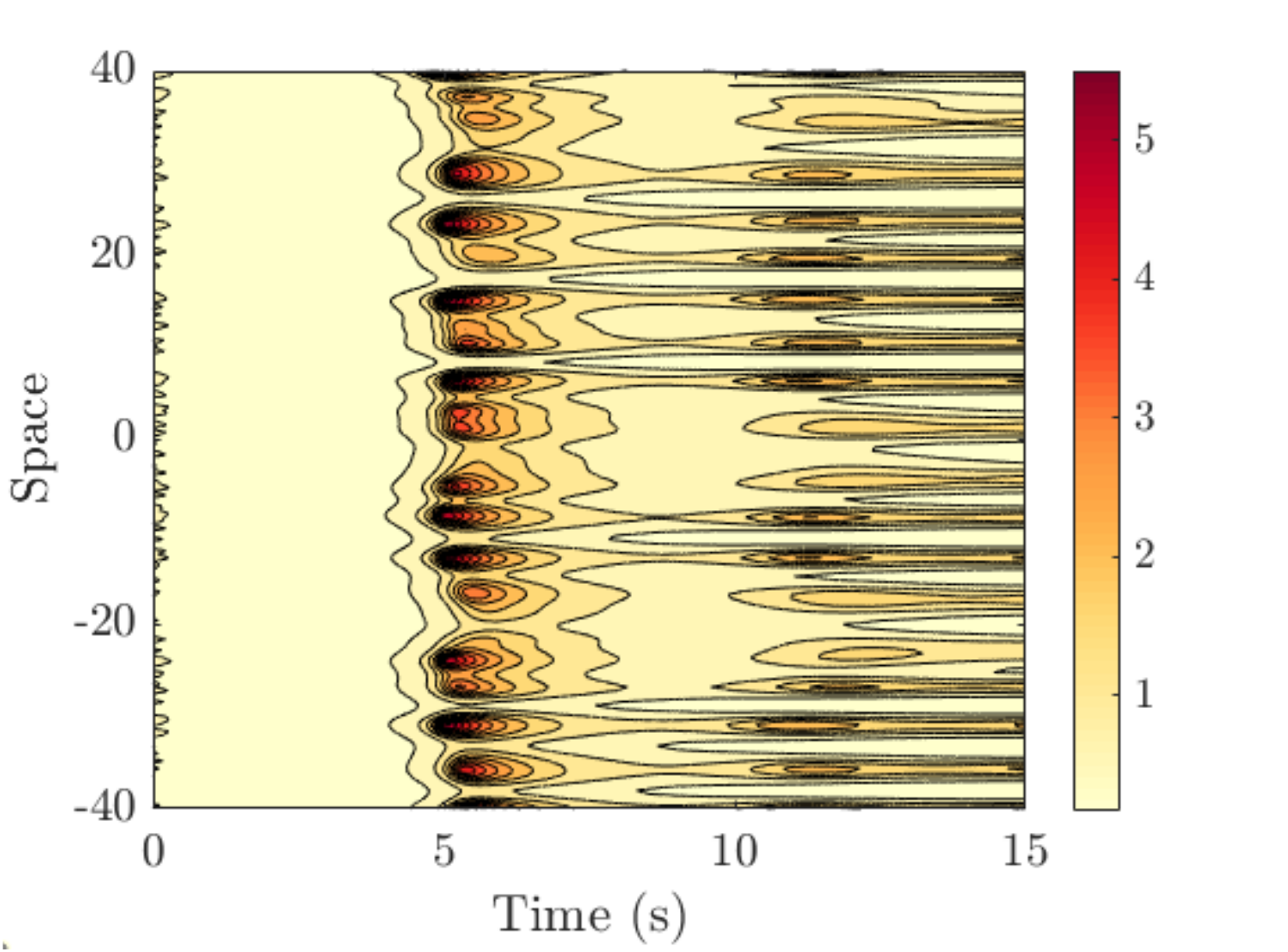}
        \caption{$C_1$ True}
        \label{fig:rdTrueContourC1}
    \end{subfigure}
    \begin{subfigure}{.49\linewidth}
        \includegraphics[width=\linewidth]{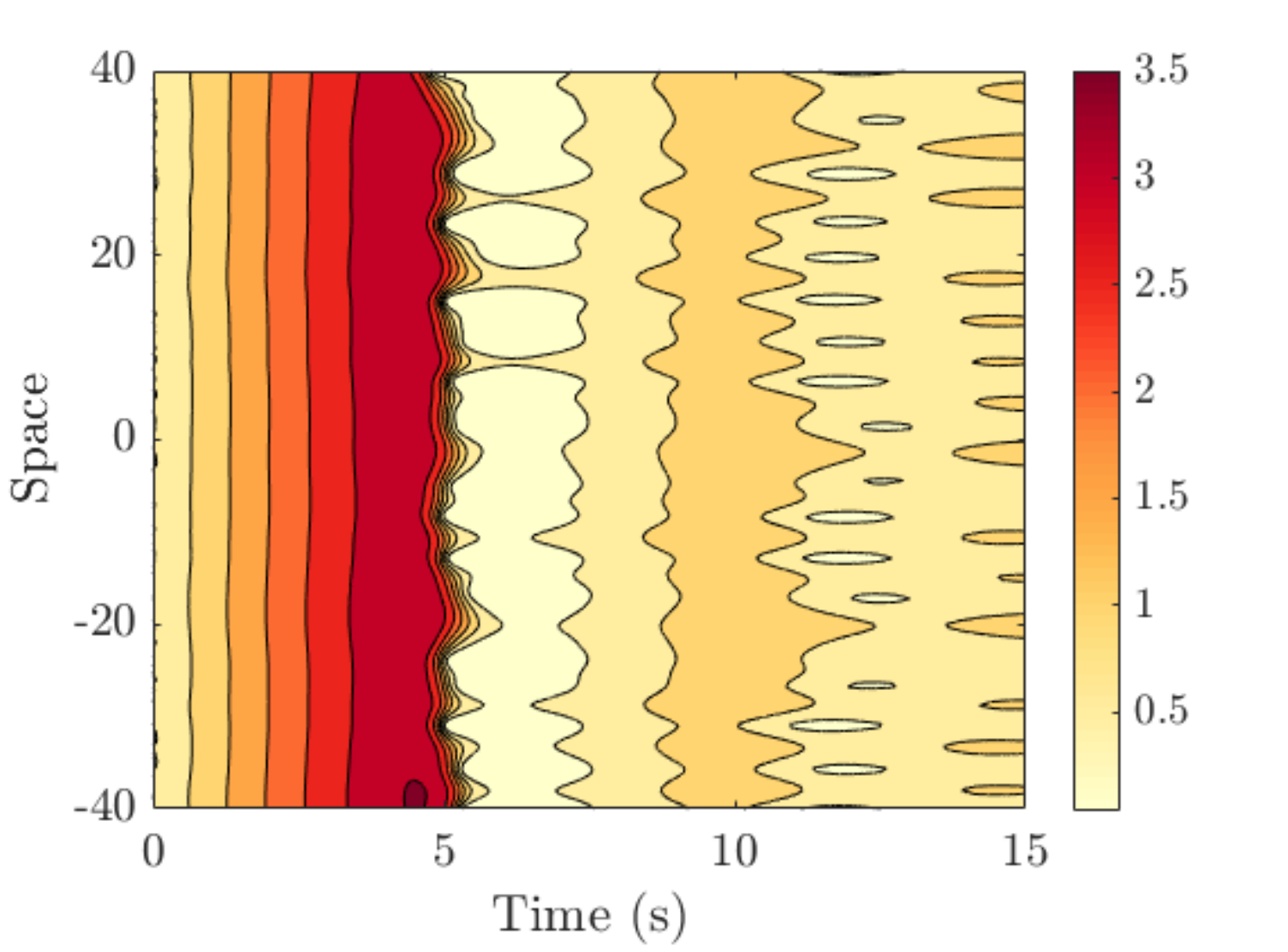}
        \caption{$C_2$ True}
        \label{fig:rdTrueContourC2}
    \end{subfigure}
    \begin{subfigure}{0.49\linewidth}
        \includegraphics[width=\linewidth]{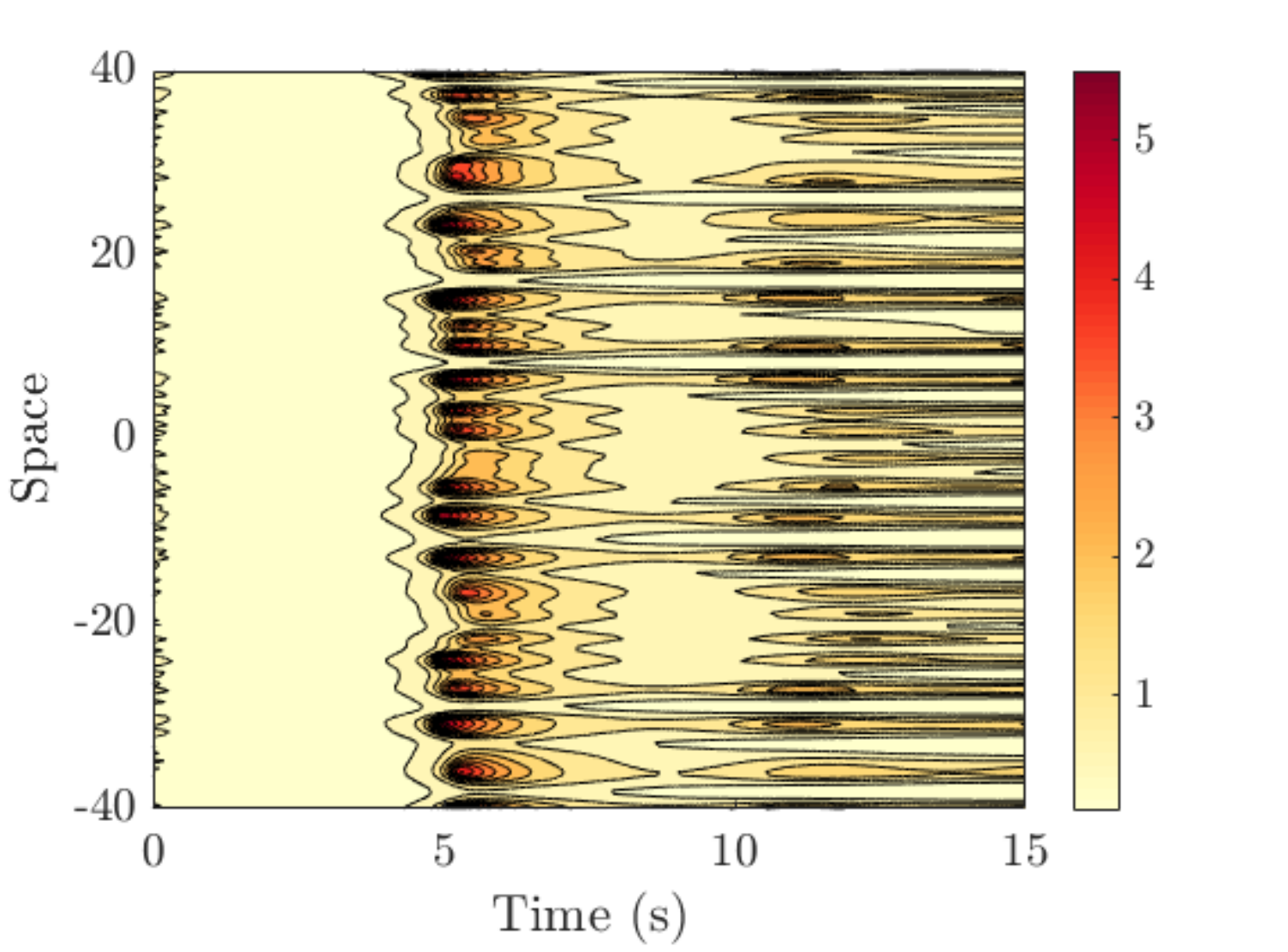}
        \caption{$C_1$ Reconstructed}
        \label{fig:rdReconContourC1}
    \end{subfigure}
    \begin{subfigure}{0.49\linewidth}
        \includegraphics[width=\linewidth]{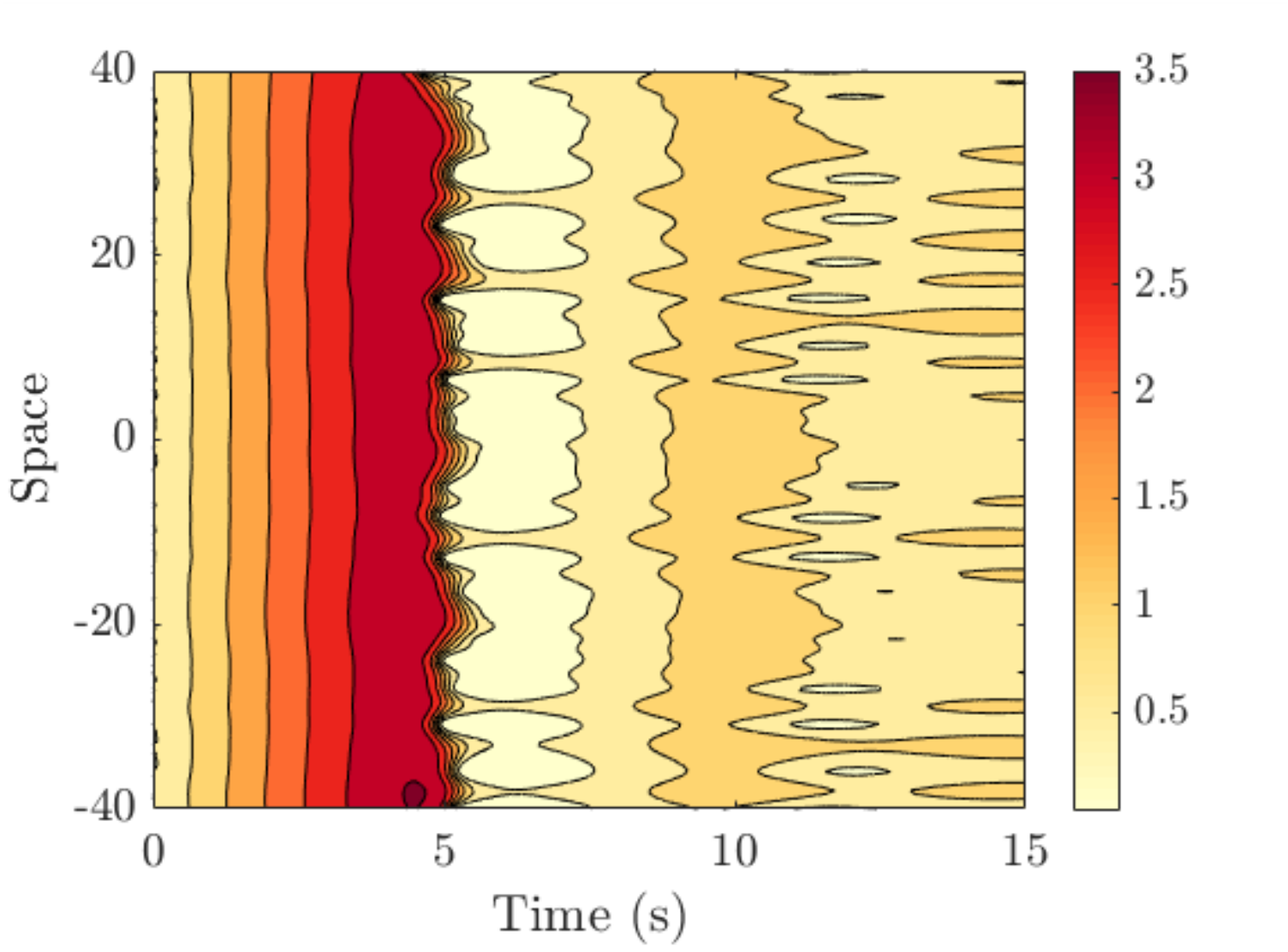}
        \caption{$C_2$ Reconstructed}
        \label{fig:rdReconContourC2}
    \end{subfigure}
    \caption{The experiment is the same as in Figure~\ref{fig:ReacDiffObs}. The top row shows the true contours of $C_1$ and $C_2$.  The bottom row shows the contours of $C_1$ and $C_2$ reconstructed using the mode of the parameter posterior distribution.  Visually, the two rows appear very similar, reflecting the strong performance of the Bayesian algorithm.}
    \label{fig:ReacDiffContour}
\end{figure}

\section{Conclusion}
In this paper, we have shown how data-driven system ID methods that consider only the measurement noise or only the process noise are impractical for many problems.  When only the measurement noise is considered, increasingly many local minima arise as data collection is continued, making identification of the optimal solution difficult.  When only process noise is considered, noisy and/or sparse measurements can cause the estimator to break down, even after incorporation of a denoising algorithm.  By deriving a probabilistic model of our dynamical system from first principles, we were able to account for how parameter, model, and measurement uncertainty can each affect the learning problem in different ways.  From this probabilistic formulation, we were then able to prove that DMD and SINDy assume noiseless measurements, and are thus poorly suited for problems where the measurement noise to process noise ratio is nonnegligible.

Next we outlined a Kalman filter and unscented Kalman filter (UKF-MCMC) MCMC algorithm for linear models and nonlinear models respectively that facilitates drawing samples from the marginal posterior without having to compute a high-dimensional integral to marginalize out the states from the joint posterior.  In the linear case, the Kalman filter algorithm targets exactly the marginal posterior, but in the nonlinear case, the UKF-MCMC targets an approximate marginal posterior.  A comparison of the computational complexity of these algorithms to that of DMD and sparse regression was then performed.  It was found that the cost of the Bayesian algorithms is roughly $n$ times more expensive than DMD and sparse regression, but for many problems, this is an acceptable cost for the enhanced performance of the Bayesian algorithms.

Lastly, the Bayesian algorithms were compared to DMD and SINDy on a number of systems for varying values of measurement noise and frequency.  It was shown that when substantial noise is introduced into the measurements, DMD and SINDy will fail due to their underlying assumption that the measurements are noiseless.  The Bayesian algorithm makes no such assumption and in addition to yielding strong performance for low-noise measurements, remains robust to noisy and infrequent data as well.  Thus it has been empirically shown that consideration of parameter, model, and measurement uncertainty leads to enhanced performance on a wider class of systems than that to which most least squares-based approaches can be reliably applied.

\section*{Acknowledgments}

This research was primarily supported by the DARPA Physics of AI Program under the grant ``Physics Inspired Learning and Learning the Order and Structure of Physics.'' It was also supported in part by the DARPA Artificial Intelligence Research Associate under the grant ``Artificial Intelligence Guided Multi-scale Multi-physics Framework for Discovering Complex Emergent Materials Phenomena.''

\appendix

\section{Pseudocode}

In this appendix we provide the pseudocode for both the linear Kalman filter and nonlinear unscented Kalman filter algorithms.

\begin{algorithm}
\begin{algorithmic}[1]
  \REQUIRE System parameters $\theta=(\thetdyn, \thetobs, \thetsig, \thetgam);$ \\
        \quad \quad Prior distribution $p(\theta)$; \\
        \quad \quad Distribution on initial condition $m_0$, $P_0$; \\
        \quad \quad Linear dynamical model parameterization $A(\thetdyn)$; \\
        \quad \quad Linear observation model parameterization $H(\thetobs)$; \\
        \quad \quad Covariance matrices $\Sigma(\thetsig)$ and $\Gamma(\thetgam)$
\ENSURE Posterior evaluation $p(\theta \mid \yn)$
\STATE Compute the prior $\probd(\theta\mid\mathcal{Y}_0)=\probd(\theta)$
 \FOR {$k=1$ to $\numObs$}
   \STATE Predict $\probd(X_{k}|\theta,\mathcal{Y}_{k-1}) = \mathcal{N}(m_{k}^{-},P_{k}^{-})$\\
   $\quad m_{k}^{-}(\theta)=A(\thetdyn)m_{k-1}$\\
   $\quad P_{k}^{-}(\theta) = A(\thetdyn)P_{k-1}A^{T}(\thetdyn)+\Sigma(\thetsig)$\\
   \STATE Compute the Evidence $\probd(y_{k}|\theta,\mathcal{Y}_{k-1}) = \mathcal{N}(\mu_{k},S_{k})$\\
   $\quad \mu_{k}(\theta) = H(\thetobs)m_{k}^{-}$\\
   $\quad S_{k}(\theta) = H(\thetobs)P_{k}^{-} H^{T}(\thetobs) + \Gamma(\thetgam)$\\
   \STATE Update $\probd(X_{k} | \theta, \mathcal{Y}_{k}) = \mathcal{N}(m_{k},P_{k})$\\
   $\quad m_{k}(\theta) = m_{k}^{-} + P_{k}^{-}H^T(\thetobs)S_{k}^{-1}(y_{k}-\mu_{k})$\\
   $\quad P_{k}(\theta) = P_{k}^{-} - P_{k}^{-}H^{T}(\thetobs)S_{k}^{-1}H(\thetobs)P_{k}^{-}$\\
   \STATE Update $\probd(\theta|\mathcal{Y}_{k}) \propto  \probd(y_k\mid\theta,\mathcal{Y}_{k-1})\probd(\theta|\mathcal{Y}_{k-1})$
\ENDFOR
\end{algorithmic}
\caption{Kalman filtering for evaluating $\probd(\theta \mid\mathcal{Y}_{n})$ (exact for linear models)}
\label{alg:margpost_KF}
\end{algorithm}

\begin{algorithm}
\begin{algorithmic}[1]
  \REQUIRE System parameters $\theta=(\thetdyn, \thetobs, \thetsig, \thetgam);$ \\
        \quad \quad Prior distribution $p(\theta)$; \\
        \quad \quad Distribution on initial condition $m_0$, $P_0$; \\
        \quad \quad Dynamical model parametrization $\Psi(\thetdyn)$; \\
        \quad \quad Observation model parameterization $h(\thetobs)$; \\
        \quad \quad Covariance matrices $\Sigma(\thetsig)$ and $\Gamma(\thetgam)$; \\
        \quad \quad UKF parameters $\alpha$, $\kappa$, $\beta$
\ENSURE Approximate evaluation of the posterior $p(\theta \mid \yn)$
\STATE Calculate $\lambda = \alpha^2(\dimx+\kappa)-\dimx$
\STATE Compute the weights\\
$W_{0}^{(m)} = \frac{\lambda}{\dimx+\lambda}$\\
$W_{0}^{(c)} = \frac{\lambda}{\dimx+\lambda}+(1-\alpha^2+\beta)$\\
$W^{(m)}_{i}=W^{(c)}_{i}=\frac{1}{2(\dimx+\lambda)},\quad \forall i=1,\ldots,2\dimx$
\STATE Compute the prior $\probd(\theta\mid\mathcal{Y}_0)=\probd(\theta)$
 \FOR {$k=1$ to $\numObs$}
   \STATE Predict $\probd(X_{k}|\theta,\mathcal{Y}_{k-1}) \approx \mathcal{N}(m_{k}^{-},P_{k}^{-})$\\
   \STATE \qquad Form the sigma points\\
   \qquad\quad $\mathcal{X}^{(0)}_{k-1}(\theta) = m_{k-1}$\\
   \qquad\quad $\mathcal{X}^{(i)}_{k-1}(\theta) = m_{k-1} + \sqrt{\dimx+\lambda}\left[\sqrt{P_{k-1}}\right]_i$\\
   \qquad\quad $\mathcal{X}^{(i+\dimx)}_{k-1}(\theta) = m_{k-1} - \sqrt{\dimx+\lambda}\left[\sqrt{P_{k-1}}\right]_i,\quad \forall i = 1,\ldots,\dimx$
   \STATE \qquad Propagate the sigma points through the dynamical model\\
   \qquad\quad $\hat{\mathcal{X}}_k^{(i)}(\theta) = \Psi(\mathcal{X}_k^{(i)},\thetdyn),\quad \forall i = 0,\ldots,2\dimx$
   \STATE \qquad Compute the mean and covariance\\
   \qquad\quad $m_{k}^{-}(\theta)=\sum_{i=0}^{2\dimx}W^{(m)}_{i}\hat{\mathcal{X}}^{(i)}_{k}$\\
   \qquad\quad $P_{k}^{-}(\theta) = \sum_{i=0}^{2\dimx}W^{(c)}_i(\hat{\mathcal{X}}^{(i)}_{k}-m_{k}^{-})(\hat{\mathcal{X}}^{(i)}_{k}-m_{k}^{-})^{T}+\Sigma(\thetsig)$\\
   
   \STATE Compute the Evidence $\probd(y_{k}|\theta,\mathcal{Y}_{k-1}) \approx \mathcal{N}(\mu_{k},S_{k})$\\
   \STATE \qquad Update the sigma points\\
   \qquad\quad $\mathcal{X}^{(0)}_{k-1}(\theta) = m_{k-1}$\\
   \qquad\quad $\mathcal{X}^{(i)}_{k-1}(\theta) = m_{k-1} + \sqrt{\dimx+\lambda}\left[\sqrt{P_{k-1}}\right]_i$\\
   \qquad\quad $\mathcal{X}^{(i+\dimx)}_{k-1}(\theta) = m_{k-1} - \sqrt{\dimx+\lambda}\left[\sqrt{P_{k-1}}\right]_i,\quad \forall i = 1,\ldots,\dimx$
   \STATE \qquad Propagate the sigma points through the observation model\\
   \qquad\quad $\hat{\mathcal{Y}}_k^{(i)}(\theta) = h(\mathcal{X}_k^{(i)},\thetobs),\quad \forall i = 0,\ldots,2\dimx$
   \STATE \qquad Compute the mean and covariance\\
   \qquad\quad $\mu_{k}(\theta) = \sum_{i=0}^{2\dimx}W^{(m)}_{i}\hat{\mathcal{Y}}^{(i)}_{k}$\\
   \qquad\quad $S_{k}(\theta) = \sum_{i=0}^{2\dimx}W^{(c)}_i(\hat{\mathcal{Y}}^{(i)}_{k}-\mu_{k}^{-})(\hat{\mathcal{Y}}^{(i)}_{k}-\mu_{k}^{-})^{T} + \Gamma(\thetgam)$\\
   
   \STATE Update $\probd(X_{k} | \theta, \mathcal{Y}_{k}) \approx \mathcal{N}(m_{k},P_{k})$\\
   \qquad\quad $C_k(\theta) = \sum_{i=0}^{2\dimx} W^{(c)}_{i}(\mathcal{X}_{k}^{(i)}-m_{k}^{-})(\hat{\mathcal{Y}}_{k}^{(i)}-\mu_{k})^{T}$\\
   \qquad\quad $m_{k}(\theta) = m_{k}^{-} + (C_{k}S_{k}^{-1})(y_{k}-\mu_k)$\\
   \qquad\quad $P_{k}(\theta) = P_{k}^{-} - (C_{k}S_{k}^{-1})S_{k}^{-1}(C_{k}S_{k}^{-1})^{T}$\\
   
   \STATE Update $\probd(\theta|\mathcal{Y}_{k}) \propto  \probd(y_k\mid\theta,\mathcal{Y}_{k-1})\probd(\theta|\mathcal{Y}_{k-1})$
\ENDFOR
\end{algorithmic}
\caption{Unscented Kalman filtering algorithm for approximating $\probd(\theta \mid\yn)$}
\label{alg:margpost_UKF}
\end{algorithm}

\bibliography{mybib}

\end{document}

%% file: macros.tex
\newcommand{\reals}{\mathbb{R}}
\newcommand{\posint}{\mathbb{Z}_{+}}

\newcommand{\prob}{\mathbb{P}}
\newcommand{\probd}{p}
\newcommand{\like}{\mathcal{L}}
\newcommand{\ee}{\mathbb{E}}

\newcommand{\yn}{\mathcal{Y}_n}
\newcommand{\xn}{\mathcal{X}_n}

\newcommand{\yk}{\mathcal{Y}_k}

\newcommand{\avg}{\textrm{\normalfont avg}}
\newcommand{\map}{\textrm{\normalfont map}}
\newcommand{\tavg}{\textrm{\normalfont $\theta$-avg}}
\newcommand{\tmap}{\textrm{\normalfont $\theta$-map}}

\newcommand{\thetsig}{\theta_{\Sigma}}
\newcommand{\thetgam}{\theta_{\Gamma}}
\newcommand{\thetdyn}{\theta_{\Psi}} 
\newcommand{\thetobs}{\theta_{h}}

\newcommand{\dimx}{d}
\newcommand{\dimpar}{p}
\newcommand{\dimy}{m}
\newcommand{\numObs}{n}